\documentclass{article}
\usepackage[numbers]{natbib}

\usepackage{PRIMEarxiv}

\usepackage[utf8]{inputenc} 
\usepackage[T1]{fontenc}    
\usepackage{hyperref}       
\usepackage{url}            
\usepackage{booktabs}       
\usepackage{amsfonts}       
\usepackage{nicefrac}       
\usepackage{microtype}      
\usepackage{fancyhdr}       
\usepackage{graphicx}       
\usepackage{amsmath} 
\usepackage{amsthm}

\usepackage{IEEEtrantools}
\usepackage{bbm}
\usepackage{algorithm}
\usepackage{algpseudocode}
\usepackage{booktabs}
\usepackage{multirow}
\usepackage{xurl}
\usepackage{longtable}
\usepackage{booktabs}
\usepackage{placeins}

\newcommand*\diff{\mathop{}\!\mathrm{d}}

\DeclareMathOperator{\argmin}{argmin}

\newtheorem{definition}{Definition}
\newtheorem{theorem}{Theorem}

\newtheorem{assumption}{Assumption}
\newtheorem{proposition}{Proposition}
\newtheorem{remark}{Remark}

\newtheorem{problem}{Problem}

\title{Kernel Affine Hull Machines as Compute-Efficient Encoders for Frozen Semantic Spaces
}

\author{
Mohit Kumar \\
University of Rostock, Germany\\ Software Competence Center Hagenberg GmbH, Austria \\
  Hagenberg, Austria \\
 \texttt{mohit.kumar@uni-rostock.de} 
\And
  Somayeh Kargaran, Bernhard A. Moser, Manuela Geiß  \\
  Software Competence Center Hagenberg GmbH \\
  Hagenberg, Austria \\
  \texttt{\{somayeh.kargaran,bernhard.moser,manuela.geiss\}@scch.at} \\
}

\begin{document}
\maketitle

\begin{abstract}
{\bf Background:}
Transformer-based semantic encoders are effective, but when the teacher representation space and corpus index are fixed, online query encoding can become the recurring deployment bottleneck. This raises the methodological question whether repeated neural inference can be replaced by a lighter estimator with analytically explicit approximation behavior.

{\bf Objectives:}
We study Kernel Affine Hull Machines (KAHMs) as mathematically explicit, compute-efficient encoders for frozen teacher representation spaces. Retrieval is used only as a proof-of-principle instance of the broader fixed-space encoding problem.

{\bf Methods:}
We formulate semantic encoding as conditional-mean estimation. The teacher representation is modeled as a noisy prototype mixture with posterior cluster-probability weights. KAHM geometry estimates these weights from lexical features in an explicitly identified RKHS hypothesis space, giving non-asymptotic posterior-approximation control. Semantic prototypes are then refined by normalized least-mean-squares updates on noisy teacher embeddings, yielding an error decomposition into posterior-approximation, finite-sample/generalization, and teacher-noise terms.

{\bf Results:}
On a controlled Austrian-law benchmark in a frozen Mixedbread semantic corpus space, KAHM achieves the strongest teacher-space reconstruction and principal rank-sensitive retrieval performance among five evaluation-matched learned adapters. Relative to direct transformer encoding on the same frozen index, it reduces online per-query time from $800.663$\,ms to $93.834$\,ms, an $8.53\times$ CPU speedup.

{\bf Conclusions:}
KAHM provides an analytically explicit, lightweight estimator for fixed-teacher lexical-to-semantic adaptation with interpretable error terms. The retrieval experiment supports this methodological claim as a controlled proof of principle.
\end{abstract}


\section{Introduction}\label{sec_introduction}
Transformer-based retrieval has reshaped semantic search by enabling documents and queries to be mapped into shared vector spaces in which relevance can be approximated by vector similarity. Representative developments include general-purpose sentence encoders such as Sentence-BERT \cite{sentence_bert}, dual-encoder dense retrieval systems such as DPR \cite{karpukhin2020dpr}, and more expressive late-interaction or learned sparse retrieval models such as ColBERT \cite{khattab2020colbert} and SPLADE \cite{formal2021splade}. These developments build on the transformer architecture \cite{vaswani2017attention} and the bidirectional pretraining paradigm introduced by BERT \cite{devlin2019bert}. At the same time, the current text-embedding landscape has become broader and stronger: modern embedding families such as E5 \cite{wang2022e5}, INSTRUCTOR \cite{su2022instructor}, and M3-Embedding \cite{chen2024m3} demonstrate that a single-vector representation can often transfer well across retrieval, clustering, and related tasks, while benchmarks such as BEIR \cite{thakur2021beir} and MTEB \cite{muennighoff2023mteb} have made it possible to compare such models across heterogeneous domains and task types. A consistent lesson from this literature is that there is no universally dominant retrieval architecture: stronger semantic models often improve relevance, but the gains come with nontrivial trade-offs in latency, compute, memory, and deployment complexity \cite{thakur2021beir,muennighoff2023mteb}.

Recent efficient retrieval systems broaden that retrieval landscape further. Strong compact or generalizable dual encoders \cite{izacard2022unsupervised,ni2022large}, efficient late-interaction systems \cite{santhanam2022colbertv2}, and newer general-purpose embedding models \cite{lee2025nvembed} provide important points of reference for deployment-oriented retrieval. The present paper addresses a different but broadly relevant systems question: when a strong representation space has already been fixed by a teacher model, can repeated online neural query encoding be replaced by a substantially lighter adapter whose ingredients, approximation behavior, and error terms remain analytically explicit? That question matters beyond the particular legal benchmark studied here. In modern AI systems, strong neural representations are often computed offline and then reused under deployment constraints in which test-time latency, energy, memory footprint, hardware simplicity, auditability, or implementation stability become first-order considerations. This pattern appears not only in retrieval, but more generally whenever a frozen representation space is treated as infrastructure and the remaining problem is to map inexpensive runtime features into that space. From that perspective, the paper is about a modular AI design problem: how to preserve the decision-relevant benefits of a strong neural teacher while replacing repeated online inference by a lighter estimator whose behavior is mathematically inspectable.

A central systems issue in semantic retrieval is the asymmetry between offline corpus processing and online query processing. In many practical deployments, corpus representations can be precomputed and indexed offline, whereas every incoming query must still be encoded at serving time. This makes query-side encoding the latency-critical part of the pipeline. Existing efficiency-oriented alternatives improve this trade-off but generally do not eliminate online neural inference: late-interaction models still require transformer-based query encoding \cite{khattab2020colbert,santhanam2022colbertv2}, learned sparse retrievers still construct query representations through a neural encoder \cite{formal2021splade}, and distillation or compression methods replace a large teacher by a smaller student that must nevertheless be executed online \cite{hinton2015distilling,sanh2019distilbert,jiao2020tinybert}. The question studied here is whether repeated online neural query encoding can be replaced by a lightweight lexical-to-semantic adapter under a fixed teacher space and serving contract.

This paper studies precisely that problem. We assume that a strong teacher encoder has already produced a frozen semantic embedding space and that the corpus has already been indexed in that space. At deployment time, rather than rerunning the teacher for every query, we ask whether the query embedding can be estimated directly from inexpensive lexical features. The goal is therefore different from standard student-network compression: we seek a deployment-time encoder whose runtime path is lightweight, whose mathematical structure is explicit, and whose approximation behavior can be analyzed directly in terms of estimator design and error decomposition.

Our starting point is the theory of kernel methods in reproducing kernel Hilbert spaces (RKHS), which provides a mathematically explicit framework for representation, regularization, and approximation \cite{aronszajn1950theory,hofmann2008kernel,scholkopf2002learning}. Within this framework, Kernel Affine Hull Machines (KAHMs) were introduced as geometrically inspired kernel machines that learn bounded structures in affine hulls via RKHS constructions \cite{KAHM}. Subsequent work showed that KAHM-based models admit gradient-free posterior-estimation procedures together with explicit guarantees in collaborative and federated settings \cite{kumar2024geometricallyinspiredkernelmachines,kumar2025operatortheoreticframeworkgradientfreefederated}. In the present paper, those earlier results matter primarily as motivation for model selection: they make KAHMs plausible candidates for a deployment-time encoder whose posterior-weight estimation is gradient-free, whose overall training pipeline is backpropagation-free, and whose inference path is lightweight and analytically tractable. Here we test whether the same geometric machinery can serve as an effective query-side adapter into a frozen semantic space.

The scientific perspective of the paper is methodological. We formulate deployment-time semantic encoding from lexical query features into a fixed teacher space as a conditional-mean estimation problem. The target semantic vector admits a noisy prototype-mixture representation in which semantic cluster prototypes are weighted by posterior cluster probabilities \cite{bishop2006prml}. KAHM geometry is used to estimate these posterior weights from lexical features in an explicitly identified RKHS hypothesis space without backpropagation through a parameterised neural network, while the semantic prototypes themselves are refined from noisy teacher embeddings by an explicit normalized least-mean-squares update. This yields an analytically explicit encoder together with an end-to-end encoding-error decomposition into posterior-approximation, finite-sample/generalization, and teacher-noise terms. The resulting online path is therefore reduced to inexpensive lexical feature extraction, evaluation of a small number of geometric quantities, formation of a weighted prototype mixture, and retrieval in the frozen dense index, with no transformer forward pass and no learned neural gating network at serving time. 
\begin{remark}[Target Deployment Regime]
The paper targets a deployment regime that is narrower than the full retrieval landscape but broad enough to be methodologically important: a strong semantic teacher space is already
fixed, corpus representations can be computed offline, query-time latency is the principal serving bottleneck, and sufficient in-domain supervision is available to learn a lightweight query-side adapter. Under this regime, the scientific question is whether repeated online neural query encoding can be replaced by an analytically explicit lexical-to-semantic adapter while keeping the teacher space, the corpus index, and the serving contract fixed. This restriction is deliberate rather than incidental. It isolates the query-adaptation problem from confounding changes in teacher model quality, corpus representation, interaction mechanism, or index structure. As a result, performance differences within the evaluation-matched comparison class can be interpreted primarily as evidence about the quality of the query-side adapter itself.
\end{remark}
\begin{remark}[Central Claim and Scope]
The central claim of this paper is that, when a strong semantic teacher space is already fixed and corpus representations can be computed offline, deployment-time semantic encoding can be treated as an explicit estimation problem rather than as repeated neural inference. In this setting, KAHM geometry provides a lightweight posterior-weight estimator for semantic prototype mixtures, admits an interpretable approximation analysis, and yields an end-to-end encoding-error decomposition when combined with NLMS prototype refinement. The Austrian-law retrieval benchmark is therefore used as a proof-of-principle instantiation of this broader fixed-representation encoding problem: it tests whether the analytically explicit estimator preserves decision-relevant behavior under a controlled serving contract.
\end{remark}

The contributions of the paper are fourfold: 1) We formulate deployment-time lexical-to-semantic encoding into a fixed teacher space as a conditional-mean estimation problem. This recasts a practical serving bottleneck in semantic retrieval as a precise estimation problem and makes explicit the broader AI setting in which a frozen neural representation space is reused under runtime constraints. 2) We derive a KAHM-based estimator for semantic-cluster posterior weights from lexical features. This turns query adaptation into a mathematically analyzable posterior-estimation problem in an explicitly identified RKHS hypothesis space and yields non-asymptotic control of the posterior-approximation step. 3) We combine this posterior estimator with an NLMS-based refinement of semantic prototypes learned from noisy teacher embeddings and obtain an explicit end-to-end decomposition of the total encoding error into posterior-approximation, finite-sample/generalization, and teacher-noise terms. The resulting method is therefore not only deployment-oriented but also analytically interpretable. 4) We validate the resulting encoder on a controlled Austrian-law retrieval benchmark under a deliberately matched deployment contract. Within a comparison class that holds fixed the lexical front end, the teacher space, the frozen corpus index, the supervision granularity, and the serving rule, the proposed KAHM encoder achieves the strongest teacher-space reconstruction and the strongest principal rank-sensitive retrieval performance among the learned adapters considered here while substantially reducing online query time relative to direct transformer query encoding.

The remainder of the paper is organized as follows. Section~\ref{sec_related_work} reviews related work. Section~\ref{sec_background} introduces the notation, problem formulation, and KAHM definitions required in the sequel. Section~\ref{sec_encoding_solution} develops the encoding solution, including the hypothesis-space construction, the KAHM-based posterior estimator, the adaptive prototype estimation procedure, and the theoretical analysis. Section~\ref{sec_experiments} describes the controlled retrieval instantiation and reports reconstruction, retrieval, and runtime results. Section~\ref{sec_conclusion} concludes with the scope of the empirical claim and directions for broader validation.

\section{Related Work and Positioning}\label{sec_related_work}
This section positions the paper relative to efficient retrieval, neural compression, kernel conditional-mean methods, prior KAHM work, legal retrieval, and reasoning-oriented retrieval. The goal is not to provide a complete survey of retrieval architectures, but to isolate the fixed-teacher query-adaptation setting addressed in this paper.
\subsection{Efficient Retrieval and the Query-Encoding Bottleneck}
Dense and neural retrieval methods map queries and documents into vector spaces in which relevance can be estimated by vector similarity or learned interaction functions. Dual-encoder retrieval models such as DPR~\cite{karpukhin2020dpr}, sentence encoders such as Sentence-BERT~\cite{sentence_bert}, and modern embedding families such as E5~\cite{wang2022e5}, INSTRUCTOR~\cite{su2022instructor}, and M3-Embedding~\cite{chen2024m3} have shown that single-vector representations can support effective retrieval and transfer across tasks. Benchmark suites such as BEIR~\cite{thakur2021beir} and MTEB~\cite{muennighoff2023mteb} have further emphasized that retrieval quality must be interpreted across heterogeneous domains, tasks, and deployment constraints. A complementary line of work improves retrieval efficiency by modifying the representation or interaction mechanism. Late-interaction systems such as ColBERT~\cite{khattab2020colbert} and ColBERTv2~\cite{santhanam2022colbertv2} reduce some of the cost of full cross-attention while retaining expressive query--document interactions. Learned sparse retrieval models such as SPLADE~\cite{formal2021splade} construct sparse lexical-semantic representations through neural encoders. Compact dense retrievers and generalizable embedding models further improve the quality--efficiency trade-off~\cite{izacard2022unsupervised,ni2022large,lee2025nvembed}. These approaches are important points of reference, but they usually retain an online neural query-encoding step. The present paper studies a narrower deployment question: when the semantic teacher space and corpus index are already fixed, can repeated online neural query encoding be replaced by a substantially lighter estimator whose structure and approximation terms remain explicit?
\subsection{Distillation, Compression, and Fixed-Teacher Adaptation}
Knowledge distillation and model compression replace a large model by a smaller learned model, often with substantial improvements in latency and deployability~\cite{hinton2015distilling,sanh2019distilbert,jiao2020tinybert}. In retrieval, distillation can transfer ranking behavior, embedding geometry, or supervision signals from a stronger model to a smaller student. However, the deployed student is still typically a parameterized neural model executed at query time.
The present paper differs from standard compression in the object being learned and in the form of the deployed estimator. We do not train a smaller transformer or neural student to imitate the teacher. Instead, we keep the teacher representation space fixed, keep the corpus index fixed, and learn a lexical-to-semantic adapter whose online path consists of inexpensive lexical feature extraction, KAHM-induced posterior-weight estimation, and a semantic prototype mixture. Thus, the comparison of interest is not against the full universe of efficient retrieval architectures, but against methods that operate under the same fixed-space serving contract.
\subsection{Kernel and Conditional-Mean Perspectives}
The mathematical formulation of this paper is closest to conditional-mean estimation. Given an inexpensive lexical representation $x$, the desired semantic encoding is treated as an estimator of the conditional mean of the teacher representation $z$. Kernel methods and RKHS are a natural choice for such problems because of the available theory on function spaces, regularization, approximation, and generalization properties~\cite{aronszajn1950theory,scholkopf2002learning,hofmann2008kernel}. Related RKHS perspectives also appear in kernel embeddings of conditional distributions and nonparametric inference with positive definite kernels~\cite{song2013kernelembeddings,fukumizu2013kernelbayes}. The present work uses this perspective in a deployment-specific way, where the conditional mean is the semantic teacher representation that would otherwise be computed by a neural encoder at query time. The prototype-mixture formulation then separates the problem into two interpretable parts: estimating semantic-cluster posterior weights from lexical features and estimating semantic prototypes in the frozen teacher space. This decomposition is what allows the final error analysis to distinguish posterior-approximation, finite-sample/generalization, and teacher-noise terms.
\subsection{Prior KAHM Work and Novelty of the Present Paper}
KAHMs were introduced as geometrically inspired kernel machines based on learning data points in RKHS~\cite{KAHM}. Subsequent work studied related KAHM-based learning procedures, including gradient-free posterior-estimation and collaborative or federated learning settings~\cite{kumar2024geometricallyinspiredkernelmachines,kumar2025operatortheoreticframeworkgradientfreefederated}. These earlier results motivate KAHMs as candidates for lightweight geometric estimation, but they do not by themselves solve the fixed-teacher semantic encoding problem studied here. The novelty of the present paper is therefore not the KAHM definition alone. The new contribution is the use of KAHM geometry as a semantic encoder for frozen neural representation spaces. Specifically, we formulate fixed-teacher lexical-to-semantic encoding as conditional-mean estimation, derive KAHM-induced posterior weights for semantic prototype mixtures, refine the prototypes by an explicit NLMS procedure from noisy teacher embeddings, and obtain an end-to-end encoding-error decomposition. The Austrian-law retrieval task is then used as a controlled proof-of-principle instance of this broader methodological claim.

\subsection{Relation to Legal and Reasoning-Oriented Retrieval}
The empirical study uses Austrian-law retrieval because it provides a controlled setting in which documents can be embedded offline, query-time encoding is the latency-critical component, and evaluation can be performed under a fixed teacher space and serving rule. Recent work on legal retrieval and legal question answering studies complementary goals, including stronger legal ranking models, legal reasoning behavior, and domain-specific retrieval pipelines~\cite{feng2024legal,deng2024learning,gao2024enhancinglegal,zheng2025reasoning}. The present paper is not primarily a contribution to legal retrieval as a domain application. Rather, legal retrieval is used as a structured testbed for the fixed-representation encoding problem: whether an analytically explicit KAHM adapter can preserve decision-relevant behavior in a frozen semantic space while avoiding repeated online transformer inference. 
\begin{table}[!htbp]
\centering
\small
\renewcommand{\arraystretch}{1.12}
\setlength{\tabcolsep}{4pt}
\caption{Positioning of the present paper relative to prior related work.}
\label{tab:positioning-novelty}
\begin{tabular}{@{}p{0.22\linewidth}p{0.23\linewidth}p{0.23\linewidth}p{0.28\linewidth}@{}}
\toprule
Reference class & What it established & Remaining gap for this paper & Contribution of the present paper \\
\midrule
Prior KAHM-based learning~\cite{KAHM}
&
Introduced the geometric kernel-machine construction based on learning in RKHS.
&
Does not formulate semantic encoding into a frozen neural teacher space.
&
Uses KAHM geometry as the posterior-weighting mechanism in a lexical-to-semantic encoder.
\\
\addlinespace[2pt]

Prior KAHM posterior estimation and gradient-free learning~\cite{kumar2024geometricallyinspiredkernelmachines,kumar2025operatortheoreticframeworkgradientfreefederated}
&
Showed that KAHM-induced geometric quantities can support posterior-like estimation without backpropagation through a neural model.
&
Does not address prototype-mixture reconstruction of teacher embeddings or query-side semantic adaptation.
&
Derives KAHM-induced semantic-cluster posterior weights for a fixed-teacher prototype-mixture estimator.
\\
\addlinespace[2pt]

Prior collaborative and federated KAHM settings~\cite{kumar2024geometricallyinspiredkernelmachines,kumar2025operatortheoreticframeworkgradientfreefederated}
&
Studied KAHM learning in distributed or collaborative regimes.
&
Does not analyze deployment-time replacement of repeated neural query encoding under a frozen corpus index.
&
Studies a fixed-space serving contract in which the corpus index and teacher representation space are held fixed.
\\
\addlinespace[2pt]

Classical kernel and RKHS methodology~\cite{aronszajn1950theory,scholkopf2002learning,hofmann2008kernel}
&
Provides general tools for explicit function spaces, regularization, and approximation analysis.
&
Does not specify a compute-efficient semantic encoder or an operational retrieval-serving path.
&
Connects RKHS-based posterior approximation to semantic prototype mixtures and query-side encoding.
\\
\addlinespace[2pt]

Conditional-mean and kernel-embedding methodology~\cite{song2013kernelembeddings,fukumizu2013kernelbayes}
&
Motivates prediction as conditional-mean or conditional-distribution estimation.
&
Does not by itself yield a KAHM-based deployment-time adapter into frozen neural embeddings.
&
Formulates fixed-teacher semantic encoding as conditional-mean estimation and decomposes the error into posterior-approximation, finite-sample/generalization, and teacher-noise terms.
\\
\addlinespace[2pt]

Efficient retrieval and neural compression
&
Improves retrieval quality--efficiency trade-offs through compact encoders, sparse representations, late interaction, or distillation.
&
Usually retains an online neural query encoder or changes the retrieval architecture being served.
&
Keeps the teacher space, corpus index, lexical front end, supervision granularity, and serving rule fixed to isolate the query-adaptation problem.
\\
\bottomrule
\end{tabular}
\end{table}
\FloatBarrier

\section{Mathematical Formulations}\label{sec_background}
This section fixes the notation needed for the theoretical analysis. Readers primarily interested in the retrieval application may read $x$ as the inexpensive lexical query representation, $z$ as the frozen teacher embedding, $y$ as the semantic-cluster indicator, and $q$ as the law/domain index. The measurable-space formulation is included only to state the later approximation and generalization results precisely.
\subsection{Notations}
We use boldface font to denote the matrices. The following notations are introduced:
\begin{itemize}
\item Let $p,n,N,c,C,q,Q,k,K \in \mathbb{Z}_{+}$ be the positive integers. 
\item For a scalar $a\in \mathbb{R}$, $|a|$ denotes its absolute value. For a set $A$, $|A|$ denotes its cardinality. For a real matrix $\mathbf{X}$, $\mathbf{X}^T$ is the transpose of $\mathbf{X}$. 
\item For a vector $y \in \mathbb{R}^C$, $\| y\|$ denotes the Euclidean norm and $y_j$ (and also $(y)_j$) denotes the $j^{th}$ element. For a matrix $\mathbf{X}\in \mathbb{R}^{N \times n}$, $\|\mathbf{X}\|_2$ denotes the spectral norm, $\| \mathbf{X} \|_F$ denotes the Frobenius norm, $(\mathbf{X})_{i,:}$ denotes the $i^{th}$ row, $(\mathbf{X})_{:,j}$ denotes the $j^{th}$ column, and $(\mathbf{X})_{i,j}$ denotes the $(i,j)^{th}$ element.  
\item The square brackets are used to represent the construction of a matrix from columns e.g. $\left[\begin{IEEEeqnarraybox*}[][c]{,c/c/c,} x^1 & \cdots & x^N \end{IEEEeqnarraybox*} \right]$ is a matrix with vectors $x^1,\cdots,x^N$ as the columns.
\item For a set $\{x^1,\cdots,x^N \} \subset  \mathbb{R}^n$, its affine hull is denoted as $\mathrm{aff}\left(\{x^1,\cdots,x^N \}\right)$. 
\item Let $\mathrm{e}^1,\cdots,\mathrm{e}^C \in \{0,1 \}^C$ be the canonical basis vectors, where 
\begin{equation}
(\mathrm{e}^c)_i  = 
\begin{cases}
1 & \text{if } i  = c,
\\
0 & \text{if } i \neq c.
\end{cases},\; \forall i \in \{1,\cdots,C\}.
\end{equation}
\item Let $\otimes$ denote the product $\sigma-$algebra.
\item The abbreviation \emph{a.s.} stands for \emph{almost surely}, i.e., a property is said to hold a.s. if it holds with probability one.
\end{itemize}
\subsection{Definitions}
\begin{definition}[Document Space and Domain]
Let $(\Omega,\mathcal{F},\mu)$ be a probability space, let $(\mathcal{D},\mathcal{G})$ be a measurable space, and let
\[
D:(\Omega,\mathcal{F}) \to (\mathcal{D},\mathcal{G})
\]
be a measurable random document. Let $\mathbb{P}_{D}:\mathcal{G} \rightarrow [0,1]$ be the distribution of $D$ defined by
\[
\mathbb{P}_{D}:=\mu\circ D^{-1}.
\]
Let
\begin{IEEEeqnarray}{rCl}
\label{eq_100320262058}
f_{D \mapsto q} : (\mathcal{D},\mathcal{G})  & \rightarrow &  (\{1,\cdots,Q\},2^{\{1,\cdots,Q\}})
\end{IEEEeqnarray}
be a measurable domain assignment map. For each $q \in \{1,\dots,Q \}$, the $q$-th domain is defined by
\begin{IEEEeqnarray}{rCl}
\mathcal{D}_q  & := & f_{D \mapsto q}^{-1}(\{q\}) \in  \mathcal{G}.
\end{IEEEeqnarray}
\end{definition}

\begin{definition}[Semantic Space and Clusters]\label{def_120320260618}
We represent documents by a measurable semantic feature map
\begin{IEEEeqnarray}{rCl}
\label{eq_100320261405}
f_{D \mapsto z} : (\mathcal{D},\mathcal{G}) & \rightarrow &  (\mathbb{R}^p,\mathcal{B}(\mathbb{R}^{p})).
\end{IEEEeqnarray}
where $p$ is the dimension of the semantic feature vector and \(\mathcal{B}(\mathbb{R}^{p})\) denotes the Borel \(\sigma\)-algebra on \(\mathbb{R}^{p}\). The associated feature vector
\[
z := f_{D \mapsto z} \circ D
\]
is therefore an \(\mathbb{R}^p\)-valued random vector with distribution
\[
\mathbb{P}_z := \mu \circ z^{-1},
\]
where \(\mathbb{P}_{z}: \mathcal{B}(\mathbb{R}^{p}) \rightarrow [0,1]\). Let
\begin{IEEEeqnarray}{rCl}
\label{eq_100320261904}
f_{z \mapsto c} : (\mathbb{R}^p,\mathcal{B}(\mathbb{R}^{p})) & \rightarrow &  (\{1,\cdots,C\},2^{\{1,\cdots,C\}})
\end{IEEEeqnarray}
be a measurable semantic clustering map. Two documents $d_1,d_2 \in \mathcal{D}$ are said to belong to the same semantic cluster if and only if
\begin{IEEEeqnarray}{rCl}
(f_{z \mapsto c} \circ f_{D \mapsto z})(d_1)& = &  (f_{z \mapsto c} \circ f_{D \mapsto z})(d_2).
\end{IEEEeqnarray}
Let $f_{c \mapsto y}: \{1,\cdots,C\} \rightarrow \{\mathrm{e}^1,\cdots,\mathrm{e}^C \}$ be the one-hot map: 
\begin{IEEEeqnarray}{rCl}
f_{c \mapsto y}(c) & := & \mathrm{e}^c.
\end{IEEEeqnarray}
\end{definition}
\begin{definition}[Lexical Feature Space]\label{def_120320260915}
Let
\begin{IEEEeqnarray}{rCl}
f_{D \mapsto x} : (\mathcal{D},\mathcal{G}) & \rightarrow & (\mathbb{R}^n,\mathcal{B}(\mathbb{R}^n))
\end{IEEEeqnarray}
be a measurable lexical feature map. Define random variables 
\begin{IEEEeqnarray}{rCl}
x & := & f_{D \mapsto x}\circ D,\\
y & := & f_{c \mapsto y} \circ f_{z \mapsto c}\circ f_{D \mapsto z}\circ D,\\
q & := & f_{D \mapsto q}\circ D.
\end{IEEEeqnarray}
Then \((x,y,q)\) is a random vector taking values in
$\mathbb{R}^n \times \{\mathrm{e}^1,\cdots,\mathrm{e}^C\} \times \{1,\cdots,Q\}$ with a distribution, $\mathbb{P}_{x,y,q} :
\mathcal{B}(\mathbb{R}^n)\otimes 2^{\{\mathrm{e}^1,\cdots,\mathrm{e}^C\}}\otimes 2^{\{1,\cdots,Q\}}
\to [0,1]$, given as
\[
\mathbb{P}_{x,y,q} := \mu \circ (x,y,q)^{-1}.
\]
\end{definition}
In our experiments, this abstraction is instantiated in a concrete way. The semantic variable $z$ will be a frozen teacher embedding of a legal retrieval unit, the lexical variable $x$ will be a cheap query-side representation derived from IDF--SVD features, and the domain variable $q$ will index the law-wise local encoder to which the sample belongs. This interpretation is not needed for the mathematical validity of the definitions, but it helps connect the formal setup directly to the retrieval problem studied later in the paper.
\subsection{Data Samples}\label{sec_130420262100}
Let $\mathcal{S}$ be a set consisting of $N$ number of samples drawn i.i.d. according to the distribution $\mathbb{P}_{x,y,q}$:
      \begin{IEEEeqnarray}{rCCCl}
\label{eq_180320261304}   \mathcal{S} & := &  \{(x^i,y^i,q^i) \in \mathbb{R}^n\times \{\mathrm{e}^1,\cdots,\mathrm{e}^C\} \times \{1,\cdots,Q\}  \; \mid \;i \in \{1,2,\cdots,N\} \}  & \sim & (\mathbb{P}_{x,y,q})^N.
      \end{IEEEeqnarray}
Let $\mathcal{I}^{c}$ be the set of indices of those samples in the sequence $\left(\left(x^i,y^i,q^i\right) \in \mathcal{S}\right)_{i=1}^N$ which are $c^{th}$ semantic-cluster labelled, i.e.,
    \begin{IEEEeqnarray}{rCl}
   \mathcal{I}^{c} & : = & \left \{ i \in \{1,2,\cdots,N\} \; \mid \; y^i = \mathrm{e}^c  \right \}. 
      \end{IEEEeqnarray} 
Let $N_c$ be the number of $c^{th}$ semantic-cluster labelled samples, i.e.,
     \begin{IEEEeqnarray}{rCl}
N_c &  = & |\mathcal{I}^{c}|.
      \end{IEEEeqnarray} 
Let $\mathrm{I}^{c} = (\mathrm{I}_1^{c},\cdots,\mathrm{I}_{N_c}^{c})$ be the sequence of elements of $\mathcal{I}^{c}$ in ascending order, i.e., 
\begin{IEEEeqnarray}{rCl}
\mathrm{I}_1^{c} & = & \min(\mathcal{I}^{c}) \\
\mathrm{I}_i^{c} & = & \min(\mathcal{I}^{c} \setminus \{\mathrm{I}_1^{c}, \cdots, \mathrm{I}_{i-1}^{c}\} ),\; \forall i \in \{2,\cdots,N_c\}.
\end{IEEEeqnarray}
Let $\mathbf{X}^c  \in \mathbb{R}^{N_c \times n}$ be the matrix storing $c^{th}$ semantic-cluster labelled samples as its rows, i.e., 
     \begin{IEEEeqnarray}{rCl}
\mathbf{X}^c &  = & \left[x^{\mathrm{I}_1^{c}} \cdots x^{\mathrm{I}_{N_c}^{c}}\right]^T.
      \end{IEEEeqnarray} 
We consider the multi-domain setting where total data samples are distributed among $Q$ different domains. Let $\mathcal{Q}^q$ be the set of semantic-cluster labels covered by $q^{th}$ domain, i.e.,  
\begin{IEEEeqnarray}{rCl}
\mathcal{Q}^q & := & \left \{ c \in \{1,2,\cdots,C\} \; \mid \; (y^i)_c = 1,\;    q^i = q,\; i \in \{1,2,\cdots,N\}  \right \}.
\end{IEEEeqnarray}
Let $\mathcal{J}^q$ be the set of sample indices covered by $q^{th}$ domain, i.e.,   
\begin{IEEEeqnarray}{rCl}
\mathcal{J}^q & := & \left \{ i \in \{1,2,\cdots,N\} \; \mid \;     q^i = q  \right \}.
\end{IEEEeqnarray}
Finally, let  $\mathcal{J}^{q,c}$ (with $c \in \mathcal{Q}^q$) be the set of $c^{th}$ semantic-cluster labelled sample indices covered by $q^{th}$ domain, i.e.,  
\begin{IEEEeqnarray}{rCl}
\mathcal{J}^{q,c} & := & \left \{ i \in \{1,2,\cdots,N\} \; \mid \;     q^i = q,\; (y^i)_c = 1,\; c \in  \mathcal{Q}^q  \right \}.
\end{IEEEeqnarray}
\begin{remark}[Data Heterogeneity across Domains]\label{rem_111320251253}
We assume that data are statistically heterogeneous across domains, i.e., for arbitrary domains $q^i$ and $q^j$ with $i \neq j$,
 \begin{IEEEeqnarray}{rCl}
 \mathbb{P}_{x,y| q}(\cdot,\cdot | q = q^i) & \neq &  \mathbb{P}_{x,y| q}(\cdot,\cdot | q = q^j) ,\\
 \mathbb{P}_{y| x, q}(\cdot | x, q = q^i) & \neq &  \mathbb{P}_{y| x, q}(\cdot | x, q = q^j).
 \end{IEEEeqnarray} 
 \end{remark}
\subsection{Kernel Affine Hull Machine (KAHM)}  
The KAHMs, originally defined in \cite{KAHM}, have been considered for automated machine learning in \cite{kumar2024geometricallyinspiredkernelmachines} and for gradient-free federated learning in \cite{kumar2025operatortheoreticframeworkgradientfreefederated}. Given a finite number of samples: $\mathbf{X} = \left[\begin{IEEEeqnarraybox*}[][c]{,c/c/c,} x^1 & \cdots & x^N \end{IEEEeqnarraybox*} \right]^T$ with $x^1,\cdots,x^N \in \mathbb{R}^n$, a KAHM $\mathcal{A}_{\mathbf{X}}: \mathbb{R}^n \rightarrow \mathrm{aff}(\{x^1,\cdots,x^N \})$ is defined as
  \begin{IEEEeqnarray}{rCl}
\label{eq_220420251843}\mathcal{A}_{\mathbf{X}}(x) & := & \frac{h_{\mathbf{X}}^1(\mathbf{P}_{\mathbf{X}}x)}{\sum_{i=1}^Nh_{\mathbf{X}}^i(\mathbf{P}_{\mathbf{X}}x)}x^1 + \cdots + \frac{h_{\mathbf{X}}^N(\mathbf{P}_{\mathbf{X}}x)}{\sum_{i=1}^Nh_{\mathbf{X}}^i(\mathbf{P}_{\mathbf{X}}x)}x^N.
   \end{IEEEeqnarray}  
A comprehensive description of the variables and functions associated with KAHM expression (\ref{eq_220420251843}) has been provided in Appendix A.

\subsection{The Semantic Feature Vector Estimation Problem}
We consider the problem of estimating the semantic feature vector of a document from its lexical feature vector. For this, the solution space is defined as
\begin{IEEEeqnarray}{rCl}
\mathbf{F}_{x \mapsto z} & := & \left \{ f : \mathbb{R}^n \rightarrow \mathbb{R}^p \; \mid \; f \mbox{ is } \mathcal{B}(\mathbb{R}^{n})-\mathcal{B}(\mathbb{R}^{p})-\mbox{measurable and}  \mathop{\mathbb{E}}_{x \sim \mathbb{P}_{x }} [ \|f(x) \|^2  ] < \infty   \right \}. 
\end{IEEEeqnarray}
Let $f_{x \mapsto z}: \mathbb{R}^n \rightarrow \mathbb{R}^p$ be a Borel measurable function satisfying
\begin{IEEEeqnarray}{rCl}
\label{eq_170320261347} f_{x \mapsto z}(x) & = & \mathop{\mathbb{E}}_{z \sim \mathbb{P}_{z \mid x}}[z \mid x]  \qquad\text{a.s.}
\end{IEEEeqnarray}
Under the assumption
\begin{IEEEeqnarray}{rCl}
 \mathop{\mathbb{E}}_{z \sim \mathbb{P}_{z}} [ \| z\|^2  ] & < & \infty,
\end{IEEEeqnarray}
it is well known and also shown in Appendix B for the sake of completeness that $f_{x \mapsto z}$ solves the least-squares estimation problem, i.e., 
\begin{IEEEeqnarray}{rCl}
\label{eq_170320261338} f_{x \mapsto z} & \in & \mathop{\argmin}_{f \in \mathbf{F}_{x \mapsto z}} \mathop{\mathbb{E}}_{(x,z) \sim \mathbb{P}_{x,z} } [\| z-f(x) \|^2  ].
\end{IEEEeqnarray}
Since $y$ is a discrete random vector taking values in $\{\mathrm{e}^1,\dots,\mathrm{e}^C\}$, the law of total expectation yields
\begin{IEEEeqnarray}{rCCCl}
\label{eq_190320261350} f_{x \mapsto z}(x) & = & \mathop{\mathbb{E}}_{z \sim \mathbb{P}_{z \mid x}}[z \mid x] & = & \sum_{c=1}^C \mathbb{P}_{ y\mid x}(y_c = 1 \mid x) \mathop{\mathbb{E}}_{z \sim \mathbb{P}_{z \mid  x, y}}[z \mid x, y_c=1] \qquad \text{a.s.}
\end{IEEEeqnarray}
\begin{definition}[Semantic Cluster Prototypes and Within-Cluster Residuals]
\label{def:prototypes_residuals}
For each semantic cluster $c \in \{1,\dots,C\}$, define the \emph{cluster prototype} as the cluster mean,
\begin{IEEEeqnarray}{rCl}
w^c & := & \mathop{\mathbb{E}}_{z \sim \mathbb{P}_{z \mid y}}\!\left[z \mid
            y_c = 1\right] \;\in\; \mathbb{R}^p,
\end{IEEEeqnarray}
the \emph{within-cluster residual} as the deviation from the prototype,
\begin{IEEEeqnarray}{rCl}
\Delta w^c(x) & := & \mathop{\mathbb{E}}_{z \sim \mathbb{P}_{z \mid x,y}}
              \!\left[z \mid x,\, y_c=1\right] - w^c,
\end{IEEEeqnarray}
and the \emph{aggregate within-cluster residual},
\begin{IEEEeqnarray}{rCl}
\label{eq_agg_residual}
\xi_{\mathrm{res}}(x) & := & \sum_{c=1}^C \mathbb{P}_{y \mid x}(y_c=1 \mid x)\, \Delta w^c(x)
               \;\in\; \mathbb{R}^p.
\end{IEEEeqnarray}
\end{definition}

\begin{remark}[Structural properties of the residual]
By definition of $w^c$,
\begin{IEEEeqnarray}{rCl}
\mathop{\mathbb{E}}_{x \sim \mathbb{P}_{x}} \!\left[\Delta w^c(x)\right] & = & \mathbf{0},
\end{IEEEeqnarray}
where $\mathbf{0}$ is the zero-vector of suitable dimension. 
\end{remark}
\begin{remark}[Exact decomposition of the conditional mean]
\label{rem:exact_decomposition}
Using Definition~\ref{def:prototypes_residuals} and the law of total expectation (\ref{eq_190320261350}), the conditional mean of $z$ given $x$ decomposes exactly as
\begin{IEEEeqnarray}{rCl}
\label{eq_230320261824}
f_{x \mapsto z}(x) & = & \underbrace{\sum_{c=1}^C \mathbb{P}_{y \mid x}(y_c=1 \mid x)\,w^c}_{%
\text{prototype-mixture term}} +\; \xi_{\mathrm{res}}(x)  \quad \text{a.s.},
\end{IEEEeqnarray}
where the quantity $\mathop{\mathbb{E}}_{x \sim \mathbb{P}_{x}} \!\left[ \left \|\xi_{\mathrm{res}}(x) \right \|^2\right] $ measures the within-cluster variance that the prototype-mixture representation leaves unexplained.
\end{remark}
We are now in a position to state the central problem of this study:
\begin{problem}[The Encoding Problem]\label{problem_encoding}
Let $\{(x^i,y^i)\}_{i=1}^N$ be the observed samples. Using the decomposition~(\ref{eq_230320261824}) in Remark~\ref{rem:exact_decomposition}, the goal is to construct an estimator $\widehat{f}_{x \mapsto z}: \mathbb{R}^n \rightarrow \mathbb{R}^p$ of the form
\begin{IEEEeqnarray}{rCl}
\widehat{f}_{x \mapsto z}(x) & : = & \sum_{c=1}^C \Phi_c(x) \widehat{w}^c,
\end{IEEEeqnarray}
where $\widehat{w}^c \in \mathbb{R}^p$ and
\begin{IEEEeqnarray}{rCl}
\Phi_c(x) & \in & [0,1], \quad \sum_{c=1}^C\Phi_c(x) = 1.
\end{IEEEeqnarray}
Considering $\widehat{f}_{x \mapsto z}$ as an approximation of $f_{x\mapsto z}(x)=\mathop{\mathbb{E}}_{z \sim \mathbb{P}_{z \mid x}}[z \mid x] $, specifically,  
\begin{enumerate}
\item derive an estimator $\Phi_c$ of posterior probability $\mathbb{P}_{y \mid x}((y)_c=1 \mid x)$ for each $c \in \{1,\cdots,C \}$, and establish non-asymptotic bounds on its estimation error: $\mathop{\mathbb{E}}_{x \sim \mathbb{P}_x}\left[ \left| \Phi_c(x) - \mathbb{P}_{y \mid x}(y_c = 1 | x) \right|^2\right]$; 
\item derive an estimator $\widehat{w}^c$ of prototype $w^c$ for each $c \in \{1,\cdots,C \}$ from the samples $\{(\Phi_c(x^i),v^i) \}_{i=1}^N$ generated by a ``noisy teacher'' such that
\begin{IEEEeqnarray}{rCl}
v^i & = & \sum_{c=1}^C \mathbb{P}_{y \mid x}(y_c=1 \mid x^i)\,w^c     + \xi(x^i),   \qquad i \in \{1,\cdots,N \},
\label{eq_240320261527}
\end{IEEEeqnarray}
where the noise function $\xi:\mathbb{R}^n\to\mathbb{R}^p$ is defined as
\begin{IEEEeqnarray}{rCl}
\label{eq_noise_decomp}
\xi(x) & := & \xi_{\mathrm{res}}(x) + \xi_{\mathrm{obs}}(x),
\end{IEEEeqnarray}
with $\xi_{\mathrm{res}}(x) $ the aggregate within-cluster residual from Definition~\ref{def:prototypes_residuals} and $\xi_{\mathrm{obs}}(x)$ the residual perturbation in the observed teacher embedding relative to $f_{x \mapsto z}(x)$.
\end{enumerate}
Here, $\xi(x^i)$ subsumes two sources of noise. The term $\xi_{\mathrm{res}}(x^i)$ is the systematic within-cluster variation: the part of the conditional mean $f_{x \mapsto z}(x^i)$ that the prototype-mixture representation cannot recover. The term $\xi_{\mathrm{obs}}(x^i)$ captures perturbations in the observed teacher embedding around the true conditional mean. 
\end{problem}
\section{The Encoding Solution}\label{sec_encoding_solution}
This section develops the proposed encoder in four steps. First, we formulate deployment-time semantic encoding as the problem of predicting a frozen teacher embedding from inexpensive lexical query features. Second, we represent that target embedding as a noisy prototype mixture weighted by semantic-cluster posterior probabilities. 
Third, we construct a KAHM-based estimator for these posterior probabilities whose weighting step is gradient-free in the sense that it requires no backpropagation through a parameterised neural network and no iterative gradient computation over a composite neural objective: the KAHM geometry supplies cluster-wise space-folding scores, these scores are normalized into posterior-like coefficients, and the resulting coefficients are then used as the mixture weights. Fourth, once these posterior-like coefficients have been obtained, we estimate the semantic prototypes by explicit NLMS updates learned from noisy teacher embeddings and derive an end-to-end error decomposition into posterior-approximation, finite-sample/generalization, and teacher-noise terms. Operationally, this means that the online path contains no transformer forward pass and no learned neural gating network: it consists only of lexical feature extraction, evaluation of the KAHM-induced posterior scores, formation of the weighted prototype mixture, and retrieval in the frozen dense index.
\subsection{Analytically Identified Hypothesis Space for Learning Semantic-Cluster Posterior Probability}
The next step uses a hypothesis-space identification result from \cite{kumar2025operatortheoreticframeworkgradientfreefederated} only as a technical device for the present analysis. We use the identified data-dependent function class to make the posterior-surrogate estimation problem analytically explicit in the fixed-teacher setting studied in this paper. Following that result of \cite{kumar2025operatortheoreticframeworkgradientfreefederated}, we work with the hypothesis space
\begin{IEEEeqnarray}{rCl}
\label{eq_150220251829} \mathcal{M}_{c} & := & \left \{ h_{x \mapsto y_c} =   \frac{\Phi_c(\cdot)}{N_c} \sum_{i=1}^{N_c} \Phi_c(x^{\mathrm{I}_i^{c}})  \; \mid \; \Phi_c: \mathbb{R}^n \rightarrow [0,1],\; \| \Phi_c \|_{L^2(\mathbb{R}^n,\mathbb{P}_{x})}^2 = \frac{N_c}{N}         \right \}, 
   \end{IEEEeqnarray}
where $\Phi_c$ is referred to as feature-map and $L^2(\mathbb{R}^n,\mathbb{P}_{x})$ is the space of all complex-valued measurable functions on $\mathbb{R}^n$ such that for a $f \in L^2(\mathbb{R}^n,\mathbb{P}_{x})$,
   \begin{IEEEeqnarray}{rCCCl}
\| f \|_{L^2(\mathbb{R}^n,\mathbb{P}_{x})}^2 & : = & \int_{\mathbb{R}^n} |f(x)|^2 \diff{\mathbb{P}_{x}}(x) & < & \infty.
        \end{IEEEeqnarray}
The factor $N_c/\sum_{i=1}^{N_c}\Phi_c(x^{\mathrm{I}_i^{c}})$ rescales the geometric score function $\Phi_c$ into a hypothesis that is comparable to a posterior-probability surrogate on the $c$-labelled region. Intuitively, $\Phi_c$ first measures how strongly the KAHM geometry associates a query with semantic cluster $c$, and the normalization then calibrates that score by the empirical mass that $\Phi_c$ assigns to the observed $c$-labelled samples. This data-dependent rescaling is important because it converts a cluster-specific geometric compatibility score into a function with the correct problem-dependent scale for the posterior-estimation space $\mathcal{M}_c$, rather than treating $\Phi_c$ as an arbitrary uncalibrated score.
\begin{theorem}[Approximation Error Bound for Hypothesis Space \cite{kumar2025operatortheoreticframeworkgradientfreefederated}]\label{theorem_120220251923}
Given a data set $\{(x^i,y^i)  \; \mid \;i \in \{1,2,\cdots,N\} \}   \sim  (\mathbb{P}_{x,y})^N$, for any $h_{x \mapsto y_c} \in \mathcal{M}_{c}$, we have with probability at least $1-\delta$ for any $\delta \in (0,1)$,
\begin{IEEEeqnarray}{rCl}
\nonumber \mathop{\mathbb{E}}_{x \sim \mathbb{P}_x}\left[ \left| h_{x \mapsto y_c}(x) - \mathbb{P}_{y | x}(y_c = 1 | x) \right|^2\right] & \leq & \min\left( \frac{1}{N} \sum_{i=1}^N |y_c^i - h_{x \mapsto y_c}(x^i) |^2 +  \frac{4}{\sqrt{N}} +  \sqrt{\frac{\log(1/\delta)}{2N}}, \right. \\
\label{eq_130220250858}&& \left. \frac{1}{\left(N_c/N \right)^2} \left( \frac{3}{\sqrt{N}} + \sqrt{\frac{8  \log(1/\delta)}{N}} \right)  \right).
\end{IEEEeqnarray}
\end{theorem}
\subsection{KAHM-Based Learning Solution}
Our approach is to choose from the space $\mathcal{M}_c$ a suitable hypothesis that can be implemented efficiently. The previous study \cite{kumar2025operatortheoreticframeworkgradientfreefederated} has exploited the so-called ``{\em space folding}'' property of the KAHMs to implement a solution, however, under certain assumptions. Based on the observation that a KAHM (associated to a given set of data samples) folds any arbitrary point in the data space around the data samples, the study in \cite{kumar2025operatortheoreticframeworkgradientfreefederated} has introduced a ``{\em space folding measure}'':          
\begin{definition}[Space Folding Measure~\cite{kumar2025operatortheoreticframeworkgradientfreefederated}]\label{definition_150220251240}
To evaluate the amount of folding required for an arbitrary point $x$ to map it (by the KAHM $\mathcal{A}_{\mathbf{X}}$) to a point closer to data samples represented by a matrix $\mathbf{X}$, the space folding measure, $\mathcal{T}_{ \mathbf{X}}: \mathbb{R}^n \rightarrow [0,1]$, associated to data samples $\mathbf{X}$, is defined as
 \begin{IEEEeqnarray}{rCl}
\label{eq_170220251438} \mathcal{T}_{ \mathbf{X}}(x) & := & 
 \sqrt{\frac{1}{2} \left(\left|\mathcal{T}_{ \mathbf{X}}^{\text{Euc}}(x) \right|^2 + \left| \mathcal{T}_{ \mathbf{X}}^{\text{Cos}}(x) \right|^2\right)}, \quad \text{where}\\
\label{eq_090420261838} \mathcal{T}_{ \mathbf{X}}^{\text{Euc}}(x) &:= & 1 - \exp\left( - \|x - \mathcal{A}_{\mathbf{X}}(x) \|  \right) \\
\label{eq_090420261839} \mathcal{T}_{ \mathbf{X}}^{\text{Cos}}(x) & := &  \frac{1}{\pi} \arccos{\left(\frac{ \left(\mathcal{A}_{\mathbf{X}}(x)\right)^T x}{ \|\mathcal{A}_{\mathbf{X}}(x) \| \| x\| }\right)}. 
    \end{IEEEeqnarray}  
The two components in~(\ref{eq_090420261838})--(\ref{eq_090420261839}) capture complementary aspects of geometric mismatch. The Euclidean term $\mathcal{T}_{\mathbf{X}}^{\text{Euc}}(x)$ measures how far the point must be moved in magnitude to reach its KAHM reconstruction, whereas the cosine term $\mathcal{T}_{\mathbf{X}}^{\text{Cos}}(x)$ measures how strongly the direction of that reconstruction differs from the direction of the original point. Thus, the combined score does not depend only on distance, as in a purely Euclidean nearest-neighbor view, but also on orientation relative to the affine structure induced by the KAHM model.
\end{definition}
The quantity $\mathcal{T}_{\mathbf{X}}(x)$ measures how much geometric correction the KAHM reconstruction must apply in order to fold $x$ toward the reference set $\mathbf{X}$. If this value is small, then $x$ already lies in a region that is geometrically compatible with the cluster represented by $\mathbf{X}$, so only a small correction is needed; if it is large, then the point is poorly aligned with that cluster and requires a stronger fold. In the later law-wise serving rule, this is precisely why smaller folding scores are treated as stronger evidence for the corresponding local model. Our intent is to utilise the space-folding property of KAHMs to define a hypothesis, possibly within $\mathcal{M}_c$, so that Theorem~\ref{theorem_120220251923} can be directly applied to provide theoretical guarantees on the performance. However, unlike \cite{kumar2025operatortheoreticframeworkgradientfreefederated}, the present construction is not restricted to binary outputs (i.e. $\{0,1\}$). The broader semantic-encoding setting therefore requires an explicit structural assumption on the concentration of the induced feature map, which we state next and later interpret operationally through empirical diagnostics.
\begin{definition}[Ordered $K$-Minimum Semantic Cluster Sequence]
Let $\left\{\mathcal{T}_{\mathbf{X}^c}:\mathbb{R}^n\to[0,1] \right\}_{c=1}^C$ be a family of functions, where $\mathcal{T}_{\mathbf{X}^c}$ is the space folding measure associated to $c^{th}$ semantic cluster. For a given $x$, let $\pi_x$ denote the permutation of $\{1,\dots,C\}$ such that
\[
\mathcal{T}_{ \mathbf{X}^{\pi_x(1)}}(x)\le \cdots \le \mathcal{T}_{ \mathbf{X}^{\pi_x(C)}}(x).
\]
Then, for $K\in\{1,\dots,C\}$, the ordered $K$-minimum semantic cluster sequence at $x$ is defined by
\[
\mathrm{J}_K(x):=\bigl(\pi_x(1),\dots,\pi_x(K)\bigr).
\]
The entries of $\mathrm{J}_K(x)$ are precisely the indices of the $K$ smallest values among $\mathcal{T}_{\mathbf{X}^1}(x),\cdots,\mathcal{T}_{\mathbf{X}^C}(x)$ listed in non-decreasing order.
\end{definition}
\begin{definition}[KAHM-Induced Feature-Map] \label{def_030420261238}
The KAHM-induced feature-map function, $\Phi_c: \mathbb{R}^n \rightarrow [0,1]$, for a given $\omega > 1$ and $K \in \{1,\dots,C\}$, is defined as
\begin{IEEEeqnarray}{rCl}
\Phi_c(x) & := &  \begin{cases} \displaystyle 
\frac{ \left(1 - \mathcal{T}_{\mathbf{X}^c}(x) \right)^{\omega}}{ \sum_{ k \in \mathrm{J}_K(x)} \left(1 - \mathcal{T}_{\mathbf{X}^k}(x) \right)^{\omega}} &  c \in \mathrm{J}_K(x), \\
0 &  c \notin \mathrm{J}_K(x)
\end{cases}, \quad  c \in \{1,\cdots,C\}.
\end{IEEEeqnarray}
\end{definition}
Here, $\omega>1$ is a temperature-like hyperparameter that controls the sharpness of the cluster assignment induced by the KAHM folding scores. Larger values of $\omega$ concentrate more mass on the clusters with the smallest folding scores, whereas smaller admissible values produce softer assignments across competing clusters. In the limiting regime $\omega \to \infty$, the feature map approaches the hard-assignment rule used in \cite{kumar2025operatortheoreticframeworkgradientfreefederated}.
\begin{assumption}[Cluster-Wise $L^2$-Mass Condition for the Feature-Map]\label{ass:feature_map_concentartion}
For each semantic cluster $c \in \{1,\cdots,C\}$, the KAHM-induced feature-map $\Phi_c$ satisfies
\begin{IEEEeqnarray}{rCl}
\label{eq_310320262030}\|\Phi_c\|_{L^2(\mathbb{R}^n,\mathbb{P}_x)}^2
&\geq&
\frac{N_c}{N}.
\end{IEEEeqnarray}
\end{assumption}
\begin{remark}[Geometric Motivation for Assumption~\ref{ass:feature_map_concentartion}]
Assumption~\ref{ass:feature_map_concentartion} is motivated by the geometric interpretation of the space folding measure. For a training sample $x^{\mathrm{I}_i^{c}}$ labelled by semantic cluster $c$, one expects the space folding measure induced by $\mathbf{X}^c$ to be relatively small compared with the corresponding folding measures induced by competing semantic clusters. Consequently, the feature-map weight assigned to cluster $c$ is expected to be large on $c$-labelled samples and, under sufficiently pronounced separation of the competing folding scores,
can become close to $1$ for sufficiently large $\omega$. Hence, for $c$-labelled training samples, one expects
\[
\Phi_c(x^{\mathrm{I}_i^{c}}) \approx 1,
\qquad i=1,\dots,N_c.
\]
If, in addition, the sample size $N$ is sufficiently large, then the empirical
average approximates the population $L^2$-norm, so that
\[
\|\Phi_c\|_{L^2(\mathbb{R}^n,\mathbb{P}_x)}^2
\approx
\frac{1}{N}\sum_{i=1}^N |\Phi_c(x^i)|^2
\ge
\frac{1}{N}\sum_{i=1}^{N_c} |\Phi_c(x^{\mathrm{I}_i^{c}})|^2
\approx
\frac{N_c}{N}.
\]
This provides the intuition behind Assumption~\ref{ass:feature_map_concentartion}, namely that the KAHM-induced feature map $\Phi_c$ carries sufficient $L^2$-mass on the $c$-labelled region.
\end{remark}
An empirical study will be performed to evaluate a sample analogue of Assumption~\ref{ass:feature_map_concentartion} on the realized law-wise KAHM-induced feature-maps in order to check whether the required cluster-wise $L^2$-mass condition is approximately satisfied in the trained deployment regime.
\begin{theorem}[Approximation Error Bound for KAHM-Induced Feature-Map]\label{theorem_300320262124}
Given a data set $\{(x^i,y^i)  \; \mid \;i \in \{1,2,\cdots,N\} \}   \sim  (\mathbb{P}_{x,y})^N$, under the assumption of inequality (\ref{eq_310320262030}), we have with probability at least $1-\delta$ for any $\delta \in (0,1)$ and $\eta > 0$,
\begin{IEEEeqnarray}{rCl}
\nonumber \lefteqn{\mathop{\mathbb{E}}_{x \sim \mathbb{P}_x}\left[ \left| \Phi_c(x) - \mathbb{P}_{y | x}(y_c = 1 | x) \right|^2\right] } \\ \nonumber & \leq & (1+\eta) \left | 1- \frac{\sum_{i=1}^{N_c} \Phi_c(x^{\mathrm{I}_i^{c}})}{N \| \Phi_c \|_{L^2(\mathbb{R}^n,\mathbb{P}_{x})}^2}  \right |^2 \| \Phi_c \|_{L^2(\mathbb{R}^n,\mathbb{P}_{x})}^2  \\
&& {+}\: (1+\eta^{-1}) \min\left( \frac{1}{N} \sum_{i=1}^N |y_c^i - \widetilde{h}_{x \mapsto y_c}(x^i) |^2 +  \frac{4}{\sqrt{N}} +  \sqrt{\frac{\log(1/\delta)}{2N}},  \frac{1}{\left(N_c/N \right)^2} \left( \frac{3}{\sqrt{N}} + \sqrt{\frac{8  \log(1/\delta)}{N}} \right)  \right),
\end{IEEEeqnarray}
where \begin{IEEEeqnarray}{rCl}
\label{eq_020420261031}\widetilde{h}_{x \mapsto y_c} & := &   \frac{\widetilde{\Phi}_c(\cdot)}{N_c} \sum_{i=1}^{N_c} \widetilde{\Phi}_c(x^{\mathrm{I}_i^{c}}), \\
\widetilde{\Phi}_c  &  : = & \frac{\sqrt{N_c}}{\sqrt{N} \| \Phi_c \|_{L^2(\mathbb{R}^n,\mathbb{P}_{x})}} \Phi_c.
\end{IEEEeqnarray}
\end{theorem}
\begin{proof}
Consider
\begin{IEEEeqnarray}{rCl}
\label{eq_310320261821}\| \Phi_c - \widetilde{h}_{x \mapsto y_c} \|_{L^2(\mathbb{R}^n,\mathbb{P}_{x})} & = & \left | 1- \frac{\sum_{i=1}^{N_c} \Phi_c(x^{\mathrm{I}_i^{c}})}{N \| \Phi_c \|_{L^2(\mathbb{R}^n,\mathbb{P}_{x})}^2}  \right | \| \Phi_c \|_{L^2(\mathbb{R}^n,\mathbb{P}_{x})}.
\end{IEEEeqnarray}
Due to the inequality~(\ref{eq_310320262030}) and $\Phi_c(x) \in [0,1]$, we have $\widetilde{\Phi}_c(x) \in [0,1]$, and further 
\begin{IEEEeqnarray}{rCl}
 \| \widetilde{\Phi}_c \|_{L^2(\mathbb{R}^n,\mathbb{P}_{x})}^2 & = & \frac{N_c}{N}.
\end{IEEEeqnarray}
Therefore, $\widetilde{h}_{x \mapsto y_c} \in \mathcal{M}_c$, and Theorem~\ref{theorem_120220251923} yields, with probability at least $1-\delta$,
\begin{IEEEeqnarray}{rCl}
\nonumber \lefteqn{\mathop{\mathbb{E}}_{x \sim \mathbb{P}_x}\left[ \left| \widetilde{h}_{x \mapsto y_c}(x) - \mathbb{P}_{y | x}(y_c = 1 | x) \right|^2\right]} \\
\label{eq_310320261824} & \leq & \min\left( \frac{1}{N} \sum_{i=1}^N |y_c^i - \widetilde{h}_{x \mapsto y_c}(x^i) |^2 +  \frac{4}{\sqrt{N}} +  \sqrt{\frac{\log(1/\delta)}{2N}},  \frac{1}{\left(N_c/N \right)^2} \left( \frac{3}{\sqrt{N}} + \sqrt{\frac{8  \log(1/\delta)}{N}} \right)  \right).
\end{IEEEeqnarray}
By the triangular inequality in $L^2(\mathbb{R}^n,\mathbb{P}_{x})$, we have
\begin{IEEEeqnarray}{rCl}
\|\Phi_c - \mathbb{P}_{y | x}(y_c = 1 | \cdot) \|_{L^2(\mathbb{R}^n,\mathbb{P}_{x})} & \leq & \| \Phi_c - \widetilde{h}_{x \mapsto y_c} \|_{L^2(\mathbb{R}^n,\mathbb{P}_{x})} + \|\widetilde{h}_{x \mapsto y_c} - \mathbb{P}_{y | x}(y_c = 1 | \cdot) \|_{L^2(\mathbb{R}^n,\mathbb{P}_{x})}.
\end{IEEEeqnarray}
Therefore, using $(a+b)^2 \leq (1+\eta)a^2 + (1+\eta^{-1})b^2$, it follows that
\begin{IEEEeqnarray}{rCl}
\nonumber \lefteqn{\|\Phi_c - \mathbb{P}_{y | x}(y_c = 1 | \cdot) \|_{L^2(\mathbb{R}^n,\mathbb{P}_{x})}^2} \\
\label{eq_310320261830} & \leq & (1+\eta) \| \Phi_c - \widetilde{h}_{x \mapsto y_c} \|_{L^2(\mathbb{R}^n,\mathbb{P}_{x})}^2 + (1+\eta^{-1})\|\widetilde{h}_{x \mapsto y_c} - \mathbb{P}_{y | x}(y_c = 1 | \cdot) \|_{L^2(\mathbb{R}^n,\mathbb{P}_{x})}^2.
\end{IEEEeqnarray}
Using (\ref{eq_310320261821}) and (\ref{eq_310320261824}) in (\ref{eq_310320261830}), the result is obtained.
\end{proof}
Crucially, Theorem~\ref{theorem_300320262124} shows that the finite-sample contribution to the approximation error of the KAHM-induced feature map decays at the standard $\mathcal{O}(N^{-1/2})$ rate, up to either the empirical approximation term or the class-balance-controlled alternative term of the bound, depending on which term in the minimum is smaller. Thus, provided that the relevant branch remains controlled, the lightweight KAHM-based posterior estimator becomes increasingly accurate as the number of document samples grows. Here, $\eta>0$ is an auxiliary constant that serves as a free balancing parameter, in particular, it may be chosen to optimize the trade-off between the two terms in the bound and thus to minimize the resulting upper estimate.
\subsection{Adaptive Estimation of Semantic-Cluster Prototypes}
The semantic-cluster prototypes must be estimated from the available data samples $\{(\Phi_c(x^i),v^i)\}_{i=1}^N$ while addressing the challenge that semantic feature vectors, that have been generated using a noisy teacher model, are uncertain. Specifically, the noise $\xi(x^i)$ in (\ref{eq_240320261527}) depends on the lexical features $x^i$, so noise samples are not independent of the regressor $\Phi_c(x^i)$. Therefore, a standard mathematical analysis cannot be performed assuming that noise samples are independent from regressor vector $\Phi_c(x^i)$. Since no prior knowledge about noise statistics is available, our approach is to employ an adaptive filtering algorithm for estimating semantic-cluster prototypes that provides robustness against the noise. In this study, we estimate the cluster prototypes using the Normalized Least-Mean Squares (NLMS) algorithm which is known to be $H^{\infty}$ optimal~\cite{Hassibi1996HinfLMS} and thus is robust in the sense that it minimizes the maximum possible value of the gain in energy from the noise to the filtering error. Beyond robustness, this choice is also computationally attractive: the NLMS refinement updates the prototype coefficients recursively with cost linear in the number of semantic clusters $C$ for each output coordinate, rather than requiring the repeated solution of a global least-squares problem over all prototypes. This makes the prototype-refinement step consistent with the overall compute-efficient design of the encoder. It follows from (\ref{eq_240320261527}) that for every $j \in \{1,\cdots,p \}$,
\begin{IEEEeqnarray}{rCl}
(v^i)_j & = & \sum_{c=1}^C \mathbb{P}_{ y\mid x}(y_c = 1 \mid x) (w^c)_j + (\xi(x^i))_j.
\end{IEEEeqnarray}
Define $\alpha^j \in \mathbb{R}^C$, $\widehat{\alpha}^j \in \mathbb{R}^C$, $G: \mathbb{R}^n \rightarrow [0,1]^C$, and $\Delta: \mathbb{R}^n \rightarrow [-1,1]^C$ as
\begin{IEEEeqnarray}{rCl}
\alpha^j &  = & \left[\begin{IEEEeqnarraybox*}[][c]{,c/c/c,} (w^1)_j & \cdots & (w^C)_j \end{IEEEeqnarraybox*} \right]^T \\
\widehat{\alpha}^j &  = & \left[\begin{IEEEeqnarraybox*}[][c]{,c/c/c,} (\widehat{w}^1)_j & \cdots & (\widehat{w}^C)_j \end{IEEEeqnarraybox*} \right]^T \\
G(x) &  = & \left[\begin{IEEEeqnarraybox*}[][c]{,c/c/c,} \Phi_1(x) & \cdots & \Phi_C(x) \end{IEEEeqnarraybox*} \right] \\
\Delta(x) &  = & \left[\begin{IEEEeqnarraybox*}[][c]{,c/c/c,} \mathbb{P}_{ y\mid x}(y_1 = 1 \mid x) - \Phi_1(x) & \cdots & \mathbb{P}_{ y\mid x}(y_C = 1 \mid x) - \Phi_C(x) \end{IEEEeqnarraybox*} \right], 
\end{IEEEeqnarray}
to express
\begin{IEEEeqnarray}{rCl}
\label{eq_260320261248}(v^i)_j & = & G(x^i) \alpha^j + \Delta(x^i) \alpha^j   + (\xi(x^i))_j. 
\end{IEEEeqnarray}
Here, $\Delta(x)$ represents the posterior-estimation error introduced by the KAHM step: it measures, component-wise, the discrepancy between the true semantic-cluster posterior probabilities $\mathbb{P}_{y\mid x}(y_c=1\mid x)$ and their KAHM-induced approximations $\Phi_c(x)$. By contrast, $\xi(x)$ is the teacher-noise term, i.e.\ the residual part of the semantic target that is not explained by the cluster-prototype representation, including perturbations in the observed teacher embedding and residual within-cluster variation. This distinction makes the subsequent error decomposition explicit: $\Delta(x)$ captures posterior mismatch, whereas $\xi(x)$ captures noise inherited from the teacher signal.

The vector $\alpha^j$ is estimated recursively from data samples $\{(G(x^i),(v^i)_j)\}_{i=1}^N$ using the NLMS algorithm. The $i^{th}$ recursion of the NLMS algorithm~\cite{Hassibi1996HinfLMS} is given as
\begin{IEEEeqnarray}{rCl}
\label{eq_260320261120}\left. \widehat{\alpha}^j \right |_{i} & = & \left. \widehat{\alpha}^j \right |_{i-1} + \beta \frac{(v^i)_j - G(x^i) \left. \widehat{\alpha}^j \right |_{i-1} }{1 + \beta \| G(x^i)\|^2} (G(x^i))^T,\quad i \in \{1,\cdots,N\}, 
\end{IEEEeqnarray}
where $\left. \widehat{\alpha}^j \right |_{0}$ is an initial guess, $0 < \beta < 1$ is the step-size, and the estimated parameters are given as
\begin{IEEEeqnarray}{rCl}
\label{eq_270320260656}\widehat{\alpha}^j & := & \left. \widehat{\alpha}^j  \right |_{N}.
\end{IEEEeqnarray}
In adaptive-filter terms, the step-size parameter $\beta$ acts as a regularization parameter that balances adaptation speed against sensitivity to teacher noise: larger values permit more aggressive prototype updates, whereas smaller admissible values yield more conservative and noise-robust adaptation. In the present setting, this normalization is also especially stable because $G(x)$ is a posterior-like probability vector, so all of its entries lie in $[0,1]$ and sum to unity, which implies $\|G(x^i)\|_2^2 \le 1$. Consequently, the NLMS denominator remains uniformly well controlled, and the practical role of $\beta$ is primarily to tune the adaptation--noise trade-off rather than to compensate for large fluctuations in the regressor norm.
\begin{remark}[Sense in Which the Encoder Is Gradient-Free]
\label{rem:gradient-free}
Throughout this paper, the term \emph{gradient-free} refers specifically to the KAHM-based posterior-weight estimation step, which requires no backpropagation through a parameterised neural network and no iterative gradient computation over a composite neural loss. The NLMS prototype-refinement step in~(\ref{eq_260320261120}) is, mathematically, a normalised stochastic-gradient update on an instantaneous squared-error objective. Accordingly, the full pipeline is not gradient-free in the strongest optimisation-theoretic sense. Its practical distinction is instead that it is \emph{neural-network-free} and \emph{backpropagation-free}: posterior weights are obtained by explicit KAHM geometry, prototype refinement is performed by an explicit adaptive update, and deployment requires no transformer forward pass and no learned neural gating network.
\end{remark}
\subsection{Performance Analysis}
\begin{assumption}[Steady-State of the Estimation Algorithm Recursions]
\label{ass:steady_state}
The number of samples $N$ is sufficiently large and the step size $\beta$ is sufficiently small so that for each $j \in \{1,\cdots,p\}$, the recursion in (\ref{eq_260320261120}) after $N$ updates has reached a steady-state regime. Hence, for an additional $(N+1)^{th}$ hypothetical update driven by an arbitrary sample  $(G(x),v)$, the change in the estimation error norm is negligible, i.e.
\begin{IEEEeqnarray}{rCl}
 \left \| \alpha^j - \left. \widehat{\alpha}^j  \right |_{N+1}  \right \|^2     & \approx & \left \| \alpha^j - \left. \widehat{\alpha}^j  \right |_{N}  \right \|^2.
\end{IEEEeqnarray}
\end{assumption}
\begin{proposition}[Prediction Error]\label{proposition_prediction_error}
For an arbitrary input $x$, the prediction error by the estimated parameters (\ref{eq_270320260656}), under Assumption~\ref{ass:steady_state}, admits the following steady-state approximation:
\begin{IEEEeqnarray}{rCl}
\label{eq_270320260707}G(x)\alpha^j - G(x)\widehat{\alpha}^j & \approx & \frac{\beta \| G(x)\|^2}{2+\beta \| G(x)\|^2}\left( \Delta(x) \alpha^j   + (\xi(x))_j  \right).
\end{IEEEeqnarray}
\end{proposition}
\begin{proof}
For an analysis of the estimation algorithm, define the following error:
\begin{IEEEeqnarray}{rCl}
\left. \widetilde{\alpha}^j  \right |_i & : = & \alpha^j  - \left. \widehat{\alpha}^j \right |_{i}.
\end{IEEEeqnarray}
Using (\ref{eq_260320261120}),
\begin{IEEEeqnarray}{rCl}
\label{eq_2603201261157} \left. \widetilde{\alpha}^j  \right |_i & = & \left. \widetilde{\alpha}^j  \right |_{i-1} -  \beta \frac{(v^i)_j - G(x^i) \left. \widehat{\alpha}^j \right |_{i-1} }{1 + \beta \| G(x^i)\|^2} (G(x^i))^T \quad \text{i.e.}\\
\label{eq_2603201261158} G(x^i) \left. \widetilde{\alpha}^j  \right |_i & = & G(x^i)\left. \widetilde{\alpha}^j  \right |_{i-1} -  \beta \frac{(v^i)_j - G(x^i) \left. \widehat{\alpha}^j \right |_{i-1} }{1 + \beta \| G(x^i)\|^2}  \| G(x^i)\|^2.
\end{IEEEeqnarray}
Combining (\ref{eq_2603201261157}) and (\ref{eq_2603201261158}), we have
\begin{IEEEeqnarray}{rCl}
\left. \widetilde{\alpha}^j  \right |_i & = & \left. \widetilde{\alpha}^j  \right |_{i-1} - \frac{G(x^i)\left. \widetilde{\alpha}^j  \right |_{i-1}  - G(x^i) \left. \widetilde{\alpha}^j  \right |_i}{\| G(x^i)\|^2} (G(x^i))^T.
\end{IEEEeqnarray}
By evaluating the squared norm, we obtain the following.
\begin{IEEEeqnarray}{rCl}
\label{eq_260320261354}\left \| \left. \widetilde{\alpha}^j  \right |_i  \right \|^2 & = & \left \| \left. \widetilde{\alpha}^j  \right |_{i-1}  \right \|^2 - \frac{\left( G(x^i)\left. \widetilde{\alpha}^j  \right |_{i-1} \right)^2}{\| G(x^i)\|^2} + \frac{\left( G(x^i)\left. \widetilde{\alpha}^j  \right |_{i} \right)^2}{\| G(x^i)\|^2}.
\end{IEEEeqnarray}
Using (\ref{eq_260320261248}) in (\ref{eq_2603201261158}),
\begin{IEEEeqnarray}{rCl}
\label{eq_260320261352} G(x^i) \left. \widetilde{\alpha}^j  \right |_i & = & G(x^i)\left. \widetilde{\alpha}^j  \right |_{i-1} -  \beta \frac{ G(x^i)\left. \widetilde{\alpha}^j  \right |_{i-1} + \Delta(x^i) \alpha^j   + (\xi(x^i))_j }{1 + \beta \| G(x^i)\|^2}  \| G(x^i)\|^2.
\end{IEEEeqnarray}
Substituting the value of $G(x^i) \left. \widetilde{\alpha}^j  \right |_i$ from (\ref{eq_260320261352}) in (\ref{eq_260320261354}), we get
\begin{IEEEeqnarray}{rCl}
\nonumber \left \| \left. \widetilde{\alpha}^j  \right |_i  \right \|^2 & = & \left \| \left. \widetilde{\alpha}^j  \right |_{i-1}  \right \|^2 + \beta^2 \frac{\left( G(x^i)\left. \widetilde{\alpha}^j  \right |_{i-1} + \Delta (x^i) \alpha^j   + (\xi(x^i))_j \right)^2}{\left( 1 + \beta \| G(x^i)\|^2\right)^2}\| G(x^i)\|^2 \\
\label{eq_260320261512}&& {-}\: 2 \beta  \frac{ G(x^i)\left. \widetilde{\alpha}^j  \right |_{i-1} + \Delta (x^i) \alpha^j   + (\xi(x^i))_j }{1 + \beta \| G(x^i)\|^2} G(x^i)\left. \widetilde{\alpha}^j  \right |_{i-1}.
\end{IEEEeqnarray}
The algorithm (\ref{eq_260320261120}) uses all of $N$ available samples to run from $i=1$ to $i=N$. Our goal is to evaluate the prediction error by the estimated parameters, i.e., we seek to evaluate $G(x)\left. \widetilde{\alpha}^j  \right |_{N}$ for an arbitrary $x$, where $\left. \widetilde{\alpha}^j  \right |_{N}$ is obviously the error in estimated parameters. For this, consider $(N+1)^{th}$ sample to be an arbitrary data pair $(G(x),v)$, on which the prediction error is governed by (\ref{eq_260320261512}) for $i = N+1$:          
\begin{IEEEeqnarray}{rCl}
\nonumber \left \| \left. \widetilde{\alpha}^j  \right |_{N+1}  \right \|^2 & = & \left \| \left. \widetilde{\alpha}^j  \right |_{N}  \right \|^2 + \beta^2 \frac{\left( G(x)\left. \widetilde{\alpha}^j  \right |_{N} + \Delta (x) \alpha^j   + (\xi(x))_j \right)^2}{\left( 1 + \beta \| G(x)\|^2\right)^2}\| G(x)\|^2 \\
\label{eq_260320261512_2}&& {-}\: 2 \beta  \frac{ G(x)\left. \widetilde{\alpha}^j  \right |_{N} + \Delta (x) \alpha^j   + (\xi(x))_j }{1 + \beta \| G(x)\|^2} G(x)\left. \widetilde{\alpha}^j  \right |_{N}.
\end{IEEEeqnarray}
Under Assumption \ref{ass:steady_state}, (\ref{eq_260320261512_2}) becomes
\begin{IEEEeqnarray}{rCl}
 2 G(x)\left. \widetilde{\alpha}^j  \right |_{N} & = & \beta \frac{ G(x)\left. \widetilde{\alpha}^j  \right |_{N} + \Delta (x) \alpha^j   + (\xi(x))_j }{ 1 + \beta \| G(x)\|^2}\| G(x)\|^2, \quad \text{i.e.} \\
G(x)\left. \widetilde{\alpha}^j  \right |_{N} & = & \frac{\beta \| G(x)\|^2}{2+\beta \| G(x)\|^2}\left( \Delta (x) \alpha^j   + (\xi(x))_j  \right).
\end{IEEEeqnarray}
Thus, (\ref{eq_270320260707}) follows.
\end{proof}
\begin{remark}[Uncertainty-Aware Error Weighting]
A noteworthy feature in (\ref{eq_270320260707}) is of the factor
\[
\frac{\beta \|G(x)\|_2^2}{2+\beta \|G(x)\|_2^2},
\]
that may also be interpreted as an uncertainty-aware error weighting: when the KAHM posterior $G(x)$ is diffuse, $\|G(x)\|_2^2$ is smaller and the corresponding noise-sensitive contribution is suppressed, whereas sharper assignments yield larger weights. Because $G(x)$ is a probability vector, $\|G(x)\|_2^2 \le 1$, so this modulation remains uniformly stable.
\end{remark}
\begin{proposition}[Encoding Error]\label{proposition_encoding_error}
The encoding error associated to the estimated parameters (\ref{eq_270320260656}), under Assumption~\ref{ass:steady_state}, for any $\rho > 0$, is upper bounded as
\begin{IEEEeqnarray}{rCl}
\nonumber  \mathop{\mathbb{E}}_{x \sim \mathbb{P}_x} \left [\left \| f_{x \mapsto z}(x) - \widehat{f}_{x \mapsto z}(x) \right \|^2 \right] & \leq &  (1+\rho) \left(1 +\frac{\beta }{2+\beta } \right)^2  \left(\sum_{c=1}^C\left \|w^c \right \|^2\right) \sum_{c = 1}^C \mathop{\mathbb{E}}_{x \sim \mathbb{P}_x} \left[ \left|  \Phi_c(x) - \mathbb{P}_{ y\mid x}(y_c = 1 \mid x)\right|^2 \right] \\
\label{eq_020420260755} && {+} \:  (1+\rho^{-1}) \left( \frac{\beta }{2+\beta } \right)^2 \mathop{\mathbb{E}}_{x \sim \mathbb{P}_x} \left[ \|\xi(x)\|^2 \right ]. \IEEEeqnarraynumspace
\end{IEEEeqnarray}
\end{proposition}
\begin{proof}
Consider
\begin{IEEEeqnarray}{rCl}
\left(f_{x \mapsto z}(x) - \widehat{f}_{x \mapsto z}(x) \right)_j & = & G(x)\alpha^j - G(x)\widehat{\alpha}^j + \Delta (x) \alpha^j.
\end{IEEEeqnarray}
Using (\ref{eq_270320260707}), we have
\begin{IEEEeqnarray}{rCl}
\left(f_{x \mapsto z}(x) - \widehat{f}_{x \mapsto z}(x) \right)_j & = & \left(1 + \frac{\beta \| G(x)\|^2}{2+\beta \| G(x)\|^2} \right) \Delta (x) \alpha^j  + \frac{\beta \| G(x)\|^2}{2+\beta \| G(x)\|^2} (\xi(x))_j.
\end{IEEEeqnarray}
Using $(a+b)^2 \leq (1+\rho)a^2 + (1+\rho^{-1})b^2$, it follows that 
\begin{IEEEeqnarray}{rCl}
\nonumber \lefteqn{\left | \left(f_{x \mapsto z}(x) - \widehat{f}_{x \mapsto z}(x) \right)_j \right |^2} \\
& \leq &   (1+\rho)\left(1 + \frac{\beta \| G(x)\|^2}{2+\beta \| G(x)\|^2} \right)^2 |\Delta (x) \alpha^j|^2 + (1+\rho^{-1}) \left( \frac{\beta \| G(x)\|^2}{2+\beta \| G(x)\|^2} \right)^2 |(\xi(x))_j|^2
\end{IEEEeqnarray}
As all elements of vector $G(x)$ lie in $[0,1]$ and their sum equals unity, $\|G(x)\|^2 \leq 1$, and thus
\begin{IEEEeqnarray}{rCl}
\frac{\beta \| G(x)\|^2}{2+\beta \| G(x)\|^2} & \leq & \frac{\beta }{2+\beta }.
\end{IEEEeqnarray}
Also,
\begin{IEEEeqnarray}{rCCCl}
|\Delta G(x) \alpha^j|^2 & \leq & \left \|\alpha^j \right \|^2 \| \Delta (x) \|^2 & = & \left \|\alpha^j \right \|^2 \sum_{c = 1}^C \left|\mathbb{P}_{ y\mid x}(y_c = 1 \mid x) - \Phi_c(x) \right|^2.
\end{IEEEeqnarray}
Therefore,
\begin{IEEEeqnarray}{rCl}
\nonumber \lefteqn{\left | \left(f_{x \mapsto z}(x) - \widehat{f}_{x \mapsto z}(x) \right)_j \right |^2} \\
& \leq &   (1+\rho) \left(1+\frac{\beta }{2+\beta } \right)^2  \left \|\alpha^j \right \|^2 \sum_{c = 1}^C \left|  \Phi_c(x) - \mathbb{P}_{ y\mid x}(y_c = 1 \mid x)\right|^2 +  (1+\rho^{-1}) \left( \frac{\beta }{2+\beta } \right)^2 |(\xi(x))_j|^2.
\end{IEEEeqnarray}
That is,
\begin{IEEEeqnarray}{rCl}
\nonumber \lefteqn{\left \| f_{x \mapsto z}(x) - \widehat{f}_{x \mapsto z}(x) \right \|^2} \\
& \leq &   (1+\rho) \left(1+\frac{\beta }{2+\beta } \right)^2  \sum_{j=1}^p\left \|\alpha^j \right \|^2 \sum_{c = 1}^C \left|  \Phi_c(x) - \mathbb{P}_{ y\mid x}(y_c = 1 \mid x)\right|^2 +  (1+\rho^{-1}) \left( \frac{\beta }{2+\beta } \right)^2 \|\xi(x)\|^2 \\
& = &  (1+\rho) \left(1+\frac{\beta }{2+\beta } \right)^2  \left(\sum_{c=1}^C\left \|w^c \right \|^2\right) \sum_{c = 1}^C \left|  \Phi_c(x) - \mathbb{P}_{ y\mid x}(y_c = 1 \mid x)\right|^2 +  (1+\rho^{-1}) \left( \frac{\beta }{2+\beta } \right)^2 \|\xi(x)\|^2.
\end{IEEEeqnarray}
Thus, (\ref{eq_020420260755}) follows.
\end{proof}
\begin{theorem}[Encoding Error Bound]\label{theorem_02041025}
Given a data set $ \{(x^i,y^i)  \; \mid \;i \in \{1,2,\cdots,N\} \}   \sim  (\mathbb{P}_{x,y})^N$; under Assumption~\ref{ass:steady_state}, and assuming that (\ref{eq_310320262030}) holds for each $c \in \{1,\cdots,C\}$; we have with probability at least $1-\delta$ for any $\delta \in (0,1)$, $\eta > 0$, and $\rho > 0$;
\begin{IEEEeqnarray}{rCl}
\nonumber \lefteqn{ \mathop{\mathbb{E}}_{x \sim \mathbb{P}_x} \left [\left \| f_{x \mapsto z}(x) - \widehat{f}_{x \mapsto z}(x) \right \|^2 \right]} \\
\nonumber & \leq & (1+\rho) (1 + \eta)\left(1+\frac{\beta }{2+\beta } \right)^2  \left(\sum_{c=1}^C\left \|w^c \right \|^2\right) \left(\sum_{c=1}^C \left | 1- \frac{\sum_{i=1}^{N_c} \Phi_c(x^{\mathrm{I}_i^{c}})}{N \| \Phi_c \|_{L^2(\mathbb{R}^n,\mathbb{P}_{x})}^2}  \right |^2 \| \Phi_c \|_{L^2(\mathbb{R}^n,\mathbb{P}_{x})}^2 \right) \\
\nonumber & & {+}\: (1+\rho) (1+\eta^{-1}) \left(1+\frac{\beta }{2+\beta } \right)^2  \left(\sum_{c=1}^C\left \|w^c \right \|^2\right) \sum_{c=1}^C \min\left( \frac{1}{N} \sum_{i=1}^N |y_c^i - \widetilde{h}_{x \mapsto y_c}(x^i) |^2 +  \frac{4}{\sqrt{N}} +  \sqrt{\frac{\log(C/\delta)}{2N}}, \right. \\
\label{eq_020420261457}&& \left. \frac{1}{\left(N_c/N \right)^2} \left( \frac{3}{\sqrt{N}} + \sqrt{\frac{8  \log(C/\delta)}{N}} \right)  \right) +  (1+\rho^{-1}) \left( \frac{\beta }{2+\beta } \right)^2 \mathop{\mathbb{E}}_{x \sim \mathbb{P}_x} \left[ \|\xi(x)\|^2 \right ].
\end{IEEEeqnarray}
\end{theorem}
\begin{proof}
For each $c \in \{1,\cdots,C \}$, Theorem~\ref{theorem_300320262124} is applied with confidence level $\delta/C$. So, we have with probability at least $1- \delta/C$, 
\begin{IEEEeqnarray}{rCl}
\nonumber \lefteqn{\mathop{\mathbb{E}}_{x \sim \mathbb{P}_x}\left[ \left| \Phi_c(x) - \mathbb{P}_{y | x}(y_c = 1 | x) \right|^2\right] } \\ \nonumber & \leq & (1+\eta) \left | 1- \frac{\sum_{i=1}^{N_c} \Phi_c(x^{\mathrm{I}_i^{c}})}{N \| \Phi_c \|_{L^2(\mathbb{R}^n,\mathbb{P}_{x})}^2}  \right |^2 \| \Phi_c \|_{L^2(\mathbb{R}^n,\mathbb{P}_{x})}^2  \\
\label{eq_020420261328}&& {+}\: (1+\eta^{-1}) \min\left( \frac{1}{N} \sum_{i=1}^N |y_c^i - \widetilde{h}_{x \mapsto y_c}(x^i) |^2 +  \frac{4}{\sqrt{N}} +  \sqrt{\frac{\log(C/\delta)}{2N}},  \frac{1}{\left(N_c/N \right)^2} \left( \frac{3}{\sqrt{N}} + \sqrt{\frac{8  \log(C/\delta)}{N}} \right)  \right).
\end{IEEEeqnarray}
Define $A^c : = \{(\ref{eq_020420261328}) \text{ holds}\}$ and $B^c : = \{(\ref{eq_020420261328}) \text{ does not hold}\}$, and apply a union bound over $c$:
\begin{IEEEeqnarray}{rCCCCCCCl}
\text{probability of } \mathop{\cap}_{c=1}^C A^c & = & 1 - \text{probability of } \mathop{\cup}_{c=1}^C B^c & \geq & 1- \sum_{c=1}^C \text{probability of } B^c & \geq & 1- \sum_{c = 1}^C \frac{\delta}{C}  & = & 1- \delta.
\end{IEEEeqnarray}
Hence, with probability at least $1- \delta$, (\ref{eq_020420261328}) holds simultaneously for all $c \in \{1,\cdots,C\}$. Substituting the simultaneous cluster-wise bounds (\ref{eq_020420261328}) into (\ref{eq_020420260755}) leads to the result.   
\end{proof}
\begin{remark}[Interpretation of Theorem~\ref{theorem_02041025}]
The upper bound in Theorem~\ref{theorem_02041025} provides an explicit decomposition of the encoding error into three components: a term associated with the approximation of the posterior probabilities, a finite-sample/generalization term inherited from the cluster-wise learning problem, and a term reflecting the influence of the teacher noise. In this decomposition, the step-size parameter $\beta$ appears through the factors $\left(1+\frac{\beta}{2+\beta}\right)^2$ and $\left(\frac{\beta}{2+\beta}\right)^2$. Since both quantities are increasing functions of $\beta$ for $\beta>0$, smaller admissible values of step-size $\beta$ lead to a tighter upper bound. In particular, the final term in (\ref{eq_020420261457}) makes explicit the noise-filtering property of the proposed KAHM--NLMS pipeline: the teacher-noise contribution enters the encoding bound only through the damping factor $\left(\frac{\beta}{2+\beta}\right)^2<1$. Hence, for every admissible $\beta>0$, the effect of teacher noise is attenuated in the final estimate rather than propagated without control. The parameters $\eta>0$ and $\rho>0$ act as free balancing parameters within the estimate. In particular, $\eta$ controls the trade-off between the posterior-approximation term and the finite-sample/generalization contribution, whereas $\rho$ balances these terms against the contribution of the teacher noise. Consequently, $\eta$ and $\rho$ may be selected so as to minimize the right-hand side of the bound (\ref{eq_020420261457}).
\end{remark}
\begin{remark}[Practical Significance of Theorem~\ref{theorem_02041025} for Austrian-Law Benchmark]
In the simulated distributed-law setting studied later, Theorem~\ref{theorem_02041025} has a direct operational reading. Each law-specific encoder is trained only on its own local supervision, so the quality of the final multi-law system depends on whether the correct semantic clusters remain geometrically identifiable from lexical features within each local model and whether the corresponding teacher embeddings admit stable prototype representatives. The theorem shows that, under these conditions, the final encoding error is controlled by three transparent factors only: posterior mismatch in the KAHM stage, finite-sample approximation in the local learning problem, and noise inherited from the teacher embeddings. This is precisely the form of guarantee that is relevant for the benchmark in Section~\ref{sec_experiments}, because it explains why a family of separate law-local encoders can be combined later by lightweight gating without requiring a centralized neural query encoder at deployment time.
\end{remark}
\begin{remark}[Scope of the Theoretical Guarantee]
Theorem~\ref{theorem_02041025} provides a conditional error decomposition for the proposed estimator rather than a fully assumptions-free finite-iteration guarantee. The posterior-estimation component is controlled through the RKHS approximation step, whereas the prototype-refinement component is analyzed under the steady-state approximation adopted for the NLMS recursion. Accordingly, the theorem identifies the terms that govern the final encoding error (posterior mismatch, finite-sample approximation, and teacher-noise effects) under the modeling and steady-state assumptions used in this paper. This interpretation is consistent with standard adaptive-filter analyses, where sufficiently small step sizes and sufficiently long runs are used to characterize steady-state behavior rather than exact transient identities \cite{sayed2008adaptive,haykin2014adaptive}.
\end{remark}

\section{Proof-of-Principle Retrieval Instantiation}\label{sec_experiments}
This section evaluates the proposed encoding method on an Austrian-law retrieval benchmark. The Austrian-law benchmark instantiates the fixed-teacher semantic encoding problem developed in Sections~\ref{sec_background}--\ref{sec_encoding_solution}. The central empirical question is whether the proposed KAHM-based encoder can serve, in this domain, as a compute-efficient substitute for online transformer query encoding by mapping inexpensive lexical query features into a high-quality semantic embedding space. To investigate this, we instantiate the semantic targets in Problem~\ref{problem_encoding} by teacher embeddings and train the proposed encoder to predict them from lexical features. The downstream task is law retrieval: given a natural-language query, the system retrieves legal passages and routes the query to the correct Austrian law.

The experimental comparison is intentionally evaluation-matched. All learned adapters considered in this section operate under the same lexical front end, the same fixed semantic teacher space, the same law-wise split, the same distance-gated serving contract, and the same frozen corpus index. The experiments isolate the contribution of KAHM geometry within a controlled query-adaptation regime. Section~\ref{sec_data_construction} describes the construction of the corpus and query benchmark, Section~\ref{sec_features_representation} introduces the semantic and lexical representations, Section~\ref{sec_encoding_model} describes the training of the KAHM-based encoding models, Section~\ref{sec_retrieval_protocol} defines the retrieval evaluation protocol and performance measures, Section~\ref{subsec:experimental-results} reports the experimental results, Section~\ref{subsec:reproducibility} documents reproducibility and artifact availability, and Section~\ref{sec_scope} discusses the scope of the current study.

\subsection{Data Construction}\label{sec_data_construction}
The law/domain universe contains 84 Austrian statutes. Table~\ref{tab:austrian-law-abbreviations} in Appendix C lists the abbreviations and English glosses used as benchmark labels. The retrieval corpus is constructed from consolidated German legal texts downloaded from the Austrian RIS portal \cite{ris_austria_portal}, with one PDF per statute. The PDFs are converted into retrieval units by a layout-aware extraction pipeline that reconstructs reading order from PDF layout, removes recurring headers and footers, segments text into paragraphs using legal section markers, indentation cues, and large vertical gaps, and merges paragraphs that continue across page breaks. The extractor also performs dehyphenation across line breaks while preserving genuine compound hyphens. In the reported configuration, extraction is performed in passage mode, which chunks the cleaned paragraph stream into overlapping passages with a target length of 260 tokens, an overlap of 60 tokens, a minimum retained length of 40 tokens, and a hard maximum of 420 tokens. The resulting aligned evaluation corpus contains 10,762 retrieval units.

We treat each law as a separate domain not only because it is the prediction target, but also because this organization matches the supervision structure studied in the paper. The training data are constructed law-wise, the serving contract is law-wise, and the later encoder family is trained law-wise as well. The benchmark does not execute a live federated protocol, but the law-wise organization is consistent with deployment settings in which supervision remains attached to specific legal areas or organizational units rather than being pooled into one undifferentiated cross-law training repository.

Queries are generated synthetically by first defining law-specific issue inventories, then expanding them with common legal facets such as prerequisites, deadlines, competence, application form, evidence, costs, remedies, service, limitation, enforcement, liability, exceptions, and sanctions. Here, a topic is not merely a law label. Rather, for each law, the generator constructs law-internal issue specifications from a law-specific issue inventory and its facet expansion, and assigns a unique topic identifier to each such issue. For each topic, the generator samples a law-conditioned context including actor, counterparty, city, time expression, amount, communication channel, evidence type, and authority, and renders the topic into several surface forms. To cover several practically relevant surface forms of legal-information needs, each underlying topic is rendered in seven surface forms: \texttt{authority} queries ask which institution is competent or whether a matter belongs before an authority or a court; \texttt{keyword} queries consist mainly of short content words or search-style phrases; \texttt{nl\_long} queries are longer natural-language questions built from fuller case descriptions; \texttt{nl\_short} queries are short natural-language questions; \texttt{procedural} queries focus on steps of the procedure, such as deadlines, applications, evidence, costs, remedies, or jurisdiction; \texttt{scenario} queries are phrased as short fact patterns followed by an explicit question; and \texttt{fragment} queries are deliberately incomplete or elliptical formulations intended to mimic terse user input. For each topic, the generator samples three different realizations of each style, so that one topic yields $7\times 3=21$ distinct queries. In the reported experiment, the query generator produces 40,000 training queries and 5,000 test queries. Both sets are stratified. Consistent with the supervised fixed semantic space deployment regime studied here, the default split is topic-overlapping by design: test queries are new surface realizations of law-specific topics represented in training for the same law, with no exact query-text overlap across the two splits.

This synthetic construction is broader than a narrow template benchmark, because the generator varies issue inventories, legal facets, sampled contexts, and user-facing formulations rather than instantiating one fixed query pattern. We therefore use it as a controlled benchmark for relative method comparison and for robustness to heterogeneous in-domain query formulations within Austrian-law retrieval. The topic-overlapping protocol is therefore deliberate rather than incidental: it tests the target deployment setting in which relevant law-internal issue types are represented during training, but users may express them through diverse paraphrases, fact-pattern descriptions, procedural framings, and elliptical formulations. The benchmark consequently measures in-domain formulation robustness and sampled contextual variation around represented legal issues, not cold-start discovery of genuinely unseen issues.

\subsection{Semantic and Lexical Representations}\label{sec_features_representation}
We instantiate the semantic target space with dense embeddings produced by the SentenceTransformer model \texttt{mixedbread-ai/deepset-mxbai-embed-de-large-v1}~\cite{mxbai_model_card,sentence_bert}. 
This model is an open-source German/English retrieval encoder developed jointly by Mixedbread and deepset, built on top of \texttt{intfloat/multilingual-e5-large}. The official model card lists it as a model of approximately 0.5B parameters, while the underlying multilingual-E5-large architecture has 24 transformer layers and 1024-dimensional embeddings \cite{mxbai_model_card,multilingual_e5_model_card,multilingual_e5_techreport}. According to the official vendor documentation, the model was initialized from \texttt{multilingual-e5-large}, fine-tuned on more than 30 million high-quality German training pairs, trained with the AnglE loss, and optimized for compression using a mixture of full fine-tuning and LoRA, while also supporting binary quantization and Matryoshka representation learning \cite{mxbai_blog,mxbai_model_card,angle_paper,matryoshka_paper}. We use this same encoder both to embed the law-passage corpus and to precompute the semantic target vectors for the training and test queries, storing the query targets in separate NPZ files so that the subsequent KAHM-based encoder learning requires no further online transformer inference. In our implementation, the SentenceTransformers interface is instantiated with \texttt{truncate\_dim=1024}, thereby fixing the output dimensionality for both passage and query embeddings. Query embeddings are $L^2$-normalized and in the final retrieval evaluation all query and corpus vectors are $L^2$-normalized again and indexed with FAISS~\cite{faiss_library} \texttt{IndexFlatIP}, so that inner-product ranking is equivalent to cosine-similarity retrieval on normalized vectors.

The lexical baseline index is constructed with a fixed IDF--SVD pipeline. Concretely, each retrieval unit is first preprocessed by normalizing case, canonicalizing legal section markers, and handling standalone numbers with a configurable replacement scheme. The resulting text is then mapped to a sparse TF--IDF representation \cite{salton1988term}, where the word-level feature space is not induced from the corpus alone but fixed to the full tokenizer vocabulary of the semantic teacher model \texttt{mixedbread-ai/deepset-mxbai-embed-de-large-v1}. The tokenization is performed through the Hugging Face \texttt{transformers} interface so that the lexical features remain compatible with the teacher model's subword inventory \cite{wolf2020transformers}. The sparse TF--IDF matrix is subsequently compressed with truncated singular value decomposition, i.e.\ latent semantic analysis (LSA) \cite{deerwester1990indexing}, using the randomized low-rank approximation strategy implemented in modern SVD solvers \cite{halko2011finding}. In our setup, this yields a 512-dimensional dense lexical vector for each retrieval unit. The fitted extraction pipeline is serialized and then applied to the complete set of retrieval units, whose identifiers and dense lexical vectors are stored in a NPZ file for subsequent retrieval and adapter-training experiments.

\subsection{Development of the Law-Specific Encoding Models}\label{sec_encoding_model}
To instantiate Problem~\ref{problem_encoding} in the Austrian-law retrieval benchmark, we constructed paired lexical-semantic training data from the synthetic query collection. Each query is associated with a query text, a query identifier, and a consensus law. Let $q \in \{1,\cdots,Q\}$ denote the law/domain index. For each query text, the lexical representation was obtained by the fixed IDF–SVD model, whereas the semantic representation was given by the corresponding Mixedbread query embedding. Both the IDF–SVD query vectors and the Mixedbread target embeddings are $L^2-$normalized in the training pipeline. Hence, for each law domain $q$, the available data form a law-specific sample set $\{ (x^i,v^i) \; \mid \; q^i = q,\; i \in \{1,\cdots,N\}\}$, where $x^i$ plays the role of the lexical feature vector and $v^i$ is the semantic vector generated by the noisy teacher. This law-wise decomposition is deliberate. It matches the distributed-law scenario motivating the paper: if training supervision remains attached to individual laws, then each law can contribute a local encoder without requiring that all query data be merged into a single cross-law training repository. The decomposition is also aligned with prior KAHM-based collaborative and federated constructions, where prediction can be assembled from local geometric models via distance-type summaries rather than repeated exchange of gradient updates \cite{KAHM,kumar2024geometricallyinspiredkernelmachines,kumar2025operatortheoreticframeworkgradientfreefederated}. For each law-specific data set, we first applied a random holdout split into a core training part and a validation part. The validation size was chosen as nearly 5\% of the total samples. The KAHM-based feature-map and cluster prototypes were learned exclusively from the core training subset. 

The semantic clusters are not manually annotated. Instead, they are obtained directly from the noisy-teacher training samples by $K-$means clustering, i.e., the set $\{ v^i \; \mid \; q^i = q,\; i \in \{1,\cdots,N\}\}$ is partitioned into $C$ different clusters. In the main experiments we use $C=300$ semantic clusters per law. This value is adopted as a practical operating point that balances teacher-space reconstruction quality against training cost. To document the sensitivity of the method to this design choice, Table~\ref{tab:ablation-nclusters} reports an ablation over $C \in \{100,200,300,400\}$ under the same law-wise training pipeline and evaluated query benchmark.
\begin{table}[t]
\centering
\renewcommand{\arraystretch}{1.1}
\setlength{\tabcolsep}{8pt}
\caption{Sensitivity of the law-specific KAHM encoders to the number of semantic clusters $C$. Reported values summarize teacher-space reconstruction and total training time on the evaluated query benchmark. The table documents the practical operating range of $C$ and the diminishing returns beyond $C=300$.}
\label{tab:ablation-nclusters}
\begin{tabular}{@{}ccccc@{}}
\toprule
$C$ & Cosine mean & MSE & Overall $R^2$ & Total train time (s) \\
\midrule
100 & 0.938682 & 0.000117 & 0.879808 & 436.81 \\
200 & 0.949674 & 0.000096 & 0.902101 & 856.31 \\
300 & 0.953579 & 0.000088 & 0.909525 & 1343.09 \\
400 & 0.954752 & 0.000086 & 0.911746 & 1687.04 \\
\bottomrule
\end{tabular}
\end{table}
\begin{remark}[Sensitivity to the Number of Semantic Clusters]
Table~\ref{tab:ablation-nclusters} shows that increasing the number of semantic clusters yields a monotone
improvement in teacher-space reconstruction, which is consistent with the view that a finer partition of the
teacher space supports a more accurate prototype-mixture approximation. At the same time, the marginal gains
decay rapidly: the increase in combined cosine mean drops from $0.010992$ when moving from $C=100$ to
$C=200$, to $0.003905$ from $C=200$ to $C=300$, and to only $0.001173$ from $C=300$ to $C=400$; the
corresponding gains in overall $R^2$ are $0.022293$, $0.007424$, and $0.002221$. Since the improvement
beyond $C=300$ is small relative to the added computational cost, $C=300$ is a reasonable operating point
for the main experiments.
\end{remark}
Since the downstream cluster-specific KAHM-based modeling of lexical features requires at least two samples per cluster, singleton clusters are handled explicitly by adding one auxiliary training point to each singleton cluster. Concretely, for each singleton cluster with unique sample $(x^s,v^s)$, the algorithm first computes nearest neighbors in the lexical input space by Euclidean distance. Let $x^n$ denote the nearest non-self neighbor of $x^s$ in the training set. The auxiliary lexical point is then constructed as $x^{aux} = 0.9x^s+0.1x^n$. After interpolation, $x^{aux}$ is rescaled to have the same Euclidean norm as $x^s$. On the semantic side, the target is copied exactly, i.e. $v^{aux} = v^s$. Hence, each singleton cluster is converted into the two-point set $\{(x^s,v^s), (x^{aux},v^{aux}) \}$, preserving the number of semantic clusters while satisfying the minimum-sample requirement of the KAHM-based modeling of a cluster. The KAHM-based modeling of a cluster is required to evaluate the space folding measure (Definition~\ref{definition_150220251240}) and thus the feature-map (Definition~\ref{def_030420261238}). The semantic prototypes are initialized as empirical centroids of the output-space clusters of the noisy-teacher embeddings. These cluster centroids serve as the starting point of the prototype vectors. The validation subset was then used to select the soft assignment hyperparameters $(\omega, K)$ by minimizing validation mean squared error in predicting Mixedbread embeddings over a grid of sharpening parameters $\omega$ and truncation levels $K$, namely
\begin{IEEEeqnarray}{rCl}
\omega & \in & \{5,8,10,11,12,13,14,15,16,17,18,19,20\},\\ 
K & \in & \{2,5,8,10,11,12,13,14,15,16,17,18,19,20,25,50,75,100,125,150,175,200\}.
\end{IEEEeqnarray}
After model selection, the semantic prototypes are refined with the NLMS update in~(\ref{eq_260320261120}) for \(20\) epochs with step size \(\beta=0.1\), using the union of the core-training and validation samples. A separate serialized encoder is then stored for each law. In the final multi-law setting, these law-specific encoders are not merged into a single cross-law model; instead, prediction is performed by distance-gated model selection. More precisely, for a query represented by lexical feature vector \(x\), we compute for each law-specific encoder the minimum clusterwise space-folding score \(\min_c \mathcal{T}_{\mathbf{X}^c}(x)\) and select the prediction produced by the encoder attaining the smallest such score. Conceptually, this selection rule is the semantic-encoding analogue of the minimum-distance aggregation used in prior KAHM federated classifiers: the local law-wise models remain separate, while the final decision is assembled from lightweight geometric scores rather than from a jointly optimized cross-law parameter vector \cite{KAHM,kumar2024geometricallyinspiredkernelmachines,kumar2025operatortheoreticframeworkgradientfreefederated}.
\begin{remark}[Relation to the i.i.d.\ sampling assumption] The i.i.d.\ assumption in Section~\ref{sec_130420262100} concerns the generation of the finite training samples $\mathcal{S}$ used in the approximation and generalization analysis. After the relevant samples have been drawn, the NLMS recursion in~(\ref{eq_260320261120}) is run for multiple passes over that fixed empirical sample in order to approximate the steady-state regime invoked in Proposition~\ref{proposition_prediction_error} through Assumption~\ref{ass:steady_state}. Accordingly, the reported $20$ epochs should be interpreted as an implementation choice for stabilizing the prototype refinement on the observed sample towards steady-state regime.
\end{remark}
\begin{remark}[Empirical Diagnostic for Assumption~\ref{ass:feature_map_concentartion}]
Because the theoretical analysis in Section~\ref{sec_encoding_solution} relies on Assumption~\ref{ass:feature_map_concentartion}, we additionally evaluated a sample analogue of the cluster-wise $L_2$-mass condition in~(\ref{eq_310320262030}) on the realized law-wise KAHM-induced feature-maps. Concretely, for each law $q$ and semantic cluster $c \in \mathcal{Q}^q$ (covered by law $q$), we computed
\[
\widehat{\|\Phi_c\|^2_{L_2}}
=
\frac{1}{|\mathcal{J}^q|}\sum_{i \in \mathcal{J}^q}^{|\mathcal{J}^q|}\Phi_c(x_i)^2,\quad c \in \mathcal{Q}^q,
\]
and compared it with $|\mathcal{J}^{q,c}|/|\mathcal{J}^q|$. Across the $84$ laws and $25{,}200$ evaluated clusters (i.e. $Q = 84$ and $C = 25200$), the worst-case margin $\widehat{\|\Phi_c\|^2_{L_2}} - |\mathcal{J}^{q,c}|/|\mathcal{J}^q|$ was $-2.898\times 10^{-3}$, while the median and mean margins were $-3.16\times 10^{-5}$ and $-7.19\times 10^{-5}$, respectively. Thus, although the exact finite-sample evaluation is slightly negative for most clusters, the deviations are uniformly very small in absolute value, in particular, the sample analogue of~(\ref{eq_310320262030}) is satisfied for all evaluated clusters up to an absolute tolerance of $3\times 10^{-3}$. We therefore interpret Assumption~\ref{ass:feature_map_concentartion} as empirically well supported, up to a very small approximation error, in the law-wise deployment regime studied here.
\end{remark}
\subsection{Retrieval Evaluation Protocol and Performance Measures}\label{sec_retrieval_protocol}
We evaluated the proposed encoder in the downstream Austrian-law retrieval task using a three-system primary
comparison together with a broader extension to additional evaluation-matched learned adapters. This design is
intentionally matched to isolate the query-adaptation problem under the fixed serving contract studied here. The
three primary retrieval pipelines are:
\begin{enumerate}
\item IDF--SVD, which embeds both queries and corpus passages in the lexical IDF--SVD space;
\item Mixedbread, which uses the transformer teacher for both query and corpus embeddings; and
\item KAHM (query$\rightarrow$MB corpus), in which the query is first embedded in the lexical IDF--SVD space and then mapped by the proposed KAHM-based adapter into the Mixedbread embedding space, while the corpus index remains fixed in that semantic space. 
\end{enumerate}
The evaluation was carried out on the synthetic test set described in Section~\ref{sec_data_construction}. In the reported run, the corpus consisted of 10{,}762 aligned retrieval units from 84 laws, and the evaluated test set contained 5{,}000 queries. The query generator produced 40,000 training queries and 5,000 test queries, sampled seven query styles, and ensured that no exact query text appears in both splits; as described in Section~\ref{sec_data_construction}, the default split is topic-overlapping by design in order to evaluate text-disjoint reformulations of represented law-specific issues.

For evaluation, the corpus parquet file and the embedding bundles were first aligned by the common set of \texttt{sentence\_id} values so that every retrieved index could be mapped back to a unique law label. All query and corpus embeddings were then \(L_2\)-normalized and retrieval was performed with FAISS~\cite{faiss_library} \texttt{IndexFlatIP}. For each query, the system returned the top-\(k\) ranked retrieval units for \(k \in \{3,5,10,15,20\}\). Let \(g_i\) denote the consensus law of query \(i\), and let
\[
L_i^{(k)}=(\ell_{i,1},\ldots,\ell_{i,k})
\]
be the sequence of law labels attached to the top-\(k\) retrieved passages for that query. All performance measures were first computed separately for each query and only then averaged.

The first metric was \emph{Hit@$k$}, defined by
\[
\mathrm{Hit@}k(i) = \mathbf{1}\!\left[\exists\, j \in \{1,\ldots,k\} : \ell_{i,j} = g_i \right],
\]
where $\mathbf{1}[\cdot]$ denotes the indicator function, which equals $1$ when the condition in brackets is true and $0$ otherwise. This metric checks whether the correct law appears at least once among the top-$k$ retrieved passages, irrespective of its exact rank.

The second metric was \emph{Top-1 Accuracy}, defined by
\[
\mathrm{Top1}(i) = \mathbf{1}\!\left[\ell_{i,1} = g_i \right],
\]
which measures whether the very first retrieved passage belongs to the correct law. Since this quantity depends only on rank $1$, it is invariant with respect to the chosen cutoff $k$.

The third metric was mean reciprocal rank at cutoff $k$ (\emph{MRR@$k$}) over unique laws. Because a single law can contribute many passages to the corpus, the ranked passage list $L_i^{(k)}$ was first collapsed to an order-preserving list of distinct laws,
\[
\widetilde{L}_i^{(k)} = (\widetilde{\ell}_{i,1}, \widetilde{\ell}_{i,2}, \ldots),
\]
obtained by scanning $L_i^{(k)}$ from top to bottom and keeping only the first occurrence of each law. Let
\[
r_i^{(k)} = \min\{r \geq 1 \mid \widetilde{\ell}_{i,r} = g_i\},
\]
with $r_i^{(k)} = \infty$ if the consensus law does not occur in $\widetilde{L}_i^{(k)}$. The per-query reciprocal-rank score was then defined as
\[
\mathrm{MRR@}k(i) =
\begin{cases}
1 / r_i^{(k)}, & r_i^{(k)} < \infty,\\[4pt]
0, & r_i^{(k)} = \infty.
\end{cases}
\]
This metric evaluates the rank of the correct law rather than the rank of an arbitrary individual passage and is therefore robust against repeated passages from the same statute in the top-$k$ list. The reported dataset-level MRR@$k$ is the average of $\mathrm{MRR@}k(i)$ over all queries.

In addition to these ranking-oriented measures, we reported three consensus-oriented diagnostics derived from the label composition of the top-\(k\) neighborhood. Let
\[
n_i^{(k)}(a)=\sum_{j=1}^{k}\mathbf{1}\!\left[\ell_{i,j}=a\right]
\]
be the number of retrieved passages belonging to law \(a\), and let
\[
m_i^{(k)}=\frac{1}{k}\max_a n_i^{(k)}(a)
\]
be the plurality fraction in the top-\(k\) set. With the predominance threshold fixed to \(\tau=0.10\), the \emph{majority-accuracy} metric was defined as
\[
\mathrm{MajAcc}_\tau(i)=
\mathbf{1}\!\left[\widehat{g}_i^{(k)}=g_i\right]\,
\mathbf{1}\!\left[m_i^{(k)}\geq \tau\right],
\]
where \(\widehat{g}_i^{(k)}\) denotes the plurality law in the top-\(k\) list. Hence, a query contributes \(1\) only if the top-\(k\) neighborhood yields a sufficiently predominant vote and that vote is correct; abstentions, i.e.\ cases with \(m_i^{(k)}<\tau\), contribute \(0\). The mean consensus fraction was defined as
\[
\mathrm{ConsFrac}_i^{(k)}=\frac{1}{k}\sum_{j=1}^{k}\mathbf{1}\!\left[\ell_{i,j}=g_i\right],
\]
that is, the fraction of the top-\(k\) retrieved passages that belong to the consensus law. Finally, the \emph{mean lift (prior)} normalized this consensus fraction by the corpus prior of the gold law,
\[
\mathrm{Lift}_i^{(k)}=
\frac{\mathrm{ConsFrac}_i^{(k)}}{\mathbb{P}_0(g_i)},
\qquad
\mathbb{P}_0(q)=\frac{N_{\mathrm{corpus},q}}{N_{\mathrm{corpus}}},
\]
where \(N_{\mathrm{corpus},q}\) is the number of aligned corpus passages labeled with law \(q\). This normalization discounts the trivial advantage of frequent laws and measures enrichment relative to chance under the empirical corpus distribution. 

The primary reported aggregates are micro-averages over queries. For a generic per-query score \(s_i^{(k)}\), the micro-average is
\[
\overline{s}^{(k)}=\frac{1}{N_{\mathrm{test}}}\sum_{i=1}^{N_{\mathrm{test}}} s_i^{(k)}.
\]
As a robustness check against label-frequency skew, we also computed macro-averages over laws: for each law \(q\), the metric was first averaged over all test queries with \(g_i=q\), and these law-wise means were then averaged uniformly over the 84 laws. Thus, macro-averaging gives every law equal weight, irrespective of how many test queries it contributes. 

Uncertainty was quantified by a paired nonparametric bootstrap with \(5{,}000\) resamples and random seed \(0\). For a given metric, each bootstrap replicate resampled queries with replacement and recomputed the same micro-averaged score; the \(2.5\%\) and \(97.5\%\) empirical percentiles of the resulting bootstrap distribution formed the reported \(95\%\) confidence interval. Pairwise method comparisons, such as \(\Delta(\mathrm{KAHM}-\mathrm{IDF\mbox{--}SVD})\), were computed from per-query paired differences and bootstrapped in the same paired manner so that the covariance induced by evaluating all methods on the same query set was preserved. For macro-level robustness analyses, the resampling unit was the law rather than the query. In addition to the main retrieval metrics above, the report also includes an auxiliary coverage--precision sweep for majority-vote routing thresholds $\tau'$. Any threshold selected from that sweep is used only as a descriptive post-hoc routing diagnostic for the evaluated run, whereas the six measures defined above are the principal criteria used for the experimental comparison in this paper. 

\begin{remark}[Statistical Interpretation]
We use paired nonparametric bootstrap intervals because all retrieval systems are evaluated on the same query set and the resulting per-query effectiveness values are neither independent across systems nor naturally justified by Gaussian assumptions. In information-retrieval evaluation, paired bootstrap procedures are a standard way to quantify uncertainty and compare systems while preserving the dependence induced by shared test queries \cite{sakai2006bootstrap}.
\end{remark}
Among the reported measures, the rank-sensitive retrieval criteria are the primary basis for the paper's architectural claim. This is because the deployment objective studied here is not merely approximate teacher-space reconstruction, but accurate law identification near the top of the ranking under a lightweight serving-time path. The consensus-oriented diagnostics remain useful supplementary evidence, especially for understanding neighborhood purity and abstention-aware routing behavior, but they are secondary to the top-of-ranking retrieval measures when interpreting the main empirical result.

\subsection{Experimental Results}
\label{subsec:experimental-results}
The central empirical question of this study is whether the proposed KAHM-based query encoder can replace online transformer query inference with a substantially lighter deployment-time procedure while retaining strong downstream retrieval behavior in the controlled Austrian-law benchmark studied here. The empirical results are organized around three claims. First, if the proposed geometry is effective as a lexical-to-semantic adapter, then it should achieve strong reconstruction of the fixed teacher space from inexpensive lexical query features. Second, if that reconstruction is decision-relevant, then it should translate into strong rank-sensitive law retrieval under the shared serving contract. Third, if the method is practically attractive as a deployment-time substitute, then those quality results should be accompanied by a materially lighter online query path than direct transformer query encoding.

Tables~\ref{tab:micro-rank-metrics} and~\ref{tab:micro-consensus-metrics} report the primary micro-averaged three-system comparison on the 5,000-query Austrian-law test benchmark with 84 candidate laws and 10,762 aligned retrieval units. The three primary systems are: (i) \emph{IDF--SVD}, a purely lexical baseline in which both queries and corpus passages are represented directly in the lexical IDF--SVD space; (ii) \emph{Mixedbread}, the direct transformer baseline in which both queries and corpus passages are embedded by the frozen teacher model; and (iii) \emph{KAHM (query$\rightarrow$MB corpus)}, the proposed adapter in which a query is first represented by lexical IDF--SVD features and then mapped into the frozen Mixedbread embedding space, while the corpus index remains fixed in that same semantic space. Thus, in the table labels, ``query$\rightarrow$MB'' always means that the query is predicted in the Mixedbread teacher space and searched against the fixed Mixedbread corpus index.
\begin{table*}
\centering
\renewcommand{\arraystretch}{1.1}
\setlength{\tabcolsep}{6pt}
\caption{Micro-averaged ranking metrics (mean with paired-bootstrap \(95\%\) confidence intervals). Best values in each row are highlighted in bold.}
\label{tab:micro-rank-metrics}

\textbf{Panel A: MRR@\(\,k\) (unique laws)}
\medskip

\begin{tabular}{@{}c|ccc@{}}
\toprule
\(\,k\) & IDF--SVD & KAHM & Mixedbread \\
\midrule
3  & \(0.334\,[0.322, 0.347]\) & \(\mathbf{0.466}\,[0.453, 0.479]\) & \(0.436\,[0.423, 0.449]\) \\
5  & \(0.349\,[0.337, 0.361]\) & \(\mathbf{0.481}\,[0.469, 0.494]\) & \(0.454\,[0.441, 0.466]\) \\
10 & \(0.364\,[0.351, 0.376]\) & \(\mathbf{0.496}\,[0.484, 0.508]\) & \(0.468\,[0.456, 0.480]\) \\
15 & \(0.371\,[0.359, 0.383]\) & \(\mathbf{0.501}\,[0.489, 0.513]\) & \(0.474\,[0.462, 0.486]\) \\
20 & \(0.376\,[0.364, 0.387]\) & \(\mathbf{0.504}\,[0.492, 0.516]\) & \(0.478\,[0.466, 0.489]\) \\
\bottomrule
\end{tabular}

\bigskip

\textbf{Panel B: Hit@\(\,k\)}
\medskip

\begin{tabular}{@{}c|ccc@{}}
\toprule
\(\,k\) & IDF--SVD & KAHM & Mixedbread \\
\midrule
3  & \(0.389\,[0.375, 0.402]\) & \(\mathbf{0.529}\,[0.516, 0.543]\) & \(0.503\,[0.488, 0.516]\) \\
5  & \(0.433\,[0.420, 0.447]\) & \(\mathbf{0.579}\,[0.566, 0.593]\) & \(0.560\,[0.546, 0.573]\) \\
10 & \(0.497\,[0.483, 0.511]\) & \(\mathbf{0.643}\,[0.629, 0.656]\) & \(0.625\,[0.612, 0.639]\) \\
15 & \(0.539\,[0.525, 0.553]\) & \(\mathbf{0.673}\,[0.659, 0.685]\) & \(0.662\,[0.649, 0.676]\) \\
20 & \(0.572\,[0.558, 0.586]\) & \(\mathbf{0.694}\,[0.681, 0.707]\) & \(0.688\,[0.675, 0.701]\) \\
\bottomrule
\end{tabular}

\bigskip

\textbf{Panel C: Top-1 accuracy}
\medskip

\begin{tabular}{@{}c|ccc@{}}
\toprule
\(\,k\) & IDF--SVD & KAHM & Mixedbread \\
\midrule
3  & \(0.287\,[0.274, 0.299]\) & \(\mathbf{0.411}\,[0.397, 0.424]\) & \(0.378\,[0.365, 0.392]\) \\
5  & \(0.287\,[0.274, 0.299]\) & \(\mathbf{0.411}\,[0.397, 0.424]\) & \(0.378\,[0.365, 0.392]\) \\
10 & \(0.287\,[0.275, 0.299]\) & \(\mathbf{0.411}\,[0.397, 0.424]\) & \(0.378\,[0.364, 0.392]\) \\
15 & \(0.287\,[0.274, 0.299]\) & \(\mathbf{0.411}\,[0.398, 0.424]\) & \(0.378\,[0.365, 0.392]\) \\
20 & \(0.287\,[0.274, 0.299]\) & \(\mathbf{0.411}\,[0.397, 0.424]\) & \(0.378\,[0.364, 0.392]\) \\
\bottomrule
\end{tabular}
\end{table*}
\begin{table*}
\centering
\renewcommand{\arraystretch}{1.1}
\setlength{\tabcolsep}{6pt}
\caption{Micro-averaged consensus-oriented metrics (mean with paired-bootstrap \(95\%\) confidence intervals). Best values in each row are highlighted in bold.}
\label{tab:micro-consensus-metrics}

\textbf{Panel A: Majority-accuracy (\(\tau=0.10\))}
\medskip

\begin{tabular}{@{}c|ccc@{}}
\toprule
\(\,k\) & IDF--SVD & KAHM & Mixedbread \\
\midrule
3  & \(0.298\,[0.286, 0.311]\) & \(\mathbf{0.422}\,[0.408, 0.435]\) & \(0.379\,[0.366, 0.393]\) \\
5  & \(0.302\,[0.289, 0.315]\) & \(\mathbf{0.430}\,[0.417, 0.444]\) & \(0.392\,[0.378, 0.405]\) \\
10 & \(0.305\,[0.292, 0.318]\) & \(\mathbf{0.436}\,[0.422, 0.450]\) & \(0.402\,[0.388, 0.416]\) \\
15 & \(0.311\,[0.298, 0.324]\) & \(\mathbf{0.427}\,[0.414, 0.440]\) & \(0.400\,[0.387, 0.414]\) \\
20 & \(0.305\,[0.292, 0.318]\) & \(\mathbf{0.424}\,[0.411, 0.438]\) & \(0.397\,[0.384, 0.411]\) \\
\bottomrule
\end{tabular}

\bigskip

\textbf{Panel B: Mean consensus fraction}
\medskip

\begin{tabular}{@{}c|ccc@{}}
\toprule
\(\,k\) & IDF--SVD & KAHM & Mixedbread \\
\midrule
3  & \(0.279\,[0.268, 0.290]\) & \(\mathbf{0.378}\,[0.367, 0.390]\) & \(0.345\,[0.334, 0.356]\) \\
5  & \(0.271\,[0.261, 0.282]\) & \(\mathbf{0.362}\,[0.352, 0.373]\) & \(0.330\,[0.320, 0.340]\) \\
10 & \(0.256\,[0.246, 0.266]\) & \(\mathbf{0.338}\,[0.328, 0.348]\) & \(0.303\,[0.293, 0.312]\) \\
15 & \(0.247\,[0.238, 0.256]\) & \(\mathbf{0.319}\,[0.310, 0.328]\) & \(0.285\,[0.277, 0.294]\) \\
20 & \(0.238\,[0.230, 0.248]\) & \(\mathbf{0.305}\,[0.296, 0.314]\) & \(0.270\,[0.262, 0.279]\) \\
\bottomrule
\end{tabular}

\bigskip

\textbf{Panel C: Mean lift (prior)}
\medskip

\begin{tabular}{@{}c|ccc@{}}
\toprule
\(\,k\) & IDF--SVD & KAHM & Mixedbread \\
\midrule
3  & \(42.062\,[39.219, 45.081]\) & \(\mathbf{67.163}\,[63.525, 70.954]\) & \(60.549\,[56.883, 64.307]\) \\
5  & \(38.900\,[36.585, 41.255]\) & \(\mathbf{62.404}\,[59.138, 65.685]\) & \(57.062\,[53.795, 60.438]\) \\
10 & \(35.715\,[33.796, 37.715]\) & \(\mathbf{55.824}\,[53.157, 58.595]\) & \(49.011\,[46.555, 51.580]\) \\
15 & \(33.889\,[32.142, 35.688]\) & \(\mathbf{50.812}\,[48.591, 53.038]\) & \(44.308\,[42.225, 46.400]\) \\
20 & \(32.062\,[30.501, 33.693]\) & \(\mathbf{47.025}\,[45.047, 49.044]\) & \(40.339\,[38.558, 42.143]\) \\
\bottomrule
\end{tabular}
\end{table*}

Table~\ref{tab:stronger-baselines} then broadens this comparison within the same deployment contract. To make the table labels explicit, ``law-wise'' means that one separate local model is trained per law. The row \emph{Ridge (law-wise, query$\rightarrow$MB)} is the linear direct lexical-to-semantic regressor, \emph{MLP-regressor (law-wise, query$\rightarrow$MB)} is its compact nonlinear counterpart, and \emph{Retrieval-distilled student (law-wise, query$\rightarrow$MB)} is a compact student model trained by a retrieval-distillation objective but still operating under the same lexical-input-to-Mixedbread-output contract. The remaining two rows are prototype-mixture controls built on the same teacher-space decomposition as the proposed method: \emph{Logistic-proto (law-wise, query$\rightarrow$MB)} uses $k$-means prototypes together with multinomial logistic weighting over prototypes, whereas \emph{MLP-proto (law-wise, query$\rightarrow$MB)} replaces that logistic weighting by a compact law-wise MLP classifier. Finally, \emph{KAHM (query$\rightarrow$MB corpus)} is the proposed prototype-mixture adapter in which the posterior weights are estimated by KAHM geometry rather than by logistic or neural classification. All five added baselines therefore preserve the same overall serving architecture as the proposed method---cheap lexical query features in, predicted query embedding in the Mixedbread space out, law-wise distance gating at serving time, and retrieval against the same frozen Mixedbread corpus index. This makes the comparison in Table~\ref{tab:stronger-baselines} a comparison of \emph{query-side adapters under a fixed teacher space and fixed serving contract}, rather than a comparison of unrelated retrieval architectures.
\begin{table*}[t]
\centering
\small
\renewcommand{\arraystretch}{1.08}
\setlength{\tabcolsep}{5pt}
\caption{Additional evaluation-matched baselines under the same law-wise split and deployment contract, now including a compact retrieval-distilled neural student trained on law-local teacher-side retrieval behavior. Panel A reports teacher-space query-embedding reconstruction on the 5{,}000 test queries for the learned adapters; the direct Mixedbread query representation is omitted there because it is the teacher target itself. Panels B1 and B2 report micro-averaged retrieval at \(k=20\) (mean with paired-bootstrap \(95\%\) confidence intervals). In the method labels, ``law-wise'' denotes one local model per law, and ``query$\rightarrow$MB'' denotes adapters that predict a query embedding in the frozen Mixedbread space while keeping the corpus index fixed in that same space.}
\label{tab:stronger-baselines}

\textbf{Panel A: Query-embedding reconstruction against Mixedbread teacher queries}
\medskip

\begin{tabular}{@{}lcccc@{}}
\toprule
Method & MSE & \(R^2\) & Cos mean & Cos p50 \\
\midrule
Ridge (law-wise, query$\rightarrow$MB) & \(0.000104\) & \(0.8932\) & \(0.9466\) & \(\mathbf{0.9616}\) \\
Retrieval-distilled student (law-wise, query$\rightarrow$MB) & \(0.000169\) & \(0.8272\) & \(0.9137\) & \(0.9010\) \\
MLP-regressor (law-wise, query$\rightarrow$MB) & \(0.000237\) & \(0.7573\) & \(0.8788\) & \(0.8845\) \\
Logistic-proto (law-wise, query$\rightarrow$MB) & \(0.000240\) & \(0.7543\) & \(0.8772\) & \(0.8708\) \\
MLP-proto (law-wise, query$\rightarrow$MB) & \(0.000215\) & \(0.7799\) & \(0.8900\) & \(0.8883\) \\
KAHM (query$\rightarrow$MB corpus) & \(\mathbf{0.000091}\) & \(\mathbf{0.9071}\) & \(\mathbf{0.9536}\) & \(0.9606\) \\
\bottomrule
\end{tabular}

\bigskip

\textbf{Panel B1: Ranking metrics at \(k=20\)}
\medskip

\begin{tabular}{@{}lccc@{}}
\toprule
Method & MRR@20 & Hit@20 & Top-1 \\
\midrule
IDF--SVD & \(0.376\,[0.364, 0.388]\) & \(0.572\,[0.558, 0.586]\) & \(0.287\,[0.274, 0.299]\) \\
Mixedbread & \(0.478\,[0.465, 0.489]\) & \(0.688\,[0.675, 0.701]\) & \(0.378\,[0.365, 0.391]\) \\
Ridge (law-wise, query$\rightarrow$MB) & \(0.404\,[0.392, 0.416]\) & \(0.566\,[0.553, 0.580]\) & \(0.328\,[0.315, 0.341]\) \\
Retrieval-distilled student (law-wise, query$\rightarrow$MB) & \(0.225\,[0.214, 0.235]\) & \(0.361\,[0.348, 0.375]\) & \(0.168\,[0.158, 0.179]\) \\
MLP-regressor (law-wise, query$\rightarrow$MB) & \(0.456\,[0.444, 0.469]\) & \(0.575\,[0.562, 0.589]\) & \(0.385\,[0.372, 0.398]\) \\
Logistic-proto (law-wise, query$\rightarrow$MB) & \(0.222\,[0.211, 0.232]\) & \(0.306\,[0.293, 0.318]\) & \(0.180\,[0.170, 0.191]\) \\
MLP-proto (law-wise, query$\rightarrow$MB) & \(0.246\,[0.235, 0.257]\) & \(0.365\,[0.352, 0.379]\) & \(0.193\,[0.182, 0.204]\) \\
KAHM (query$\rightarrow$MB corpus) & \(\mathbf{0.504\,[0.492, 0.516]}\) & \(\mathbf{0.694\,[0.681, 0.707]}\) & \(\mathbf{0.411\,[0.397, 0.424]}\) \\
\bottomrule
\end{tabular}

\bigskip

\textbf{Panel B2: Consensus-oriented metrics at \(k=20\)}
\medskip

\begin{tabular}{@{}lccc@{}}
\toprule
Method & MajAcc & ConsFrac & Lift \\
\midrule
IDF--SVD & \(0.305\,[0.292, 0.318]\) & \(0.238\,[0.229, 0.248]\) & \(32.062\,[30.484, 33.665]\) \\
Mixedbread & \(0.397\,[0.383, 0.411]\) & \(0.270\,[0.262, 0.279]\) & \(40.339\,[38.552, 42.073]\) \\
Ridge (law-wise, query$\rightarrow$MB) & \(0.338\,[0.325, 0.351]\) & \(0.242\,[0.233, 0.250]\) & \(37.938\,[36.140, 39.846]\) \\
Retrieval-distilled student (law-wise, query$\rightarrow$MB) & \(0.174\,[0.163, 0.184]\) & \(0.124\,[0.117, 0.130]\) & \(22.408\,[20.845, 24.007]\) \\
MLP-regressor (law-wise, query$\rightarrow$MB) & \(\mathbf{0.436\,[0.422, 0.450]}\) & \(\mathbf{0.315\,[0.306, 0.325]}\) & \(\mathbf{51.117\,[49.048, 53.252]}\) \\
Logistic-proto (law-wise, query$\rightarrow$MB) & \(0.194\,[0.183, 0.205]\) & \(0.143\,[0.136, 0.151]\) & \(23.272\,[21.671, 24.906]\) \\
MLP-proto (law-wise, query$\rightarrow$MB) & \(0.204\,[0.193, 0.215]\) & \(0.147\,[0.139, 0.154]\) & \(22.260\,[20.791, 23.807]\) \\
KAHM (query$\rightarrow$MB corpus) & \(0.424\,[0.410, 0.438]\) & \(0.305\,[0.296, 0.315]\) & \(47.025\,[45.073, 49.023]\) \\
\bottomrule
\end{tabular}
\end{table*}

\begin{remark}[On the Baselines]
The baseline set is chosen to test a more specific question than broad architecture ranking. The lexical IDF--SVD system provides a non-neural low-complexity reference; direct Mixedbread query encoding provides the frozen teacher-space reference; ridge and the compact MLP regressor provide direct lexical$\rightarrow$semantic mapping baselines, one linear and one nonlinear \cite{hoerl1970ridge,rumelhart1986learning}; the retrieval-distilled student provides a compact neural baseline trained to mimic teacher-side retrieval behavior over the frozen law-local corpus rather than only pointwise teacher embeddings \cite{hofstaetter2021efficiently,zeng2022curriculum,tao2024adam}; and the logistic-prototype and MLP-prototype systems provide two structure-matched prototype-mixture controls built on the same \(k\)-means teacher-space decomposition, differing only in whether the prototype posterior weights are estimated by multinomial logistic regression or by a compact MLP classifier \cite{bishop2006prml,macqueen1967some,rumelhart1986learning}. Accordingly, the central empirical question is whether, once a frozen semantic corpus space and fixed supervision budget have been specified, KAHM geometry adds value beyond simpler evaluation-matched adapters under the same serving contract. The results indicate that it does, but in a more nuanced sense. Among the learned adapters considered here, KAHM yields the best teacher-space reconstruction and the strongest rank-sensitive law-level retrieval, whereas the compact MLP regressor is strongest on some neighborhood-consensus diagnostics and the retrieval-distilled student remains clearly weaker under the same law-wise serving contract.
\end{remark}

Confidence intervals are paired nonparametric bootstrap \(95\%\) intervals with \(5{,}000\) resamples and seed \(0\). The majority-accuracy metric uses the predominance threshold \(\tau=0.10\). In the reported run, the exact-text overlap between train and test is zero, all test topics are drawn from the topic pool already present in training, and \(15.6\%\) of test queries mention the gold law label verbatim. These details are relevant when interpreting the absolute values of all methods. Accordingly, the primary effectiveness question in this study is the relative delta between direct transformer query encoding and the proposed KAHM adapter under the same frozen corpus space and serving contract, rather than the absolute value of any single Hit@$k$ score taken in isolation. Taken together, the main empirical pattern is not merely that the proposed encoder reduces deployment-time query cost relative to direct transformer inference, but that it does so while remaining the strongest evaluation-matched adapter on the principal rank-sensitive retrieval measures. Additional diagnostics in Remark~\ref{remark_150420261^202} and Table~\ref{tab:diag-label-style} show that this relative advantage is not plausibly driven by explicit law-label cues and remains robust across six of the seven evaluated query styles. For the use case studied here, those measures are the most decision-relevant because they capture whether the correct law remains near the top of the ranking under the fixed serving contract. The stronger teacher-space reconstruction is consistent with the view that KAHM preserves more of the teacher-induced retrieval geometry in the frozen corpus space. At the same time, the broadened baseline comparison shows that the advantage is not uniform across all diagnostics: the compact law-wise MLP regressor attains slightly stronger neighborhood-consensus values on some measures. We therefore interpret the main result as stronger rank-sensitive lexical-to-semantic adaptation under the fixed-teacher and fixed-serving-contract regime studied here. In the reported CPU setting, the KAHM adapter reduces online per-query time from 800.663\,ms for direct Mixedbread query encoding to 93.834\,ms, corresponding to an 8.53$\times$ speedup, while still outperforming the compact learned adapter baselines on the main rank-sensitive law-retrieval metrics under the same fixed-teacher and fixed-serving-contract regime. The resulting compute--quality positioning is summarized in Fig.~\ref{fig:compute-quality-tradeoff}.
\begin{figure*}
    \centering
    \includegraphics[width=0.5\textwidth]{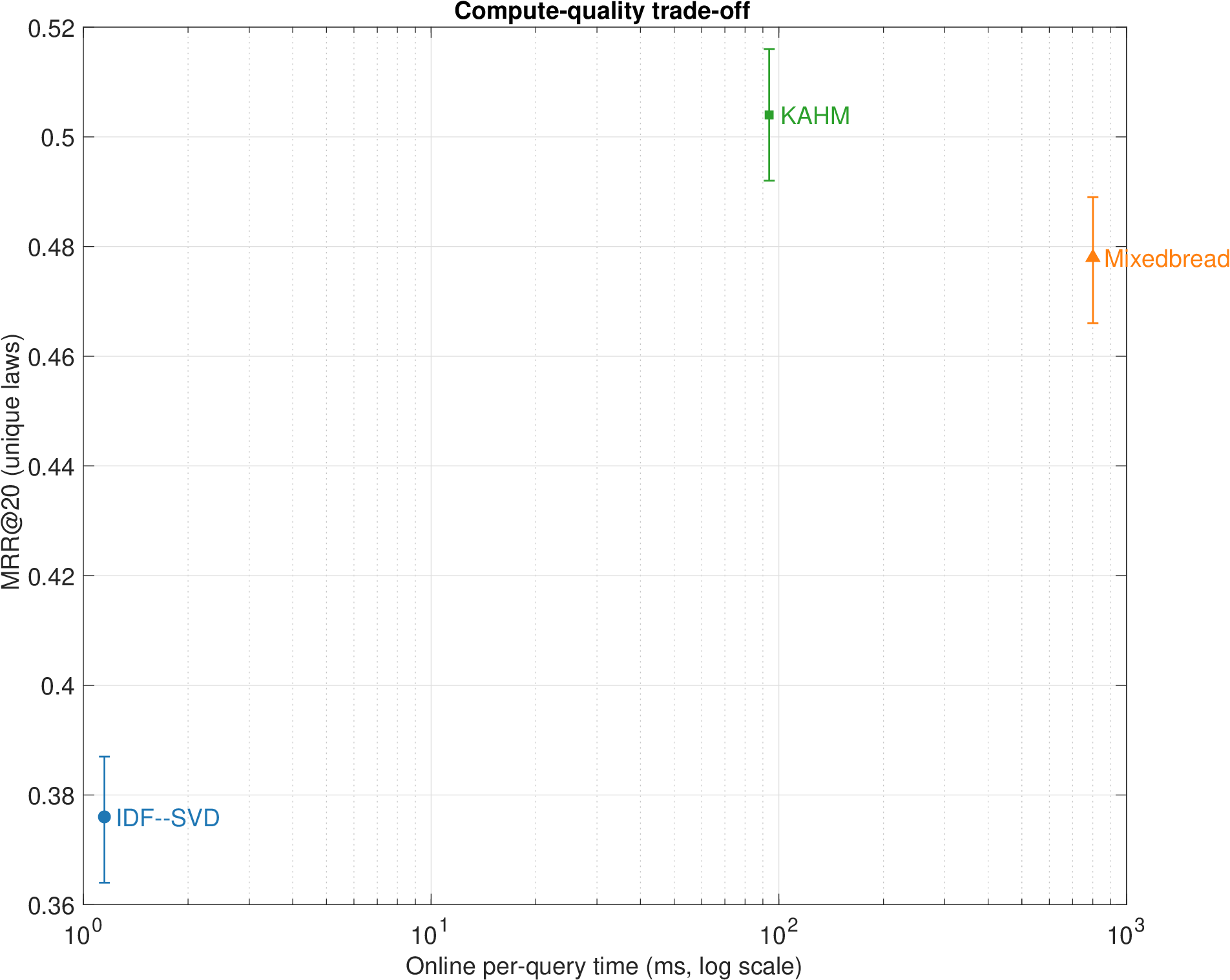}
    \caption{Compute-quality trade-off based on online per-query time and MRR@20. The proposed KAHM-based encoder occupies the middle ground between the very fast lexical baseline and the expensive online transformer baseline, while attaining the strongest retrieval quality among the three methods in this run. Error bars show paired-bootstrap 95\% confidence intervals for MRR@20.}
    \label{fig:compute-quality-tradeoff}
\end{figure*}

Across all cutoffs, the proposed KAHM-based encoder substantially outperforms the lexical IDF--SVD baseline on every reported metric. At \(k=20\), the proposed method reaches \(\mathrm{MRR}@20=0.504\), \(\mathrm{Hit}@20=0.694\), \(\mathrm{Top}\mbox{-}1=0.411\), majority-accuracy \(=0.424\), mean consensus fraction \(=0.305\), and mean lift \(=47.025\), compared with \(0.376\), \(0.572\), \(0.287\), \(0.305\), \(0.238\), and \(32.062\), respectively, for IDF--SVD. The corresponding paired deltas (KAHM minus IDF--SVD) at \(k=20\) are
\[
\Delta \mathrm{Hit}@20 = +0.122\;[0.109,0.136],\quad
\Delta \mathrm{MRR}@20 = +0.128\;[0.117,0.139],
\]
\[
\Delta \mathrm{Top}\mbox{-}1 = +0.124\;[0.110,0.137],\quad
\Delta \mathrm{MajAcc} = +0.119\;[0.106,0.132],
\]
\[
\Delta \mathrm{ConsFrac} = +0.067\;[0.060,0.074],\quad
\Delta \mathrm{Lift} = +14.963\;[13.196,16.715].
\]
All confidence intervals exclude zero, which indicates a stable advantage of the proposed encoder over the lexical baseline on this benchmark. The learned-baseline comparison yields a more differentiated and informative picture than the earlier two-baseline control alone. Among the evaluation-matched adapters, KAHM gives the strongest teacher-space reconstruction overall, with $\mathrm{MSE}=0.000091$ and $R^2=0.9071$, compared with $0.000104/0.8932$ for ridge, $0.000169/0.8272$ for the retrieval-distilled student, $0.000240/0.7543$ for logistic-proto, $0.000237/0.7573$ for the MLP regressor, and $0.000215/0.7799$ for the MLP prototype adapter. Further, the law-wise retrieval-distilled student remains well below KAHM on every reported measure and at every evaluated cutoff. At \(k=20\), it attains $\mathrm{MRR}@20=0.225$, $\mathrm{Hit}@20=0.361$, $\mathrm{Top\mbox{-}1}=0.168$, majority-accuracy \(=0.174\), mean consensus fraction \(=0.124\), and mean lift \(=22.408\), all substantially below the corresponding KAHM values. Under the same law-wise split, frozen Mixedbread corpus space, and distance-gated serving contract, replacing pointwise embedding regression by a retrieval-distillation objective alone is therefore not sufficient to recover the rank-sensitive advantages of the proposed KAHM geometry.

Table~\ref{tab:stronger-baselines} further shows that KAHM remains the strongest learned adapter on the three rank-sensitive law-retrieval measures across all evaluated cutoffs. At \(k=20\), it reaches $\mathrm{MRR}@20=0.504$, $\mathrm{Hit}@20=0.694$, and $\mathrm{Top\mbox{-}1}=0.411$, exceeding not only ridge, the retrieval-distilled student, and the two prototype-mixture controls, but also the strongest learned comparator on these metrics, the compact law-wise MLP regressor, which attains $0.456$, $0.575$, and $0.385$, respectively. This is the most important empirical point for the present paper, because these are the metrics most directly tied to correct law identification near the top of the ranking.

At the same time, the broadened neural comparison shows that the advantage of KAHM is not uniform across all diagnostics. The compact law-wise MLP regressor is stronger on the consensus-oriented metrics at \(k=20\), reaching majority-accuracy \(0.436\), mean consensus fraction \(0.315\), and mean lift \(51.117\), whereas KAHM attains \(0.424\), \(0.305\), and \(47.025\). A natural interpretation is that the compact neural direct adapter can produce slightly purer local top-\(k\) neighborhoods, while KAHM remains superior at reconstructing the teacher space and converting that reconstruction into stronger law-level ranking performance. By contrast, the retrieval-distilled student and both prototype-mixture controls remain clearly weaker than KAHM. This matters because it shows that neither the prototype-mixture form by itself nor a more retrieval-aware compact neural student under the same serving contract is sufficient to recover the KAHM result. The evidence is therefore most consistent with a more specific claim: under the fixed deployment contract studied here, KAHM geometry yields the best overall rank-sensitive lexical$\rightarrow$semantic adaptation, even though a compact MLP regressor can be preferable on some consensus-style diagnostics.

A second notable empirical finding is that the KAHM adapter also surpasses the direct transformer-query baseline on most quality measures. At \(k=20\), the paired deltas (KAHM minus Mixedbread) are
\[
\Delta \mathrm{MRR}@20 = +0.026\;[0.018,0.035],\quad
\Delta \mathrm{Top}\mbox{-}1 = +0.033\;[0.022,0.043],
\]
\[
\Delta \mathrm{MajAcc} = +0.027\;[0.018,0.037],\quad
\Delta \mathrm{ConsFrac} = +0.035\;[0.031,0.039],
\]
\[
\Delta \mathrm{Lift} = +6.687\;[5.625,7.776].
\]
For \(\mathrm{Hit}@20\), the corresponding delta is \(+0.006\;[-0.003,0.016]\), so the difference at that cutoff is statistically inconclusive under the paired bootstrap criterion. Thus, the proposed encoder is not merely a cheap approximation of the transformer query encoder in this setting; it yields systematically better rank-sensitive and neighborhood-purity measures while keeping the strong semantic corpus index fixed.

\begin{remark}[KAHM Encoder Can Outperform the Teacher Query Encoder]
The comparison against the direct Mixedbread query baseline is best viewed as an in-domain query-adaptation result within the teacher’s frozen semantic space. Both methods operate in the same frozen corpus space, but the proposed encoder is trained specifically on the Austrian-law query distribution considered here. It can therefore act as a task-specific query-side adapter that re-expresses lexical evidence in a form better aligned with the downstream law-retrieval objective and with the fixed corpus index. Under this reading, the observed gain is most naturally interpreted as in-domain calibration of query representations within a fixed teacher space, not as a general claim that the analytical encoder dominates the underlying transformer.
\end{remark}

Several qualitative trends are visible across the cutoffs; see Fig.~\ref{fig:quality-vs-k} for the rank-sensitive retrieval measures and Fig.~\ref{fig:consensus-vs-k} for the consensus- and routing-oriented diagnostics. First, \(\mathrm{Hit}@k\) and \(\mathrm{MRR}@k\) increase with \(k\) for all systems, as expected. Second, the mean consensus fraction and mean lift decrease as \(k\) grows, reflecting the fact that a larger retrieval neighborhood dilutes the local purity around the correct law. Third, majority-accuracy peaks at an intermediate cutoff and then slightly decreases: for the proposed method it rises from \(0.422\) at \(k=3\) to \(0.436\) at \(k=10\), then decreases to \(0.424\) at \(k=20\). This pattern suggests that the correct law is concentrated early in the ranked list, whereas larger neighborhoods admit increasing contamination from semantically adjacent statutes. Finally, \(\mathrm{Top}\mbox{-}1\) is constant across all values of \(k\), which is consistent with its definition as a rank-1 measure.
\begin{figure*}
    \centering
    \includegraphics[width=\textwidth]{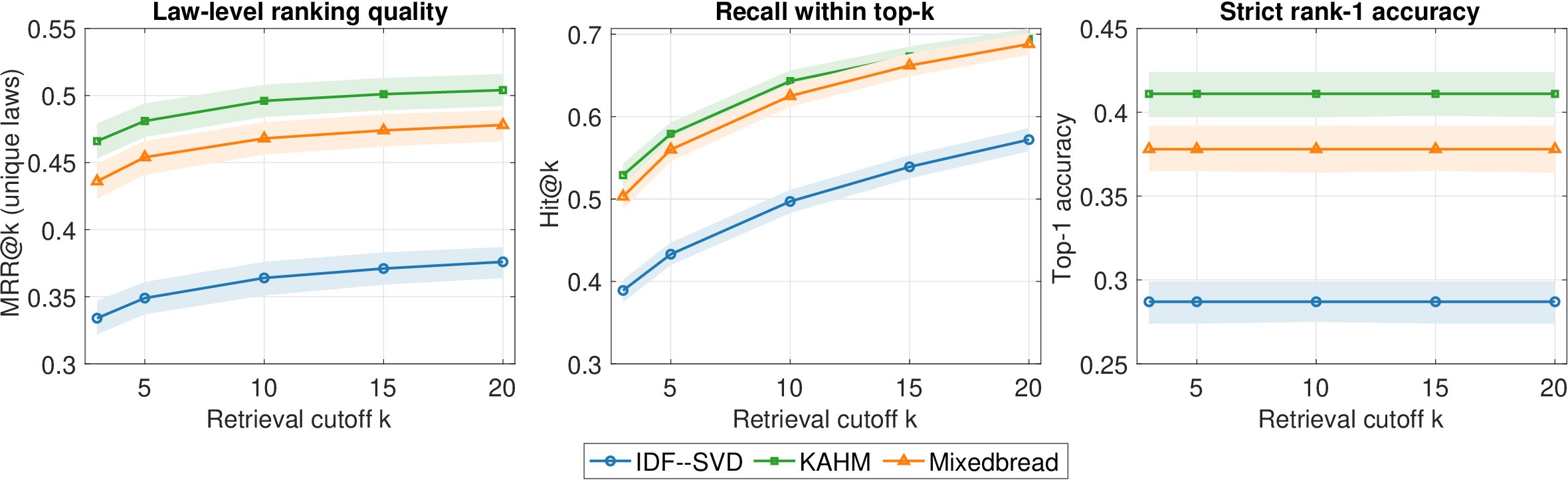}
    \caption{Main retrieval-quality metrics across cutoffs. The proposed KAHM (query$\rightarrow$MB corpus) encoder consistently dominates the lexical IDF--SVD baseline and remains above the direct transformer-query baseline on the principal rank-sensitive measures. Shaded regions indicate paired-bootstrap 95\% confidence intervals.}
    \label{fig:quality-vs-k}
\end{figure*}
\begin{figure*}
    \centering
    \includegraphics[width=\textwidth]{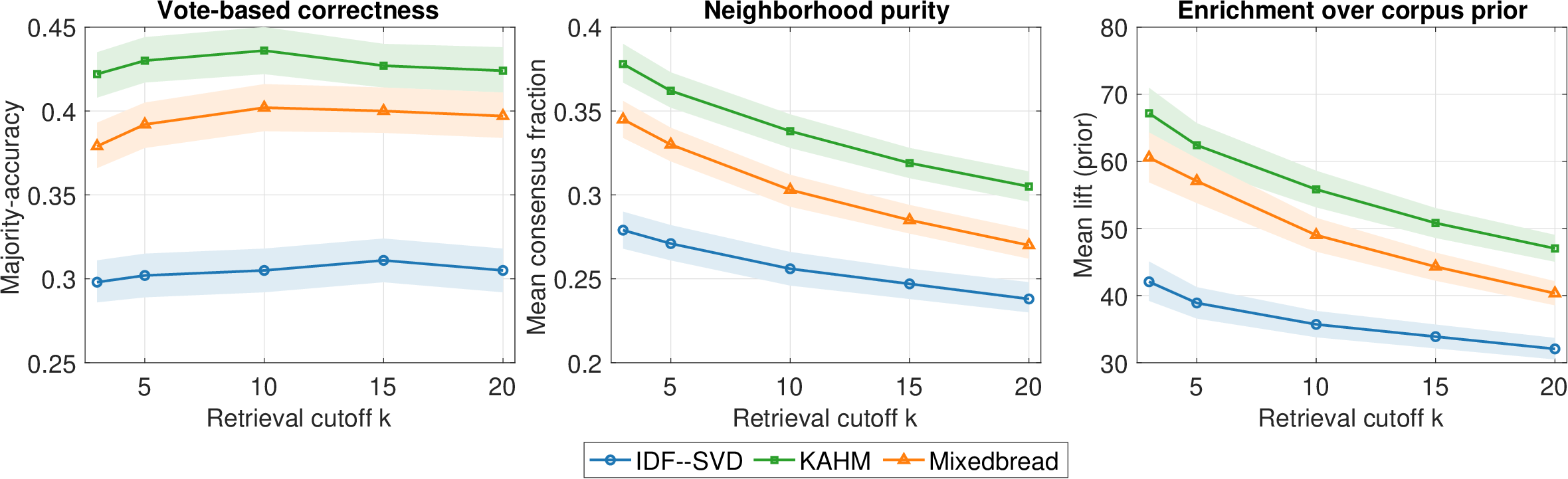}
    \caption{Consensus- and routing-sensitive metrics across cutoffs for the three systems shown (IDF--SVD, Mixedbread, and KAHM). Within this three-way comparison, the KAHM-based encoder yields purer retrieval neighborhoods around the gold law, as reflected in majority-accuracy, mean consensus fraction, and mean lift. Shaded regions indicate paired-bootstrap 95\% confidence intervals.}
    \label{fig:consensus-vs-k}
\end{figure*}

The report also includes a macro-averaged robustness analysis in which each law receives equal weight irrespective of query frequency. The macro-level point estimates reproduce the same ordering as the micro-level results, while the confidence intervals are wider because resampling occurs at the law level. At \(k=20\), the macro scores of the proposed system are \(\mathrm{MRR}=0.504\,[0.446,0.562]\), \(\mathrm{Hit}=0.694\,[0.625,0.756]\), \(\mathrm{Top}\mbox{-}1=0.411\,[0.354,0.468]\), and majority-accuracy \(=0.424\,[0.361,0.491]\). The corresponding macro deltas against IDF--SVD remain strictly positive, for example \(\Delta \mathrm{MRR}=+0.128\,[0.100,0.157]\) and \(\Delta \mathrm{Hit}=+0.122\,[0.086,0.159]\), indicating that the gains of the proposed method are not an artifact of a few high-frequency laws. For completeness, Tables~\ref{tab:full-micro-deltas} and~\ref{tab:full-macro-deltas} report the paired delta results across all evaluated cutoffs, both for the primary micro-averaged analysis and for the macro-averaged robustness analysis in which each law receives equal weight. These tables make explicit that, relative to the lexical IDF--SVD baseline, the proposed method yields strictly positive paired deltas on all six reported measures for every evaluated cutoff. Relative to the direct transformer-query baseline, the proposed method retains positive paired deltas for MRR, Top-1 accuracy, majority-accuracy, mean consensus fraction, and mean lift at every evaluated cutoff, whereas Hit@\(\,k\) becomes statistically inconclusive only at \(k=20\). Table~\ref{tab:timing-breakdown} further reports the component-level wall-clock timings and memory footprints underlying the compact online profile in Table~\ref{tab:routing-and-speed}, confirming that in the semantic retrieval paths the dominant online cost lies in query embedding rather than in FAISS search.
\begin{table*}
\centering
\renewcommand{\arraystretch}{1.1}
\setlength{\tabcolsep}{6pt}
\caption{Full micro-averaged paired deltas across all evaluated cutoffs. Positive values favor KAHM. Confidence intervals are paired-bootstrap \(95\%\) intervals.}
\label{tab:full-micro-deltas}

\textbf{Panel A: KAHM minus IDF--SVD (ranking metrics)}
\medskip

\begin{tabular}{@{}c|ccc@{}}
\toprule
\(\,k\) & \(\Delta\)Hit@\(\,k\) & \(\Delta\)MRR@\(\,k\) & \(\Delta\)Top-1 \\
\midrule
3  & \(+0.141\,[0.127, 0.155]\) & \(+0.132\,[0.119, 0.144]\) & \(+0.124\,[0.110, 0.137]\) \\
5  & \(+0.146\,[0.132, 0.159]\) & \(+0.133\,[0.121, 0.144]\) & \(+0.124\,[0.110, 0.138]\) \\
10 & \(+0.145\,[0.132, 0.159]\) & \(+0.132\,[0.121, 0.143]\) & \(+0.124\,[0.110, 0.137]\) \\
15 & \(+0.134\,[0.120, 0.148]\) & \(+0.130\,[0.119, 0.141]\) & \(+0.124\,[0.110, 0.137]\) \\
20 & \(+0.122\,[0.109, 0.136]\) & \(+0.128\,[0.117, 0.139]\) & \(+0.124\,[0.110, 0.137]\) \\
\bottomrule
\end{tabular}

\bigskip

\textbf{Panel B: KAHM minus IDF--SVD (consensus- and routing-sensitive metrics)}
\medskip

\begin{tabular}{@{}c|ccc@{}}
\toprule
\(\,k\) & \(\Delta\)MajAcc & \(\Delta\)ConsFrac & \(\Delta\)Lift \\
\midrule
3  & \(+0.123\,[0.110, 0.137]\) & \(+0.100\,[0.090, 0.110]\) & \(+25.101\,[21.323, 28.982]\) \\
5  & \(+0.128\,[0.115, 0.141]\) & \(+0.091\,[0.082, 0.100]\) & \(+23.504\,[20.436, 26.677]\) \\
10 & \(+0.131\,[0.117, 0.144]\) & \(+0.081\,[0.073, 0.089]\) & \(+20.109\,[17.773, 22.528]\) \\
15 & \(+0.116\,[0.103, 0.129]\) & \(+0.072\,[0.065, 0.079]\) & \(+16.922\,[14.899, 18.916]\) \\
20 & \(+0.119\,[0.106, 0.132]\) & \(+0.067\,[0.060, 0.074]\) & \(+14.963\,[13.196, 16.715]\) \\
\bottomrule
\end{tabular}

\bigskip

\textbf{Panel C: KAHM minus Mixedbread (ranking metrics)}
\medskip

\begin{tabular}{@{}c|ccc@{}}
\toprule
\(\,k\) & \(\Delta\)Hit@\(\,k\) & \(\Delta\)MRR@\(\,k\) & \(\Delta\)Top-1 \\
\midrule
3  & \(+0.027\,[0.016, 0.037]\) & \(+0.030\,[0.021, 0.040]\) & \(+0.033\,[0.021, 0.044]\) \\
5  & \(+0.019\,[0.009, 0.030]\) & \(+0.028\,[0.019, 0.037]\) & \(+0.033\,[0.022, 0.043]\) \\
10 & \(+0.017\,[0.007, 0.028]\) & \(+0.028\,[0.019, 0.036]\) & \(+0.033\,[0.022, 0.044]\) \\
15 & \(+0.010\,[0.001, 0.020]\) & \(+0.027\,[0.019, 0.035]\) & \(+0.033\,[0.021, 0.043]\) \\
20 & \(+0.006\,[-0.003, 0.016]\) & \(+0.026\,[0.018, 0.035]\) & \(+0.033\,[0.022, 0.043]\) \\
\bottomrule
\end{tabular}

\bigskip

\textbf{Panel D: KAHM minus Mixedbread (consensus- and routing-sensitive metrics)}
\medskip

\begin{tabular}{@{}c|ccc@{}}
\toprule
\(\,k\) & \(\Delta\)MajAcc & \(\Delta\)ConsFrac & \(\Delta\)Lift \\
\midrule
3  & \(+0.042\,[0.031, 0.053]\) & \(+0.033\,[0.026, 0.041]\) & \(+6.614\,[3.911, 9.331]\) \\
5  & \(+0.038\,[0.028, 0.049]\) & \(+0.032\,[0.026, 0.039]\) & \(+5.341\,[3.146, 7.481]\) \\
10 & \(+0.034\,[0.023, 0.044]\) & \(+0.035\,[0.029, 0.040]\) & \(+6.813\,[5.177, 8.336]\) \\
15 & \(+0.028\,[0.018, 0.038]\) & \(+0.034\,[0.029, 0.039]\) & \(+6.503\,[5.246, 7.718]\) \\
20 & \(+0.027\,[0.018, 0.037]\) & \(+0.035\,[0.031, 0.039]\) & \(+6.687\,[5.625, 7.776]\) \\
\bottomrule
\end{tabular}
\end{table*}
\begin{table*}
\centering
\renewcommand{\arraystretch}{1.1}
\setlength{\tabcolsep}{6pt}
\caption{Full macro-averaged paired deltas across all evaluated cutoffs. Macro-averaging first averages within each law and then averages uniformly across laws. Positive values favor KAHM. Confidence intervals are law-bootstrap \(95\%\) intervals.}
\label{tab:full-macro-deltas}

\textbf{Panel A: KAHM minus IDF--SVD (ranking metrics)}
\medskip

\begin{tabular}{@{}c|ccc@{}}
\toprule
\(\,k\) & \(\Delta\)Hit@\(\,k\) & \(\Delta\)MRR@\(\,k\) & \(\Delta\)Top-1 \\
\midrule
3  & \(+0.141\,[0.109, 0.172]\) & \(+0.132\,[0.103, 0.162]\) & \(+0.124\,[0.093, 0.156]\) \\
5  & \(+0.146\,[0.114, 0.178]\) & \(+0.133\,[0.104, 0.162]\) & \(+0.124\,[0.092, 0.155]\) \\
10 & \(+0.145\,[0.110, 0.182]\) & \(+0.132\,[0.104, 0.162]\) & \(+0.124\,[0.092, 0.156]\) \\
15 & \(+0.134\,[0.098, 0.171]\) & \(+0.130\,[0.103, 0.160]\) & \(+0.124\,[0.093, 0.155]\) \\
20 & \(+0.122\,[0.086, 0.159]\) & \(+0.128\,[0.100, 0.157]\) & \(+0.124\,[0.093, 0.156]\) \\
\bottomrule
\end{tabular}

\bigskip

\textbf{Panel B: KAHM minus IDF--SVD (consensus- and routing-sensitive metrics)}
\medskip

\begin{tabular}{@{}c|ccc@{}}
\toprule
\(\,k\) & \(\Delta\)MajAcc & \(\Delta\)ConsFrac & \(\Delta\)Lift \\
\midrule
3  & \(+0.123\,[0.091, 0.156]\) & \(+0.100\,[0.073, 0.129]\) & \(+25.136\,[12.537, 38.554]\) \\
5  & \(+0.128\,[0.096, 0.160]\) & \(+0.091\,[0.065, 0.118]\) & \(+23.541\,[12.724, 35.273]\) \\
10 & \(+0.131\,[0.099, 0.165]\) & \(+0.081\,[0.057, 0.106]\) & \(+20.141\,[11.253, 29.602]\) \\
15 & \(+0.117\,[0.086, 0.148]\) & \(+0.072\,[0.048, 0.096]\) & \(+16.945\,[9.248, 25.482]\) \\
20 & \(+0.119\,[0.089, 0.152]\) & \(+0.067\,[0.045, 0.090]\) & \(+14.980\,[7.698, 22.644]\) \\
\bottomrule
\end{tabular}

\bigskip

\textbf{Panel C: KAHM minus Mixedbread (ranking metrics)}
\medskip

\begin{tabular}{@{}c|ccc@{}}
\toprule
\(\,k\) & \(\Delta\)Hit@\(\,k\) & \(\Delta\)MRR@\(\,k\) & \(\Delta\)Top-1 \\
\midrule
3  & \(+0.027\,[0.013, 0.041]\) & \(+0.030\,[0.019, 0.042]\) & \(+0.033\,[0.020, 0.046]\) \\
5  & \(+0.019\,[0.008, 0.031]\) & \(+0.028\,[0.017, 0.039]\) & \(+0.033\,[0.020, 0.046]\) \\
10 & \(+0.017\,[0.006, 0.028]\) & \(+0.028\,[0.017, 0.038]\) & \(+0.033\,[0.019, 0.046]\) \\
15 & \(+0.010\,[-0.002, 0.023]\) & \(+0.027\,[0.017, 0.037]\) & \(+0.033\,[0.020, 0.046]\) \\
20 & \(+0.006\,[-0.006, 0.019]\) & \(+0.026\,[0.016, 0.036]\) & \(+0.033\,[0.020, 0.046]\) \\
\bottomrule
\end{tabular}

\bigskip

\textbf{Panel D: KAHM minus Mixedbread (consensus- and routing-sensitive metrics)}
\medskip

\begin{tabular}{@{}c|ccc@{}}
\toprule
\(\,k\) & \(\Delta\)MajAcc & \(\Delta\)ConsFrac & \(\Delta\)Lift \\
\midrule
3  & \(+0.042\,[0.028, 0.059]\) & \(+0.033\,[0.023, 0.044]\) & \(+6.625\,[3.169, 10.335]\) \\
5  & \(+0.038\,[0.025, 0.052]\) & \(+0.032\,[0.023, 0.041]\) & \(+5.346\,[2.044, 8.660]\) \\
10 & \(+0.034\,[0.020, 0.047]\) & \(+0.035\,[0.027, 0.042]\) & \(+6.821\,[4.254, 9.603]\) \\
15 & \(+0.028\,[0.016, 0.040]\) & \(+0.034\,[0.027, 0.041]\) & \(+6.509\,[4.378, 8.704]\) \\
20 & \(+0.027\,[0.014, 0.042]\) & \(+0.035\,[0.028, 0.042]\) & \(+6.692\,[4.704, 8.841]\) \\
\bottomrule
\end{tabular}
\end{table*}
\begin{table*}
\centering
\small
\caption{Component-level wall-clock timings and memory footprints for the reported run. Warm-up times are excluded from the online per-query totals in Table~\ref{tab:routing-and-speed}.}
\label{tab:timing-breakdown}
\begin{tabular}{@{}lcc@{}}
\toprule
Component & Wall time & Per-query equivalent \\
\midrule
IDF--SVD query pipeline init (cold-start) & 2.634 s & 0.527 ms \\
IDF--SVD query embedding (batch) & 4.220 s & 0.844 ms \\
KAHM query model init (cold-start) & 5.126 s & 1.025 ms \\
KAHM query warm-up (excluded from online total) & 5.665 s & n/a \\
KAHM query embedding (batch) & 466.249 s & 93.250 ms \\
Mixedbread model init (cold-start) & 7.606 s & 1.521 ms \\
Mixedbread query warm-up (excluded from online total) & 1.721 s & n/a \\
Mixedbread query embedding (batch) & 4000.415 s & 800.083 ms \\
FAISS build (IDF corpus index) & 0.316 s & n/a \\
FAISS search (IDF path) & 1.526 s & 0.305 ms \\
FAISS build (MB corpus index) & 0.127 s & n/a \\
FAISS search (Mixedbread path) & 2.898 s & 0.580 ms \\
FAISS search (KAHM\(\rightarrow\)MB path) & 2.919 s & 0.584 ms \\
Corpus embedding memory (IDF matrix) & 22{,}040{,}576 bytes & n/a \\
Corpus embedding memory (MB matrix) & 44{,}081{,}152 bytes & n/a \\
\bottomrule
\end{tabular}
\end{table*}
\begin{table}
\centering
\small
\caption{Routing-oriented diagnostics and online computational profile.}
\label{tab:routing-and-speed}
\begin{tabular}{lccc}
\toprule
Method / Path & Coverage & Majority-acc & Precision \\
\midrule
IDF--SVD & 0.538 & 0.251 & 0.468 \\
Mixedbread & 0.525 & 0.308 & 0.587 \\
KAHM (query$\rightarrow$MB corpus) & 0.562 & 0.346 & 0.615 \\
\midrule
& \multicolumn{3}{c}{Per-query online profile} \\
\midrule
Path & Embed / q & Search / q & Total / q \\
\midrule
IDF--SVD & 0.844 ms & 0.305 ms & 1.149 ms \\
KAHM (query$\rightarrow$MB corpus) & 93.250 ms & 0.584 ms & 93.834 ms \\
Mixedbread & 800.083 ms & 0.580 ms & 800.663 ms \\
\bottomrule
\end{tabular}
\end{table}

From an operational perspective, the majority-vote routing sweep in Table~\ref{tab:routing-and-speed} shows that the KAHM-based system also provides the strongest confident-decision profile. For the auxiliary abstention-aware routing analysis, we report the operating point $\tau' = 0.41$, obtained by an exhaustive sweep over $\tau' \in \{0.00,0.01,\ldots,1.00\}$. The threshold $\tau'$ was selected to maximize conditional majority precision subject to a minimum coverage constraint of $0.50$ on the evaluated query set. At this operating point, the proposed method achieves coverage $0.562$, majority-accuracy $0.346$, and conditional precision $0.615$. This exceeds both the lexical baseline $(0.538/0.251/0.468)$ and the transformer-query baseline $(0.525/0.308/0.587)$, suggesting that the top-ranked local neighborhoods produced by the KAHM adapter are not only more accurate but also more reliable under this abstention-aware routing analysis. Because $\tau'$ was selected on the evaluated query set itself, this result should be read as a descriptive post-hoc routing diagnostic rather than as an independently validated deployment threshold.

Table~\ref{tab:routing-and-speed} further reports the measured query-time computational profile. The main quantity of interest is the online per-query time, defined as query embedding plus FAISS retrieval. IDF--SVD requires \(1.149\) ms per query, the KAHM adapter requires \(93.834\) ms, and online Mixedbread inference requires \(800.663\) ms. Hence, the proposed adapter is slower than the lightweight lexical baseline but still \(8.53\times\) faster than online transformer query encoding in this run. The dominant cost of both semantic methods lies in query embedding rather than search: FAISS search itself contributes only \(0.584\) ms per query on the KAHM$\rightarrow$MB path and \(0.580\) ms on the direct Mixedbread path. The aligned corpus matrices occupy \(22{,}040{,}576\) bytes in the IDF--SVD space and \(44{,}081{,}152\) bytes in the Mixedbread space. Thus, the proposed method substantially reduces online semantic-query cost relative to transformer inference, while remaining slower than the purely lexical baseline. 

The experiments were executed on an x86\_64 Apple desktop running macOS \(26.3.1\), with \(8\) logical CPU cores and \(8.00\) GiB RAM. The software stack consisted of Python \(3.11.14\), PyTorch \(2.2.2\), \texttt{faiss-cpu} \(1.13.2\), \texttt{sentence-transformers} \(5.2.0\), scikit-learn \(1.8.0\)~\cite{pedregosa2011scikit}, NumPy \(1.26.4\), pandas \(2.3.3\), and joblib \(1.5.3\). Although Apple Metal (MPS) support was available in the runtime, the evaluation reported here was explicitly executed on the CPU with thread cap \(1\). Table~\ref{tab:hardware-config} summarizes the hardware and runtime configuration.
\begin{table}
\centering
\small
\caption{Hardware and runtime configuration used for the reported experiment.}
\label{tab:hardware-config}
\begin{tabular}{ll}
\toprule
Field & Value \\
\midrule
Operating system & macOS 26.3.1 \\
Architecture & x86\_64 \\
Logical CPU cores & 8 \\
RAM & 8.00 GiB \\
Python & 3.11.14 \\
PyTorch runtime & 2.2.2 \\
FAISS & faiss-cpu 1.13.2 \\
SentenceTransformers & 5.2.0 \\
scikit-learn & 1.8.0 \\
NumPy / pandas & 1.26.4 / 2.3.3 \\
joblib & 1.5.3 \\
Requested device & CPU \\
Thread cap & 1 \\
\bottomrule
\end{tabular}
\end{table}

\begin{remark}[Diagnostic Robustness to Law-Label Cues and Query Style]\label{remark_150420261^202}
A natural concern for a synthetic law-retrieval benchmark is that performance might be inflated by explicit statute-name mentions or by concentration on only a narrow query surface form. We therefore performed two additional diagnostics on the same retrieval outputs: a label-cue analysis and a style-stratified analysis. The first result is reassuring. Of the 5{,}000 test queries, only $778$ ($15.6\%$) mention the gold law abbreviation and only $799$ ($16.0\%$) mention any law abbreviation. More importantly, the principal KAHM result remains essentially unchanged when those explicit label cues are removed. On the full test set, KAHM attains $\mathrm{MRR}@20=0.504$, $\mathrm{Hit}@20=0.694$, and Top-1 Accuracy $=0.411$. On the gold-label-free subset ($n=4222$), the corresponding values remain $\mathrm{MRR}@20=0.504$, $\mathrm{Hit}@20=0.694$, and Top-1 Accuracy $=0.411$, and on the stricter any-law-free subset ($n=4201$) they remain $\mathrm{MRR}@20=0.503$, $\mathrm{Hit}@20=0.694$, and Top-1 Accuracy $=0.411$. These diagnostics indicate that the observed retrieval gains are not primarily driven by explicit statute-name cues. The style-stratified analysis likewise supports the intended interpretation of the benchmark as a test of robustness to heterogeneous in-domain formulations. At $\mathrm{MRR}@20$, KAHM is strongest on six of the seven query styles considered here: \texttt{nl\_long} ($0.424$), \texttt{nl\_short} ($0.576$), \texttt{fragment} ($0.594$), \texttt{scenario} ($0.442$), \texttt{keyword} ($0.472$), and \texttt{authority} ($0.549$). The direct Mixedbread query encoder remains strongest on the \texttt{procedural} style ($0.508$ versus $0.470$ for KAHM). The largest KAHM advantages appear on the more compressed lexical styles, especially \texttt{keyword} and \texttt{fragment}, whereas the teacher retains an advantage on the more procedure-focused style. We therefore interpret the method as robust across a broad range of query formulations, but not uniformly dominant over the teacher oracle for every style. Table~\ref{tab:diag-label-style} summarizes these additional diagnostics.
\begin{table*}
\centering
\small
\caption{Additional diagnostic evaluation on label-free subsets and query-style strata. The table reports $\mathrm{MRR}@20$ and Top-1 Accuracy for the lexical baseline, the direct Mixedbread query encoder, and the proposed KAHM encoder. The label-free rows show that the principal KAHM result is essentially unchanged when explicit statute-name mentions are removed, while the style rows show that KAHM is strongest on six of the seven query styles considered here by $\mathrm{MRR}@20$.}
\label{tab:diag-label-style}
\begin{tabular}{lcccccc}
\toprule
& \multicolumn{3}{c}{$\mathrm{MRR}@20$} & \multicolumn{3}{c}{Top-1 Accuracy} \\
\cmidrule(lr){2-4}\cmidrule(lr){5-7}
Group & IDF--SVD & Mixedbread & KAHM & IDF--SVD & Mixedbread & KAHM \\
\midrule
Full test set & 0.376 & 0.478 & \textbf{0.504} & 0.287 & 0.378 & \textbf{0.411} \\
Gold-label-free ($n=4222$) & 0.371 & 0.479 & \textbf{0.504} & 0.285 & 0.383 & \textbf{0.411} \\
Any-law-free ($n=4201$) & 0.372 & 0.479 & \textbf{0.503} & 0.286 & 0.383 & \textbf{0.411} \\
\midrule
\texttt{nl\_long} ($n=717$) & 0.305 & 0.413 & \textbf{0.424} & 0.215 & 0.308 & \textbf{0.340} \\
\texttt{nl\_short} ($n=717$) & 0.420 & 0.565 & \textbf{0.576} & 0.332 & 0.470 & \textbf{0.478} \\
\texttt{procedural} ($n=717$) & 0.366 & \textbf{0.508} & 0.470 & 0.273 & \textbf{0.407} & 0.377 \\
\texttt{scenario} ($n=706$) & 0.302 & 0.416 & \textbf{0.442} & 0.225 & 0.319 & \textbf{0.358} \\
\texttt{keyword} ($n=719$) & 0.370 & 0.379 & \textbf{0.472} & 0.274 & 0.274 & \textbf{0.370} \\
\texttt{authority} ($n=723$) & 0.416 & 0.503 & \textbf{0.549} & 0.332 & 0.401 & \textbf{0.451} \\
\texttt{fragment} ($n=701$) & 0.451 & 0.560 & \textbf{0.594} & 0.357 & 0.468 & \textbf{0.501} \\
\bottomrule
\end{tabular}
\end{table*}
\end{remark}

\begin{remark}[Interpretation of the Experimental Findings]\label{rem_experimental_results}
Taken together, the reported experiments support four main inferences.
\begin{enumerate}
\item  The proposed KAHM-based encoder is not merely competitive with a lexical retrieval baseline, but consistently and substantially superior to it across all evaluated cutoffs and across all reported measures considered here, including both rank-sensitive and consensus-oriented criteria.
\item Relative to direct transformer query encoding in the same frozen Mixedbread corpus space, KAHM remains favorable on the main rank-sensitive law-retrieval measures while being substantially faster at query time.
\item The learned-baseline comparison sharpens the architectural interpretation of the result. Among the evaluation-matched adapters considered here, KAHM yields the strongest teacher-space reconstruction and the strongest rank-sensitive law-level retrieval, including MRR, Hit, and Top-1 accuracy across all evaluated cutoffs. Even against the compact law-wise retrieval-distilled student, KAHM retains a clear advantage. At the same time, the strongest non-KAHM comparator on the consensus-oriented diagnostics at \(k=20\) remains the compact law-wise MLP regressor, not the retrieval-distilled student. The evidence therefore supports a more specific conclusion than uniform dominance: under the fixed lexical$\to$semantic deployment contract studied here, KAHM is especially effective when the objective is accurate law identification near the top of the ranking, whereas a compact neural direct adapter can produce slightly purer local top-\(k\) neighborhoods on some consensus-style diagnostics.

\item The retrieval-distilled student and the prototype-mixture controls all remain clearly weaker than KAHM. This matters because it suggests that neither the prototype-mixture form by itself nor a more retrieval-aware compact neural student under the same serving contract is sufficient to recover the observed result. Rather, the evidence is more consistent with the view that the KAHM-specific geometric posterior estimator contributes materially to the quality of the lexical$\to$semantic adaptation under the deployment regime studied here. This comparison is important for the interpretation of the method. Because the logistic-prototype and MLP-prototype controls share the same prototype-mixture output form but replace the KAHM posterior estimator by standard discriminative classifiers, their weaker performance indicates that the gain cannot be attributed to the prototype-mixture architecture alone. Rather, the evidence points to the KAHM geometry itself as the main source of improvement, in particular to the space-folding-based estimation of posterior weights. In that sense, the experiments support the theoretical emphasis placed earlier on the KAHM space-folding measure as a more effective mechanism for posterior estimation in this lexical-to-semantic adaptation setting than a standard neural or logistic classifier operating on the same prototype decomposition.
\end{enumerate}
Overall, these findings support the proof-of-principle claim that, under the fixed-teacher deployment contract studied here, KAHM geometry can provide an analytically explicit lexical-to-semantic adapter that preserves the principal decision-relevant behavior of the teacher space while reducing online query-encoding cost.
\end{remark}

\subsection{Reproducibility and Artifact Availability}
\label{subsec:reproducibility}
To support reproducibility, the code and principal experimental artifacts used in this study are publicly available at
\url{https://github.com/MohitKumarRostock/Austrian_law_assistant}. The public repository contains the scripts for corpus construction, query-set generation, lexical and semantic embedding-index construction, training of the law-specific KAHM encoders, retrieval evaluation, and figure generation. In particular, it includes the end-to-end evaluation script \texttt{evaluate\_three\_embeddings\_storylines.py}, the baseline comparison scripts, the KAHM training scripts, the lexical-index construction script, and the figure-generation script used to derive the graphical summaries of the reported results. The repository also provides the principal stored artifacts required for rerunning the reported experiments, including the aligned corpus file, the train/test query files, the precomputed embedding bundles, the trained KAHM encoders, and the generated evaluation reports for baseline comparison.

The main numerical results reported in Sections~\ref{subsec:experimental-results} can be reproduced directly from the included artifacts by rerunning the evaluation pipeline on the public repository. The repository additionally provides installation instructions and environment specifications through a root \texttt{README.md}, a \texttt{requirements.txt} file, and a Conda \texttt{environment.yml} file. To facilitate stable paper-linked reruns, the artifact package is versioned through the tagged snapshot \texttt{v1.0.0}, which can be cited as the fixed public code state corresponding to the experiments reported in this paper. Unless otherwise noted, the source code in the repository is distributed under the Apache License 2.0.

In addition to the public artifact package, the present paper reports the random seeds, hyperparameter grids, bootstrap protocol, hardware profile, thread settings, and software environment used in the experiments. Together, these materials provide a complete basis for reproducing the reported retrieval tables and figures, as well as for rerunning the training and evaluation pipeline under the same experimental protocol.

\subsection{Threats to Validity, Scope, and Interpretation}\label{sec_scope}
The empirical results should be interpreted as evidence on a controlled benchmark rather than on fully naturalistic legal search traffic. The test queries are synthetically generated from law-specific issue inventories, so the present evaluation supports conclusions about retrieval behavior under a structured in-domain benchmark, not directly about performance on organically occurring user queries.

The train and test queries are text-disjoint, but the default split is intentionally topic-overlapping within each law. This choice follows from the supervised fixed-space deployment regime studied in the paper: the adapter is not intended to discover previously unseen legal issues, but to serve a known in-domain distribution in which the relevant law-internal issue types are represented during training. The serving-time challenge is therefore not topic-level extrapolation, but robust lexical-to-semantic encoding of new formulations of represented issues, including paraphrases, query styles, sampled factual contexts, procedural framings, and elliptical user input. The reported results should accordingly be interpreted as evidence for in-domain formulation robustness under a fixed teacher space, frozen corpus index, and matched serving contract, rather than as evidence for cold-start generalization to genuinely unseen issues within a law or transfer across laws.

A minority of test queries mention the gold law label explicitly. This can inflate absolute retrieval scores, especially for lexical front ends, and the absolute values of all methods should therefore be read in light of the benchmark construction. At the same time, the additional label-free diagnostics indicate that the main relative-comparison result is not plausibly explained by explicit law-label cues alone.

These limitations concern the empirical instantiation rather than the mathematical formulation of the fixed-teacher encoding problem. The architectural claim is deliberately narrow. The comparison is evaluation-matched rather than architecture-exhaustive: it holds fixed the lexical front end, semantic teacher space, frozen corpus index, supervision granularity, and serving contract in order to isolate the query-adaptation problem. Accordingly, the validated empirical claim is not that the proposed method dominates the broader space of efficient retrieval systems, but that within this fixed-space deployment regime it provides strong evidence that an analytically explicit KAHM encoder can preserve decision-relevant behavior while substantially reducing online query-time cost.

A final statistical caveat is that multiple test queries may instantiate alternative surface realizations of the same underlying law-topic specification. Query-level bootstrap resampling may therefore understate uncertainty relative to grouped resampling at the topic level, a concern consistent with the broader literature on validation under structured dependence~\cite{roberts2017crossvalidation}.

\section{Concluding Remarks}\label{sec_conclusion}
This paper studied deployment-time semantic encoding as the problem of estimating a query's semantic teacher representation from inexpensive lexical features when the teacher space is already fixed, with the aim of avoiding repeated online neural query encoding under a fixed serving contract. We cast this problem as conditional-mean estimation, derive a KAHM-based posterior estimator in an explicit RKHS hypothesis space, combine it with an explicit NLMS refinement of semantic prototypes learned from noisy teacher embeddings, and obtain an end-to-end error decomposition into posterior-approximation, finite-sample/generalization, and teacher-noise terms. The proposed encoder is backpropagation-free and neural-network-free in both its posterior-estimation mechanism and its deployment path.

Empirically, the Austrian-law benchmark serves as a controlled proof-of-principle test of that fixed-space query-adaptation problem. Under a comparison protocol that keeps the lexical front end, semantic target space, frozen corpus index, supervision granularity, and serving rule fixed, KAHM achieves the strongest teacher-space reconstruction and the strongest principal rank-sensitive retrieval among the learned adapters considered here while substantially reducing online query time relative to direct transformer query encoding. The evidence remains benchmark-specific, but it supports the broader methodological point: frozen neural representation spaces can, in some deployment regimes, be served by lightweight geometric estimators whose runtime path and approximation behavior remain explicit. Future work should test this principle under topic-disjoint splits, naturally occurring queries, alternative teacher encoders, additional compact adapter families, and non-retrieval fixed-representation tasks.

\section*{Acknowledgments}
The research reported in this paper has been supported by the Austrian Ministry for Transport, Innovation and Technology, the Federal Ministry for Digital and Economic Affairs, and the State of Upper Austria in the frame of the SCCH competence center INTEGRATE [(FFG grant no. 892418)] part of the FFG COMET Competence Centers for Excellent Technologies Programme. 

\appendix

\section*{Appendix A. KAHM Definitions Restated for Self-Containedness}
For self-containedness, this appendix restates the KAHM quantities used in the main text. The construction follows prior KAHM work~\cite{kumar2025operatortheoreticframeworkgradientfreefederated}. With reference to the KAHM expression (\ref{eq_220420251843}), the following definitions are provided:
\begin{itemize}
\item $\mathbf{P}_{\mathbf{X}} \in \mathbb{R}^{\underline{n} \times n}\: (\underline{n} \in \{1,2,\cdots,n \})$ is an encoding matrix such that product $\mathbf{P}_{\mathbf{X}}x$ is a lower-dimensional (i.e. $\underline{n}-$dimensional) encoding for $x$. The encoding matrix is computed from the data samples $\mathbf{X}$ using Algorithm~\ref{algorithm_encoding_matrix}. The lower-range cutoff $10^{-3}$ in Algorithm~\ref{algorithm_encoding_matrix} is used as a numerical safeguard: projected coordinates whose empirical spread is nearly collapsed are removed. This prevents near-degenerate directions from dominating the covariance-based scaling in the Gaussian kernel of~(\ref{eq_090420261651}) and ensures that the retained encoding dimensions preserve a minimal amount of geometric variation.   
\begin{algorithm}
\caption{Determination of Encoding Matrix $\mathbf{P}_{\mathbf{X}}$}
\begin{algorithmic}[1]
\Require Matrix $\mathbf{X} \in \mathbb{R}^{N \times n}$, equivalently represented as dataset $\{x^i \in \mathbb{R}^n\}_{i=1}^{N}$.
\State $\underline{n} \gets \min(20,n,N-1)$.
\State  Define $\mathbf{P}_{\mathbf{X}} \in \mathbb{R}^{\underline{n} \times n}$ such that the $i^{th}$ row of $\mathbf{P}_{\mathbf{X}}$ is equal to transpose of eigenvector corresponding to $i-$th largest eigenvalue of sample covariance matrix of samples $\{x^1,\cdots,x^N \}$. 
\While{$\mathop{\min}_{1\leq j \leq \underline{n}}\left( \mathop{\max}_{1 \leq i \leq N} (\mathbf{P}_{\mathbf{X}}x^i)_j  - \mathop{\min}_{1 \leq i \leq N} (\mathbf{P}_{\mathbf{X}}x^i)_j \right) < 1\mathrm{e}{-3}$}
\State  $\underline{n} \gets \underline{n}-1$.
\State  Define $\mathbf{P}_{\mathbf{X}} \in \mathbb{R}^{\underline{n} \times n}$ such that the $i^{th}$ row of $\mathbf{P}_{\mathbf{X}}$ is equal to transpose of eigenvector corresponding to $i^{th}$ largest eigenvalue of sample covariance matrix of dataset $\{x^1,\cdots,x^N \}$. 
\EndWhile
\State \Return $\mathbf{P}_{\mathbf{X}}$.
\end{algorithmic}  
\label{algorithm_encoding_matrix}
\end{algorithm}
\item We have
   \begin{IEEEeqnarray}{rCl} 
\underline{\mathcal{X}} & := & \{ \mathbf{P}_{\mathbf{X}} x \; \mid \; x \in \mathbb{R}^n \},
   \end{IEEEeqnarray} 
and a positive-definite real-valued kernel, $k_{\mathbf{X}}: \underline{\mathcal{X}} \times \underline{\mathcal{X}} \rightarrow \mathbb{R}$ on $\underline{\mathcal{X}}$ with a corresponding reproducing kernel Hilbert space $\mathcal{H}_{k_{\mathbf{X}}}(\underline{\mathcal{X}})$, as 
   \begin{IEEEeqnarray}{rCl}
\label{eq_090420261651}k_{\mathbf{X}}(\underline{x}^i,\underline{x}^j) & := & \exp\left(-\frac{1}{2\underline{n}}(\underline{x}^i-\underline{x}^j)^T\theta_{\mathbf{X}}^{-1}(\underline{x}^i-\underline{x}^j)\right), 
  \end{IEEEeqnarray} 
where $\underline{x}^i,\underline{x}^j \in \underline{\mathcal{X}}$ and $\theta_{\mathbf{X}}  \succ 0$ is sample covariance matrix of dataset $\{\mathbf{P}_{\mathbf{X}}x^1,\cdots,\mathbf{P}_{\mathbf{X}}x^N \}$.
\item The function $h_{\mathbf{X}}^i: \underline{\mathcal{X}} \rightarrow \mathbb{R}$, such that $h_{\mathbf{X}}^i \in \mathcal{H}_{k_{\mathbf{X}}}(\underline{\mathcal{X}})$, approximates the indicator function $\mathbbm{1}_{\{\mathbf{P}_{\mathbf{X}}x^i\}}: \underline{\mathcal{X}} \rightarrow \{0,1 \}$ as the solution of following kernel regularized least squares problem:  
  \begin{IEEEeqnarray}{rCl}
h_{\mathbf{X}}^i & = & \arg \; \min_{f \in \mathcal{H}_{k_{\mathbf{X}}}(\underline{\mathcal{X}})} \; \left( \sum_{j=1}^N \left |\mathbbm{1}_{\{\mathbf{P}_{\mathbf{X}}x^i\}}(\mathbf{P_{\mathbf{X}}}x^j) - f(\mathbf{P}_{\mathbf{X}}x^j) \right |^2 + \lambda_{\mathbf{X}}^* \left \| f \right \|^2_{\mathcal{H}_{k_{\mathbf{X}}}(\underline{\mathcal{X}})} \right), \IEEEeqnarraynumspace
  \end{IEEEeqnarray}
where the regularization parameter $\lambda_{\mathbf{X}}^* \in \mathbb{R}_+$ is given as
  \begin{IEEEeqnarray}{rCl}
\label{eq_020920251811} \lambda_{\mathbf{X}}^* & = &  \hat{e} + \frac{2}{nN}\|\mathbf{X} \|^2_F, 
     \end{IEEEeqnarray}   
where $\hat{e}$ is the unique fixed point of the function $r$ such that
  \begin{IEEEeqnarray}{rCl}
\label{eq_090120230831}\hat{e} & = & r(\hat{e},\frac{2}{nN}\|\mathbf{X} \|^2_F),
     \end{IEEEeqnarray}   
with $r: \mathbb{R}_{+} \times \mathbb{R}_{+} \rightarrow \mathbb{R}_{+}$ defined as
   \begin{IEEEeqnarray}{rCl}
r(e,\tau)& := &   \frac{1}{nN} \sum_{j=1}^n \|(\mathbf{X})_{:,j} - \mathbf{K}_{\mathbf{X}} \left(\mathbf{K}_{\mathbf{X}} + (e+\tau) \mathbf{I}_N \right)^{-1} (\mathbf{X})_{:,j}\|^2,
     \end{IEEEeqnarray} 
where $(\mathbf{I}_N)_{i,:}$ denotes the $i-$th row of identity matrix of size $N$ and $\mathbf{K}_{\mathbf{X}}$ is $N \times N$ kernel matrix with its $(i,j)-$th element defined as
   \begin{IEEEeqnarray}{rCl}
\label{eq_020920251810} (\mathbf{K}_{\mathbf{X}})_{ij}& := & k_{\mathbf{X}}(\mathbf{P}_{\mathbf{X}}x^i,\mathbf{P}_{\mathbf{X}}x^j).
  \end{IEEEeqnarray} 
The following iterations
 \begin{IEEEeqnarray}{rCl}
\label{eq_090420261740}e|_{it+1} & = & r(e|_{it},\frac{2}{nN}\|\mathbf{X} \|^2_F),\; it \in \{0,1,\cdots \}  \\
\label{eq_090420261741}e|_0 & \in & (0,\frac{1}{nN} \|\mathbf{X} \|^2_F)  
  \end{IEEEeqnarray}
converge to $\hat{e}$. By Theorem~1 in \cite{KAHM}, the map $r(\cdot,\tfrac{2}{nN}\|\mathbf{X}\|_F^2)$ is a contraction, so (\ref{eq_090420261740})--(\ref{eq_090420261741}) are the associated fixed-point iterations converging to the unique solution $\hat{e}$. The solution of the kernel regularized least squares problem follows as
  \begin{IEEEeqnarray}{rCl}
\label{eq_010920251549} h_{\mathbf{X}}^i(\cdot) & = & (\mathbf{I}_N)_{i,:} \left(\mathbf{K}_{\mathbf{X}} + \lambda_{\mathbf{X}}^* \mathbf{I}_N \right)^{-1}  \left[\begin{IEEEeqnarraybox*}[][c]{,c/c/c,} k_{\mathbf{X}}(\cdot,\mathbf{P}_{\mathbf{X}}x^1) & \cdots & k_{\mathbf{X}}(\cdot,\mathbf{P}_{\mathbf{X}}x^N)\end{IEEEeqnarraybox*} \right]^T.
  \end{IEEEeqnarray} 
The value $ h_{\mathbf{X}}^i(\mathbf{P}_{\mathbf{X}}x)$ represents the kernel-smoothed membership of point $\mathbf{P}_{\mathbf{X}}x$ to the set $\{ \mathbf{P}_{\mathbf{X}}x^i\}$. Equation~(\ref{eq_010920251549}) is also the point at which the compute-efficient serving path becomes explicit. Once the quantities that depend on the reference set $\mathbf{X}$, including the linear-algebra terms associated with $K_{\mathbf{X}}$, have been precomputed offline, evaluating $h_{\mathbf{X}}^i(\mathbf{P}_{\mathbf{X}}x)$ for a new query requires only the projection $\mathbf{P}_{\mathbf{X}}x$, the corresponding kernel evaluations, and the combination with those precomputed coefficients. Thus, the expensive matrix solve/factorization tied to $\mathbf{K}_{\mathbf{X}}$ is incurred once during model construction rather than at deployment time, which is consistent with the paper's objective of replacing repeated online neural query encoding by a lighter analytical estimator.
\item The image of $\mathcal{A}_{\mathbf{X}}$ defines a region in the affine hull of $\{x^1,\cdots,x^N\}$. That is,
 \begin{IEEEeqnarray}{rCCCl}
 \mathcal{A}_{\mathbf{X}}[\mathbb{R}^n]& := & \{ \mathcal{A}_{\mathbf{X}}(x) \; \mid \; x \in \mathbb{R}^n  \}  & \subset & \mathrm{aff}(\{x^1,\cdots,x^N \}).  
  \end{IEEEeqnarray}   
\end{itemize} 

\section*{Appendix B. Proof of Conditional Expectation as the Solution of Least-Squares Estimation Problem}
Applying Jensen's inequality for conditional expectation in (\ref{eq_170320261347}) gives
\begin{IEEEeqnarray}{rCCCl}
\| f_{x \mapsto z}(x) \|^2 & = & \|\mathop{\mathbb{E}}_{z \sim \mathbb{P}_{z \mid x}}[z \mid x] \|^2 & \leq & \mathop{\mathbb{E}}_{z \sim \mathbb{P}_{z \mid x}} [ \| z\|^2 \mid x] \qquad\text{a.s.}
\end{IEEEeqnarray}
Taking expectations yields
\begin{IEEEeqnarray}{rCCCl}
\mathop{\mathbb{E}}_{x \sim \mathbb{P}_{x }} [\| f_{x \mapsto z}(x) \|^2  ] & \leq & \mathop{\mathbb{E}}_{z \sim \mathbb{P}_{z }}[ \| z\|^2 ] & < & \infty, 
\end{IEEEeqnarray}
so indeed $f_{x \mapsto z} \in \mathbf{F}_{x \mapsto z}$. Let $f \in \mathbf{F}_{x \mapsto z}$ be arbitrary and its risk is given as
\begin{IEEEeqnarray}{rCl}
z-f(x) & = & (z-f_{x \mapsto z}(x) ) + (f_{x \mapsto z}(x) -f(x)).
\end{IEEEeqnarray}
Expanding the squared norm, we obtain
\begin{IEEEeqnarray}{rCl}
\mathop{\mathbb{E}}_{(x,z) \sim \mathbb{P}_{x,z }} [\| z-f(x) \|^2  ] & = & \mathop{\mathbb{E}}_{(x,z) \sim \mathbb{P}_{x,z }} [\| z-f_{x \mapsto z}(x) \|^2 ] + \mathop{\mathbb{E}}_{x \sim \mathbb{P}_{x }} [ \| f_{x \mapsto z}(x) -f(x) \|^2 ]\\
&& {+}\: 2 \mathop{\mathbb{E}}_{(x,z) \sim \mathbb{P}_{x,z}} [(f_{x \mapsto z}(x) -f(x))^T (z-f_{x \mapsto z}(x))  ].
\end{IEEEeqnarray}
Indeed, $f_{x \mapsto z}(x) -f(x)$ is $\sigma(x)$-measurable, and therefore 
\begin{IEEEeqnarray}{rCl}
\mathop{\mathbb{E}}_{(x,z) \sim \mathbb{P}_{x,z}} [ (f_{x \mapsto z}(x) -f(x))^T (z-f_{x \mapsto z}(x))  ] & = & \mathop{\mathbb{E}}_{x \sim \mathbb{P}_{x }}\left[ (f_{x \mapsto z}(x) -f(x))^T \mathop{\mathbb{E}}_{z \sim \mathbb{P}_{z \mid x}} \left[  (z-f_{x \mapsto z}(x)) \mid x  \right ]  \right ].
\end{IEEEeqnarray}
Since $f_{x \mapsto z}(x)=\mathop{\mathbb{E}}_{z \sim \mathbb{P}_{z \mid x}}[z \mid x]$ almost surely and $f_{x \mapsto z}(x)$ is $\sigma(x)$-measurable,
\begin{IEEEeqnarray}{rCCCl}
\mathop{\mathbb{E}}_{z \sim \mathbb{P}_{z \mid x}} \left[  (z-f_{x \mapsto z}(x)) \mid x \right ] & = & \mathop{\mathbb{E}}_{z \sim \mathbb{P}_{z \mid x}} \left[  z \mid x \right ] - f_{x \mapsto z}(x) & = & 0 \qquad\text{a.s.}
\end{IEEEeqnarray}
Hence $\mathop{\mathbb{E}}_{(x,z) \sim \mathbb{P}_{x,z }} [(z-f_{x \mapsto z}(x))^T (f_{x \mapsto z}(x) -f(x))  ] = 0$, and thus
\begin{IEEEeqnarray}{rCl}
\mathop{\mathbb{E}}_{(x,z) \sim \mathbb{P}_{x,z}} [\| z-f(x) \|^2  ] & = & \mathop{\mathbb{E}}_{(x,z) \sim \mathbb{P}_{x,z }} [\| z-f_{x \mapsto z}(x) \|^2 ] + \mathop{\mathbb{E}}_{x \sim \mathbb{P}_{x }} [ \| f_{x \mapsto z}(x) -f(x) \|^2 ].
\end{IEEEeqnarray}
It follows that
\begin{IEEEeqnarray}{rCl}
\mathop{\mathbb{E}}_{(x,z) \sim \mathbb{P}_{x,z }} [\| z-f(x) \|^2  ] & \geq & \mathop{\mathbb{E}}_{(x,z) \sim \mathbb{P}_{x,z }} [\| z-f_{x \mapsto z}(x) \|^2 ] \qquad\text{for all } f \in \mathbf{F}_{x \mapsto z}.
\end{IEEEeqnarray}
Therefore, (\ref{eq_170320261338}) follows.
\section*{Appendix C. List of Austrian Statute Abbreviations and English Glosses}
\begin{longtable}{p{0.12\textwidth}p{0.43\textwidth}p{0.35\textwidth}}
\caption{Austrian law abbreviations used in the retrieval benchmark. The English titles are descriptive glosses rather than official translations.}
\label{tab:austrian-law-abbreviations}\\
\hline
\textbf{Abbrev.} & \textbf{German short title} & \textbf{English gloss} \\
\hline
\endfirsthead

\hline
\textbf{Abbrev.} & \textbf{German short title} & \textbf{English gloss} \\
\hline
\endhead

\hline
\endfoot

ABGB   & Allgemeines bürgerliches Gesetzbuch & General Civil Code \\
APG    & Allgemeines Pensionsgesetz & General Pensions Act \\
ASVG   & Allgemeines Sozialversicherungsgesetz & General Social Insurance Act \\
AVG    & Allgemeines Verwaltungsverfahrensgesetz & General Administrative Procedure Act \\
AVRAG  & Arbeitsvertragsrechts-Anpassungsgesetz & Employment Contract Law Adaptation Act \\
AWG    & Abfallwirtschaftsgesetz & Waste Management Act \\
AZG    & Arbeitszeitgesetz & Working Hours Act \\
AktG   & Aktiengesetz & Stock Corporation Act \\
AlVG   & Arbeitslosenversicherungsgesetz & Unemployment Insurance Act \\
AngG   & Angestelltengesetz & Salaried Employees Act \\
ArbVG  & Arbeitsverfassungsgesetz & Labour Constitution Act \\
AsylG  & Asylgesetz & Asylum Act \\
AußStrG & Außerstreitgesetz & Non-Contentious Proceedings Act \\
AÜG    & Arbeitskräfteüberlassungsgesetz & Temporary Agency Work Act \\
B-VG   & Bundes-Verfassungsgesetz & Federal Constitutional Law \\
BAO    & Bundesabgabenordnung & Federal Fiscal Code \\
BEinstG & Behinderteneinstellungsgesetz & Disabled Persons Employment Act \\
BWG    & Bankwesengesetz & Banking Act \\
DSG    & Datenschutzgesetz & Data Protection Act \\
DSGVO  & Datenschutz-Grundverordnung & General Data Protection Regulation (GDPR) \\
ECG    & E-Commerce-Gesetz & E-Commerce Act \\
EKHG   & Eisenbahn- und Kraftfahrzeughaftpflichtgesetz & Railway and Motor Vehicle Liability Act \\
EO     & Exekutionsordnung & Enforcement Code \\
EStG   & Einkommensteuergesetz & Income Tax Act \\
EheG   & Ehegesetz & Marriage Act \\
ElWOG  & Elektrizitätswirtschafts- und -organisationsgesetz & Electricity Industry and Organisation Act \\
FAGG   & Fern- und Auswärtsgeschäfte-Gesetz & Distance and Off-Premises Contracts Act \\
FPG    & Fremdenpolizeigesetz & Aliens Police Act \\
FSG    & Führerscheingesetz & Driving Licence Act \\
FinStrG & Finanzstrafgesetz & Fiscal Penal Act \\
GSVG   & Gewerbliches Sozialversicherungsgesetz & Social Insurance Act for the Self-Employed \\
GewO   & Gewerbeordnung & Trade Regulation Act \\
GlBG   & Gleichbehandlungsgesetz & Equal Treatment Act \\
GmbHG  & GmbH-Gesetz & Limited Liability Companies Act \\
GrEStG & Grunderwerbsteuergesetz & Real Estate Transfer Tax Act \\
IG-L   & Immissionsschutzgesetz-Luft & Air Pollution Control Act \\
IO     & Insolvenzordnung & Insolvency Code \\
KAKuG  & Krankenanstalten- und Kuranstaltengesetz & Hospitals and Health Resorts Act \\
KBGG   & Kinderbetreuungsgeldgesetz & Childcare Allowance Act \\
KFG    & Kraftfahrgesetz & Motor Vehicles Act \\
KJBG   & Kinder- und Jugendlichen-Beschäftigungsgesetz & Children and Young Persons Employment Act \\
KSchG  & Konsumentenschutzgesetz & Consumer Protection Act \\
KStG   & Körperschaftsteuergesetz & Corporate Income Tax Act \\
KartG  & Kartellgesetz & Cartel Act \\
KommStG & Kommunalsteuergesetz & Municipal Tax Act \\
MRG    & Mietrechtsgesetz & Tenancy Act \\
MSchG  & Markenschutzgesetz & Trademark Act \\
MeldeG & Meldegesetz & Residence Registration Act \\
NAG    & Niederlassungs- und Aufenthaltsgesetz & Settlement and Residence Act \\
NoVAG  & Normverbrauchsabgabegesetz & Standard Consumption Tax Act \\
PSG    & Privatstiftungsgesetz & Private Foundations Act \\
PatG   & Patentgesetz & Patent Act \\
SH-GG  & Sozialhilfe-Grundsatzgesetz & Social Assistance Basic Principles Act \\
SMG    & Suchtmittelgesetz & Narcotic Drugs Act \\
SPG    & Sicherheitspolizeigesetz & Security Police Act \\
SchUG  & Schulunterrichtsgesetz & School Education Act \\
StGB   & Strafgesetzbuch & Criminal Code \\
StPO   & Strafprozessordnung & Code of Criminal Procedure \\
StVO   & Straßenverkehrsordnung & Road Traffic Code \\
StbG   & Staatsbürgerschaftsgesetz & Citizenship Act \\
TKG    & Telekommunikationsgesetz & Telecommunications Act \\
TSchG  & Tierschutzgesetz & Animal Protection Act \\
UG     & Universitätsgesetz & Universities Act \\
UGB    & Unternehmensgesetzbuch & Enterprise Code \\
UStG   & Umsatzsteuergesetz & Value Added Tax Act \\
UVP-G  & Umweltverträglichkeitsprüfungsgesetz & Environmental Impact Assessment Act \\
UWG    & Gesetz gegen den unlauteren Wettbewerb & Unfair Competition Act \\
UrhG   & Urheberrechtsgesetz & Copyright Act \\
UrlG   & Urlaubsgesetz & Annual Leave Act \\
VKG    & Väter-Karenzgesetz & Fathers' Leave Act \\
VStG   & Verwaltungsstrafgesetz & Administrative Penal Act \\
VbVG   & Verbandsverantwortlichkeitsgesetz & Corporate Criminal Liability Act \\
VerG   & Vereinsgesetz & Associations Act \\
VersG  & Versammlungsgesetz & Assembly Act \\
VersVG & Versicherungsvertragsgesetz & Insurance Contract Act \\
VfGG   & Verfassungsgerichtshofgesetz & Constitutional Court Act \\
VwGG   & Verwaltungsgerichtshofgesetz & Administrative Court Act \\
VwGVG  & Verwaltungsgerichtsverfahrensgesetz & Administrative Courts Procedure Act \\
WEG    & Wohnungseigentumsgesetz & Condominium Act \\
WGG    & Wohnungsgemeinnützigkeitsgesetz & Limited-Profit Housing Act \\
WRG    & Wasserrechtsgesetz & Water Rights Act \\
WaffG  & Waffengesetz & Weapons Act \\
ZPO    & Zivilprozessordnung & Code of Civil Procedure \\
ZustG  & Zustellgesetz & Service of Documents Act \\
\end{longtable}

\bibliographystyle{ACM-Reference-Format}  
\bibliography{sample-base.bib}

@String{Computing = "Computing" }

@String{Computer = "{IEEE} Computer" }

@String{Springer = "Springer-Verlag" }

@Article{kumar2024geometricallyinspiredkernelmachines,
author = {Kumar, Mohit and Valentinitsch, Alexander and Fuchs, Magdalena and Brucker, Mathias and Bowles, Juliana and Husakovic, Adnan and Abbas, Ali and Moser, Bernhard A.},
title = {Geometrically Inspired Kernel Machines for Collaborative Learning Beyond Gradient Descent},
year = {2025},
issue_date = {Aug 2025},
publisher = {AI Access Foundation},
address = {El Segundo, CA, USA},
volume = {83},
url = {https://doi.org/10.1613/jair.1.16821},
doi = {10.1613/jair.1.16821},
journal = {Journal of Artificial Intelligence Research},
month = jul,
numpages = {35}
}

@misc{kumar2025operatortheoreticframeworkgradientfreefederated,
      title={Operator-Theoretic Framework for Gradient-Free Federated Learning}, 
      author={Mohit Kumar and Mathias Brucker and Alexander Valentinitsch and Adnan Husakovic and Ali Abbas and Manuela Geiß and Bernhard A. Moser},
      year={2025},
      url={https://arxiv.org/abs/2512.01025}
}

@article{Hassibi1996HinfLMS,
  author  = {Hassibi, Babak and Sayed, Ali H. and Kailath, Thomas},
  title   = {{$H^\infty$ Optimality of the {LMS} Algorithm}},
  journal = {IEEE Transactions on Signal Processing},
  volume  = {44},
  number  = {2},
  pages   = {267--280},
  year    = {1996},
  doi     = {10.1109/78.485923}
}

@article{KAHM,
  author  = {Kumar, Mohit and Moser, Bernhard A. and Fischer, Lukas},
  title   = {On Mitigating the Utility-Loss in Differentially Private Learning: A New Perspective by a Geometrically Inspired Kernel Approach},
  journal = {Journal of Artificial Intelligence Research},
  volume  = {79},
  pages   = {515--567},
  year    = {2024},
  doi     = {10.1613/jair.1.15071}
}

@inproceedings{khattab2020colbert,
  author    = {Khattab, Omar and Zaharia, Matei},
  title     = {ColBERT: Efficient and Effective Passage Search via Contextualized Late Interaction over BERT},
  booktitle = {Proceedings of the 43rd International ACM SIGIR Conference on Research and Development in Information Retrieval},
  year      = {2020},
  pages     = {39--48},
  publisher = {Association for Computing Machinery},
  address   = {New York, NY, USA},
  doi       = {10.1145/3397271.3401075}
}

@inproceedings{formal2021splade,
  author    = {Formal, Thibault and Piwowarski, Benjamin and Clinchant, St{\'e}phane},
  title     = {SPLADE: Sparse Lexical and Expansion Model for First Stage Ranking},
  booktitle = {Proceedings of the 44th International ACM SIGIR Conference on Research and Development in Information Retrieval},
  year      = {2021},
  pages     = {2288--2292},
  publisher = {Association for Computing Machinery},
  address   = {New York, NY, USA},
  doi       = {10.1145/3404835.3463098}
}

@article{sanh2019distilbert,
  author  = {Sanh, Victor and Debut, Lysandre and Chaumond, Julien and Wolf, Thomas},
  title   = {DistilBERT, a Distilled Version of BERT: Smaller, Faster, Cheaper and Lighter},
  journal = {arXiv preprint arXiv:1910.01108},
  year    = {2019},
  doi     = {10.48550/arXiv.1910.01108}
}

@inproceedings{jiao2020tinybert,
  author    = {Jiao, Xiaoqi and Yin, Yichun and Shang, Lifeng and Jiang, Xin and Chen, Xiao and Li, Linlin and Wang, Fang and Liu, Qun},
  title     = {TinyBERT: Distilling BERT for Natural Language Understanding},
  booktitle = {Findings of the Association for Computational Linguistics: EMNLP 2020},
  year      = {2020},
  pages     = {4163--4174},
  publisher = {Association for Computational Linguistics},
  address   = {Online},
  doi       = {10.18653/v1/2020.findings-emnlp.372}
}

@online{ris_austria_portal,
  author       = {{Republic of Austria}},
  title        = {Legal Information System of the Republic of Austria (RIS)},
  year         = {2026},
  url          = {https://www.ris.bka.gv.at/},
  urldate      = {2026-04-04},
  note         = {Official portal for Austrian legal information}
}

@online{mxbai_model_card,
  author       = {{Mixedbread} and {deepset}},
  title        = {mixedbread-ai/deepset-mxbai-embed-de-large-v1},
  year         = {2024},
  url          = {https://huggingface.co/mixedbread-ai/deepset-mxbai-embed-de-large-v1},
  urldate      = {2026-04-04},
  note         = {Hugging Face model card}
}

@online{mxbai_blog,
  author       = {Lee, Sean and Shakir, Aamir and Lipp, Julius and Koenig, Darius},
  title        = {Open Source Gets DE-licious: Mixedbread x deepset German/English Embeddings},
  year         = {2024},
  url          = {https://www.mixedbread.com/blog/deepset-mxbai-embed-de-large-v1},
  urldate      = {2026-04-04},
  note         = {Mixedbread blog post}
}

@online{multilingual_e5_model_card,
  author       = {{intfloat}},
  title        = {multilingual-e5-large},
  year         = {2024},
  url          = {https://huggingface.co/intfloat/multilingual-e5-large},
  urldate      = {2026-04-04},
  note         = {Hugging Face model card}
}

@article{multilingual_e5_techreport,
  author       = {Wang, Liang and Yang, Nan and Huang, Xiaolong and Yang, Linjun and Majumder, Rangan and Wei, Furu},
  title        = {Multilingual E5 Text Embeddings: A Technical Report},
  journal      = {arXiv preprint arXiv:2402.05672},
  year         = {2024},
  doi          = {10.48550/arXiv.2402.05672},
  url          = {https://arxiv.org/abs/2402.05672}
}

@article{angle_paper,
  author       = {Li, Xianming and Li, Jing},
  title        = {AnglE-optimized Text Embeddings},
  journal      = {arXiv preprint arXiv:2309.12871},
  year         = {2023},
  doi          = {10.48550/arXiv.2309.12871},
  url          = {https://arxiv.org/abs/2309.12871}
}

@inproceedings{matryoshka_paper,
  author       = {Kusupati, Aditya and Bhatt, Gantavya and Rege, Aniket and Wallingford, Matthew and Sinha, Aditya and Ramanujan, Vivek and Howard-Snyder, William and Chen, Kaifeng and Kakade, Sham and Jain, Prateek and Farhadi, Ali},
  title        = {Matryoshka Representation Learning},
  booktitle    = {Advances in Neural Information Processing Systems},
  year         = {2022},
  url          = {https://arxiv.org/abs/2205.13147}
}

@article{faiss_library,
  author       = {Douze, Matthijs and Guzhva, Alexandr and Deng, Chengqi and Johnson, Jeff and Szilvasy, Gergely and Mazar{\'e}, Pierre-Emmanuel and Lomeli, Maria and Hosseini, Lucas and J{\'e}gou, Herv{\'e}},
  title        = {The Faiss Library},
  journal      = {arXiv preprint arXiv:2401.08281},
  year         = {2024},
  doi          = {10.48550/arXiv.2401.08281},
  url          = {https://arxiv.org/abs/2401.08281}
}

@article{salton1988term,
  author  = {Salton, Gerard and Buckley, Christopher},
  title   = {Term-weighting approaches in automatic text retrieval},
  journal = {Information Processing \& Management},
  volume  = {24},
  number  = {5},
  pages   = {513--523},
  year    = {1988},
  doi     = {10.1016/0306-4573(88)90021-0}
}

@article{halko2011finding,
  author  = {Halko, Nathan and Martinsson, Per-Gunnar and Tropp, Joel A.},
  title   = {Finding Structure with Randomness: Probabilistic Algorithms for Constructing Approximate Matrix Decompositions},
  journal = {SIAM Review},
  volume  = {53},
  number  = {2},
  pages   = {217--288},
  year    = {2011},
  doi     = {10.1137/090771806}
}

@inproceedings{wolf2020transformers,
  author    = {Wolf, Thomas and Debut, Lysandre and Sanh, Victor and Chaumond, Julien and Delangue, Clement and Moi, Anthony and Cistac, Pierric and Rault, Tim and Louf, Remi and Funtowicz, Morgan and Davison, Joe and Shleifer, Sam and von Platen, Patrick and Ma, Clara and Jernite, Yacine and Plu, Julien and Xu, Canwen and Le Scao, Teven and Gugger, Sylvain and Drame, Mariama and Lhoest, Quentin and Rush, Alexander},
  title     = {Transformers: State-of-the-Art Natural Language Processing},
  booktitle = {Proceedings of the 2020 Conference on Empirical Methods in Natural Language Processing: System Demonstrations},
  pages     = {38--45},
  year      = {2020},
  address   = {Online},
  publisher = {Association for Computational Linguistics},
  doi       = {10.18653/v1/2020.emnlp-demos.6},
  url       = {https://aclanthology.org/2020.emnlp-demos.6/}
}

@inproceedings{vaswani2017attention,
  author    = {Vaswani, Ashish and Shazeer, Noam and Parmar, Niki and Uszkoreit, Jakob and Jones, Llion and Gomez, Aidan N. and Kaiser, {\L}ukasz and Polosukhin, Illia},
  title     = {Attention Is All You Need},
  booktitle = {Advances in Neural Information Processing Systems},
  volume    = {30},
  year      = {2017}
}

@inproceedings{devlin2019bert,
  author    = {Devlin, Jacob and Chang, Ming-Wei and Lee, Kenton and Toutanova, Kristina},
  title     = {{BERT}: Pre-training of Deep Bidirectional Transformers for Language Understanding},
  booktitle = {Proceedings of the 2019 Conference of the North American Chapter of the Association for Computational Linguistics: Human Language Technologies, Volume 1 (Long and Short Papers)},
  pages     = {4171--4186},
  year      = {2019},
  doi       = {10.18653/v1/N19-1423}
}

@inproceedings{sentence_bert,
  author    = {Reimers, Nils and Gurevych, Iryna},
  title     = {Sentence-{BERT}: Sentence Embeddings using Siamese {BERT}-Networks},
  booktitle = {Proceedings of the 2019 Conference on Empirical Methods in Natural Language Processing and the 9th International Joint Conference on Natural Language Processing (EMNLP-IJCNLP)},
  pages     = {3982--3992},
  year      = {2019},
  doi       = {10.18653/v1/D19-1410}
}

@inproceedings{karpukhin2020dpr,
  author    = {Karpukhin, Vladimir and Oguz, Barlas and Min, Sewon and Lewis, Patrick and Wu, Ledell and Edunov, Sergey and Chen, Danqi and Yih, Wen-tau},
  title     = {Dense Passage Retrieval for Open-Domain Question Answering},
  booktitle = {Proceedings of the 2020 Conference on Empirical Methods in Natural Language Processing (EMNLP)},
  pages     = {6769--6781},
  year      = {2020},
  doi       = {10.18653/v1/2020.emnlp-main.550}
}

@article{deerwester1990indexing,
  author  = {Deerwester, Scott and Dumais, Susan T. and Furnas, George W. and Landauer, Thomas K. and Harshman, Richard},
  title   = {Indexing by Latent Semantic Analysis},
  journal = {Journal of the American Society for Information Science},
  volume  = {41},
  number  = {6},
  pages   = {391--407},
  year    = {1990},
  doi     = {10.1002/(SICI)1097-4571(199009)41:6<391::AID-ASI1>3.0.CO;2-9}
}

@misc{hinton2015distilling,
  author       = {Hinton, Geoffrey and Vinyals, Oriol and Dean, Jeff},
  title        = {Distilling the Knowledge in a Neural Network},
  year         = {2015},
  eprint       = {1503.02531},
  archivePrefix= {arXiv},
  primaryClass = {cs.LG}
}

@article{aronszajn1950theory,
  author  = {Aronszajn, Nachman},
  title   = {Theory of Reproducing Kernels},
  journal = {Transactions of the American Mathematical Society},
  volume  = {68},
  number  = {3},
  pages   = {337--404},
  year    = {1950},
  doi     = {10.1090/S0002-9947-1950-0051437-7}
}

@book{scholkopf2002learning,
  author    = {Sch{\"o}lkopf, Bernhard and Smola, Alexander J.},
  title     = {Learning with Kernels: Support Vector Machines, Regularization, Optimization, and Beyond},
  publisher = {MIT Press},
  address   = {Cambridge, MA},
  year      = {2002}
}

@article{hofmann2008kernel,
  author  = {Hofmann, Thomas and Sch{\"o}lkopf, Bernhard and Smola, Alexander J.},
  title   = {Kernel Methods in Machine Learning},
  journal = {The Annals of Statistics},
  volume  = {36},
  number  = {3},
  pages   = {1171--1220},
  year    = {2008},
  doi     = {10.1214/009053607000000677}
}

@book{bishop2006prml,
  author    = {Bishop, Christopher M.},
  title     = {Pattern Recognition and Machine Learning},
  publisher = {Springer},
  address   = {New York},
  year      = {2006}
}

@article{thakur2021beir,
  author  = {Thakur, Nandan and Reimers, Nils and R{\"u}ckl{\'e}, Andreas and Srivastava, Abhishek and Gurevych, Iryna},
  title   = {{BEIR}: A Heterogenous Benchmark for Zero-shot Evaluation of Information Retrieval Models},
  journal = {arXiv preprint arXiv:2104.08663},
  year    = {2021},
  doi     = {10.48550/arXiv.2104.08663},
  url     = {https://arxiv.org/abs/2104.08663}
}

@inproceedings{muennighoff2023mteb,
  author    = {Muennighoff, Niklas and Tazi, Nouamane and Magne, Loic and Reimers, Nils},
  title     = {{MTEB}: Massive Text Embedding Benchmark},
  booktitle = {Proceedings of the 17th Conference of the European Chapter of the Association for Computational Linguistics},
  year      = {2023},
  pages     = {2014--2037},
  address   = {Dubrovnik, Croatia},
  publisher = {Association for Computational Linguistics},
  doi       = {10.18653/v1/2023.eacl-main.148},
  url       = {https://aclanthology.org/2023.eacl-main.148/}
}

@article{wang2022e5,
  author  = {Wang, Liang and Yang, Nan and Huang, Xiaolong and Jiao, Binxing and Yang, Linjun and Jiang, Daxin and Majumder, Rangan and Wei, Furu},
  title   = {Text Embeddings by Weakly-Supervised Contrastive Pre-training},
  journal = {arXiv preprint arXiv:2212.03533},
  year    = {2022},
  doi     = {10.48550/arXiv.2212.03533},
  url     = {https://arxiv.org/abs/2212.03533}
}

@article{su2022instructor,
  author  = {Su, Hongjin and Shi, Weijia and Kasai, Jungo and Wang, Yizhong and Hu, Yushi and Ostendorf, Mari and Yih, Wen-tau and Smith, Noah A. and Zettlemoyer, Luke and Yu, Tao},
  title   = {One Embedder, Any Task: Instruction-Finetuned Text Embeddings},
  journal = {arXiv preprint arXiv:2212.09741},
  year    = {2022},
  doi     = {10.48550/arXiv.2212.09741},
  url     = {https://arxiv.org/abs/2212.09741}
}

@inproceedings{chen2024m3,
  author    = {Chen, Jianlyu and Xiao, Shitao and Zhang, Peitian and Luo, Kun and Lian, Defu and Liu, Zheng},
  title     = {{M}3-Embedding: Multi-Linguality, Multi-Functionality, Multi-Granularity Text Embeddings Through Self-Knowledge Distillation},
  booktitle = {Findings of the Association for Computational Linguistics: ACL 2024},
  year      = {2024},
  pages     = {2318--2335},
  address   = {Bangkok, Thailand},
  publisher = {Association for Computational Linguistics},
  doi       = {10.18653/v1/2024.findings-acl.137},
  url       = {https://aclanthology.org/2024.findings-acl.137/}
}

@inproceedings{sakai2006bootstrap,
  author    = {Sakai, Tetsuya},
  title     = {Evaluating Evaluation Metrics Based on the Bootstrap},
  booktitle = {Proceedings of the 29th Annual International ACM SIGIR Conference on Research and Development in Information Retrieval},
  year      = {2006},
  pages     = {525--532},
  publisher = {Association for Computing Machinery},
  address   = {New York, NY, USA},
  doi       = {10.1145/1148170.1148261}
}

@book{sayed2008adaptive,
  author    = {Sayed, Ali H.},
  title     = {Adaptive Filters},
  publisher = {Wiley},
  address   = {Hoboken, NJ},
  year      = {2008}
}

@book{haykin2014adaptive,
  author    = {Haykin, Simon},
  title     = {Adaptive Filter Theory},
  edition   = {5},
  publisher = {Pearson},
  year      = {2014}
}

@article{roberts2017crossvalidation,
  author  = {Roberts, David R. and Bahn, Volker and Ciuti, Simone and Boyce, Mark S. and Elith, Jane and Guillera-Arroita, Gurutzeta and Hauenstein, Severin and Lahoz-Monfort, Jos{\'e} J. and Schr{\"o}der, Boris and Thuiller, Wilfried and Warton, David I. and Wintle, Brendan A. and Hartig, Florian and Dormann, Carsten F.},
  title   = {Cross-validation strategies for data with temporal, spatial, hierarchical, or phylogenetic structure},
  journal = {Ecography},
  volume  = {40},
  number  = {8},
  pages   = {913--929},
  year    = {2017},
  doi     = {10.1111/ecog.02881}
}

@inproceedings{hofstaetter2021efficiently,
  author    = {Hofst{\"a}tter, Sebastian and Lin, Sheng-Chieh and Yang, Jheng-Hong and Lin, Jimmy and Hanbury, Allan},
  title     = {Efficiently Teaching an Effective Dense Retriever with Balanced Topic Aware Sampling},
  booktitle = {Proceedings of the 44th International ACM SIGIR Conference on Research and Development in Information Retrieval},
  pages     = {113--122},
  year      = {2021},
  publisher = {Association for Computing Machinery},
  address   = {New York, NY, USA},
  doi       = {10.1145/3404835.3462891}
}

@article{hoerl1970ridge,
  author  = {Hoerl, Arthur E. and Kennard, Robert W.},
  title   = {Ridge Regression: Biased Estimation for Nonorthogonal Problems},
  journal = {Technometrics},
  volume  = {12},
  number  = {1},
  pages   = {55--67},
  year    = {1970},
  doi     = {10.1080/00401706.1970.10488634}
}

@inproceedings{macqueen1967some,
  author    = {MacQueen, J.},
  title     = {Some Methods for Classification and Analysis of Multivariate Observations},
  booktitle = {Proceedings of the Fifth Berkeley Symposium on Mathematical Statistics and Probability, Volume 1: Statistics},
  pages     = {281--297},
  year      = {1967},
  address   = {Berkeley, CA},
  publisher = {University of California Press},
  url       = {https://projecteuclid.org/euclid.bsmsp/1200512992}
}

@article{rumelhart1986learning,
  author  = {Rumelhart, David E. and Hinton, Geoffrey E. and Williams, Ronald J.},
  title   = {Learning Representations by Back-Propagating Errors},
  journal = {Nature},
  volume  = {323},
  number  = {6088},
  pages   = {533--536},
  year    = {1986},
  doi     = {10.1038/323533a0}
}

@article{pedregosa2011scikit,
  author  = {Pedregosa, Fabian and Varoquaux, Ga{\"e}l and Gramfort, Alexandre and Michel, Vincent and Thirion, Bertrand and Grisel, Olivier and Blondel, Mathieu and Prettenhofer, Peter and Weiss, Ron and Dubourg, Vincent and VanderPlas, Jake and Passos, Alexandre and Cournapeau, David and Brucher, Matthieu and Perrot, Matthieu and Duchesnay, {\'E}douard},
  title   = {Scikit-learn: Machine Learning in Python},
  journal = {Journal of Machine Learning Research},
  volume  = {12},
  pages   = {2825--2830},
  year    = {2011}
}

@article{izacard2022unsupervised,
  title   = {Unsupervised Dense Information Retrieval with Contrastive Learning},
  author  = {Izacard, Gautier and Caron, Mathilde and Hosseini, Lucas and Riedel, Sebastian and Bojanowski, Piotr and Joulin, Armand and Grave, Edouard},
  journal = {Transactions on Machine Learning Research},
  year    = {2022},
  url     = {https://openreview.net/forum?id=jKN1pXi7b0}
}

@inproceedings{ni2022large,
  title     = {Large Dual Encoders Are Generalizable Retrievers},
  author    = {Ni, Jianmo and Qu, Chen and Lu, Jing and Dai, Zhuyun and Hern{\'a}ndez {\'A}brego, Gustavo and Zhao, Vincent Y. and Luan, Yi and Hall, Keith B. and Chang, Ming-Wei and Yang, Yinfei},
  booktitle = {Proceedings of the 2022 Conference on Empirical Methods in Natural Language Processing},
  pages     = {9844--9855},
  year      = {2022},
  address   = {Abu Dhabi, United Arab Emirates},
  publisher = {Association for Computational Linguistics},
  doi       = {10.18653/v1/2022.emnlp-main.669},
  url       = {https://aclanthology.org/2022.emnlp-main.669/}
}

@inproceedings{santhanam2022colbertv2,
  title     = {{ColBERT}v2: Effective and Efficient Retrieval via Lightweight Late Interaction},
  author    = {Santhanam, Keshav and Khattab, Omar and Saad-Falcon, Jon and Potts, Christopher and Zaharia, Matei},
  booktitle = {Proceedings of the 2022 Conference of the North American Chapter of the Association for Computational Linguistics: Human Language Technologies},
  pages     = {3715--3734},
  year      = {2022},
  address   = {Seattle, United States},
  publisher = {Association for Computational Linguistics},
  doi       = {10.18653/v1/2022.naacl-main.272},
  url       = {https://aclanthology.org/2022.naacl-main.272/}
}

@inproceedings{lee2025nvembed,
  title     = {{NV-Embed}: Improved Techniques for Training {LLM}s as Generalist Embedding Models},
  author    = {Lee, Chankyu and Roy, Rajarshi and Xu, Mengyao and Raiman, Jonathan and Shoeybi, Mohammad and Catanzaro, Bryan and Ping, Wei},
  booktitle = {The Thirteenth International Conference on Learning Representations},
  year      = {2025},
  note      = {Spotlight},
  url       = {https://openreview.net/forum?id=lgsyLSsDRe}
}

@inproceedings{feng2024legal,
  title     = {Legal Case Retrieval: A Survey of the State of the Art},
  author    = {Feng, Yi and Li, Chuanyi and Ng, Vincent},
  booktitle = {Proceedings of the 62nd Annual Meeting of the Association for Computational Linguistics (Volume 1: Long Papers)},
  pages     = {6472--6485},
  year      = {2024},
  address   = {Bangkok, Thailand},
  publisher = {Association for Computational Linguistics},
  doi       = {10.18653/v1/2024.acl-long.350},
  url       = {https://aclanthology.org/2024.acl-long.350/}
}

@inproceedings{deng2024learning,
  title     = {Learning Interpretable Legal Case Retrieval via Knowledge-Guided Case Reformulation},
  author    = {Deng, Chenlong and Mao, Kelong and Dou, Zhicheng},
  booktitle = {Proceedings of the 2024 Conference on Empirical Methods in Natural Language Processing},
  pages     = {1253--1265},
  year      = {2024},
  address   = {Miami, Florida, USA},
  publisher = {Association for Computational Linguistics},
  doi       = {10.18653/v1/2024.emnlp-main.73},
  url       = {https://aclanthology.org/2024.emnlp-main.73/}
}

@inproceedings{gao2024enhancinglegal,
  title     = {Enhancing Legal Case Retrieval via Scaling High-quality Synthetic Query-Candidate Pairs},
  author    = {Gao, Cheng and Xiao, Chaojun and Liu, Zhenghao and Chen, Huimin and Liu, Zhiyuan and Sun, Maosong},
  booktitle = {Proceedings of the 2024 Conference on Empirical Methods in Natural Language Processing},
  year      = {2024},
  publisher = {Association for Computational Linguistics},
  doi       = {10.18653/v1/2024.emnlp-main.402},
  url       = {https://aclanthology.org/2024.emnlp-main.402/}
}

@inproceedings{zheng2025reasoning,
  title     = {A Reasoning-Focused Legal Retrieval Benchmark},
  author    = {Zheng, Lucia and Guha, Neel and Arifov, Javokhir and Zhang, Sarah and Skreta, Michal and Manning, Christopher D. and Henderson, Peter and Ho, Daniel E.},
  booktitle = {Proceedings of the 4th ACM Symposium on Computer Science and Law},
  year      = {2025},
  doi       = {10.1145/3709025.3712219}
}

@inproceedings{zeng2022curriculum,
  author    = {Zeng, Hansi and Zamani, Hamed and Vinay, Vishwa},
  title     = {Curriculum Learning for Dense Retrieval Distillation},
  booktitle = {Proceedings of the 45th International ACM SIGIR Conference on Research and Development in Information Retrieval},
  pages     = {1979--1983},
  year      = {2022},
  publisher = {Association for Computing Machinery},
  address   = {New York, NY, USA},
  doi       = {10.1145/3477495.3531791}
}

@inproceedings{tao2024adam,
  author    = {Tao, Chongyang and Liu, Chang and Shen, Tao and Xu, Can and Geng, Xiubo and Jiao, Binxing and Jiang, Daxin},
  title     = {{ADAM}: Dense Retrieval Distillation with Adaptive Dark Examples},
  booktitle = {Findings of the Association for Computational Linguistics: ACL 2024},
  pages     = {11639--11651},
  year      = {2024},
  address   = {Bangkok, Thailand},
  publisher = {Association for Computational Linguistics},
  doi       = {10.18653/v1/2024.findings-acl.692},
  url       = {https://aclanthology.org/2024.findings-acl.692/}
}

@article{song2013kernelembeddings,
author  = {Song, Le and Fukumizu, Kenji and Gretton, Arthur},
title   = {Kernel Embeddings of Conditional Distributions: A Unified Kernel Framework for Nonparametric Inference in Graphical Models},
journal = {IEEE Signal Processing Magazine},
volume  = {30},
number  = {4},
pages   = {98--111},
year    = {2013},
doi     = {10.1109/MSP.2013.2252713}
}

@article{fukumizu2013kernelbayes,
author  = {Fukumizu, Kenji and Song, Le and Gretton, Arthur},
title   = {Kernel {Bayes}' Rule: Bayesian Inference with Positive Definite Kernels},
journal = {Journal of Machine Learning Research},
volume  = {14},
pages   = {3753--3783},
year    = {2013}
}

\end{document}